\documentclass{article}

\usepackage[final]{neurips_2024}

\usepackage[utf8]{inputenc} %
\usepackage[T1]{fontenc}    %
\usepackage{hyperref}       %
\usepackage{url}            %
\usepackage{booktabs}       %
\usepackage{amsfonts}       %
\usepackage{nicefrac}       %
\usepackage{microtype}      %
\usepackage{xcolor}         %

\usepackage{amsmath}
\usepackage{amssymb}
\usepackage{mathtools}
\usepackage{amsthm}

\usepackage{tcolorbox}
\usepackage{fontawesome5}

\usepackage{algorithm,algpseudocode}

\usepackage[disable,textsize=tiny]{todonotes}
\newcommand{\hamed}[2][noinline]{\todo[color=blue!20,#1]{H: #2}}

\usepackage{graphicx}
\usepackage[caption = false]{subfig}
\usepackage{wrapfig}
\usepackage[font=small,labelfont=bf]{caption}
\usepackage{tabularray}

\usepackage[capitalize,noabbrev]{cleveref}

\newtheorem{theorem}{Theorem}[section]
\newtheorem{corollary}[theorem]{Corollary}
\newtheorem{proposition}[theorem]{Proposition}
\newtheorem{lemma}[theorem]{Lemma}

\newcommand{\R}{\mathbb{R}}

\newcommand{\W}{\mathbf{W}}
\newcommand{\M}{\mathbf{M}}
\newcommand{\hh}{\hat{h}}
\newcommand{\Q}{\mathbf{Q}}
\newcommand{\K}{\mathbf{K}}
\newcommand{\V}{\mathbf{V}}

\newcommand{\Ho}{\mathbf{H}}

\newcommand{\Hh}{\widehat{\mathbf{H}}}
\newcommand{\Qh}{\widehat{\mathbf{Q}}}
\newcommand{\Kh}{\widehat{\mathbf{K}}}

\newcommand{\qh}{\hat{q}}
\newcommand{\kh}{\hat{k}}
\newcommand{\vh}{\hat{v}}

\newcommand{\Hb}{\Bar{\mathbf{H}}}

\newcommand{\Wh}{\widehat{\mathbf{W}}}
\newcommand{\llo}{^{(\ell)}}
\newcommand{\llh}{^{(\ell+1/2)}}

\newcommand{\lln}{^{(\ell+1)}}
\newcommand{\ah}{\widehat{a}}
\newcommand{\eps}{\varepsilon}
\newcommand{\norm}[1]{\lVert #1 \rVert}
\newcommand{\abs}[1]{\lvert #1 \rvert}
\newcommand{\tp}{^\mathsf{T}}

\DeclareMathOperator{\ReLU}{ReLU}

\title{Even Sparser Graph Transformers}

\author{%
  Hamed Shirzad \\
  University of British Columbia\\
  \texttt{shirzad@cs.ubc.ca}
  \And
  Honghao Lin \\
  Carnegie Mellon University\\
  \texttt{honghaol@andrew.cmu.edu}
  \And
  Balaji Venkatachalam \\
  Meta\thanks{Work done in part while at Google.}\\
  \texttt{bave@meta.com}
  \And
  Ameya Velingker \\
  Independent Researcher\footnotemark[1]\\
  \texttt{ameyav@gmail.com}
  \And
  David P. Woodruff \\
  CMU \& Google Research \\
  \texttt{dwoodruf@cs.cmu.edu}
  \And
  Danica J.\ Sutherland \\
  UBC \& Amii\\
  \texttt{dsuth@cs.ubc.ca}
}

\begin{document}

\maketitle

\begin{abstract}
  Graph Transformers excel in long-range dependency modeling, but generally require quadratic memory complexity in the number of nodes in an input graph, and hence have trouble scaling to large graphs. Sparse attention variants such as Exphormer can help, but may require high-degree augmentations to the input graph for good performance, and do not attempt to sparsify an already-dense input graph. As the learned attention mechanisms tend to use few of these edges, such high-degree connections may be unnecessary. We show (empirically and with theoretical backing) that attention scores on graphs are usually quite consistent across network widths, and use this observation to propose a two-stage procedure, which we call Spexphormer: first, train a narrow network on the full augmented graph. Next, use only the active connections to train a wider network on a much sparser graph. We establish theoretical conditions when a narrow network's attention scores can match those of a wide network, and show that Spexphormer achieves good performance with drastically reduced memory requirements on various graph datasets. Code can be found at \url{https://github.com/hamed1375/Sp_Exphormer}.
\end{abstract}

\section{Introduction}

The predominant story of the last half-decade of machine learning has been the runaway success of
Transformer models \citep{VaswaniSPUJGKP17}, across domains from natural language processing \citep{VaswaniSPUJGKP17, devlin2018bert, zaheer2020big} to computer vision \citep{dosovitskiy2020image} and, more recently, geometric deep learning \citep{DwivediBresson21, kreuzer2021rethinking, Ying2021DoTR, rampavsek2022recipe, shirzad2023exphormer, muller2023attending}.
Conventional (``full'') Transformers, however, have a time and memory complexity of $\mathcal{O}(nd^2+n^2d)$, where $n$ is the number of entities (nodes, in the case of graphs), and $d$ is the width of the network. Many attempts have been made to make Transformers more efficient (see \citet{TayDBM20} for a survey on efficient variants for \emph{sequence modeling}).
One major line of work involves \emph{sparsifying} the attention mechanism, constraining attention from all $\mathcal O(n^2)$ pairs to some smaller set of connections. For instance, for sequential data, BigBird~\citep{zaheer2020big} constructs a sparse attention mechanism by combining sliding windows, Erd\H{o}s-R\'{e}nyi auxiliary graphs, and universal connectors. Similarly, for graph data, Exphormer~\citep{shirzad2023exphormer} constructs a sparse interaction graph consisting of edges from the input graph, an overlay expander graph, and universal connections. 
We refer to such a network as a \emph{sparse attention network}. %

Exphormer %
reduces each layer's complexity from $\mathcal{O}(nd^2+n^2d)$ to $\mathcal{O}((m+n)d^2)$, where $n$ is the number of nodes, $m$ is the number of interaction edges in the sparse attention mechanism, and $d$ is the hidden dimension or width. Even so, training is still very memory-intensive for medium to large scale graphs. Also, for densely-connected input graphs with $\Theta(n^2)$ edges, there is no asymptotic improvement in complexity, as Exphormer uses all of the $\Theta(n^2)$ edges of the original input graph.
Our goal is to scale efficient graph Transformers, such as Exphormer, to even larger graphs.

One general approach for scaling models to larger graphs is based on {batching techniques}. Prominent approaches include egocentric subgraphs and random node subsets \citep{wu2022nodeformer, wu2023difformer, wu2024simplifying}. Egocentric subgraphs choose a node and include all of its $k$-hop neighbors; the expander graphs used in Exphormer, however, are exactly defined so that the size of these subgraphs grows exponentially in the number of layers -- prohibitively expensive for larger graphs. 
A similar issue arises with universally-connected nodes,
whose representation depends on all other nodes.
For uniformly-random subset batching, as the number $b$ of batches into which the graph is divided grows, each edge has chance $\frac{1}{b}$ to appear in a given step. Thus, $b$ cannot be very large without dropping important edges. A similar problem can happen in random neighbor sampling methods such as GraphSAGE \citep{hamilton2017inductive}. Although this model works well on message-passing neural networks (MPNNs) which only use the graph edges, using it for expander-augmented graphs will select only a small ratio of the expander edges, thereby breaking the universality properties provided by the expander graph.

Expander graphs enable global information propagation, and when created with Hamiltonian cycles and self-loops, produce a model that can provably approximate a full Transformer \citep[Theorem E.3]{shirzad2023exphormer}.
Yet not all of these edges turn out to be important in practice: we expect some neighboring nodes in the updated graph to have more of an effect on a given node than others. Thus, removing low-impact neighbors can improve the scalability of the model. The challenge is to identify low-impact edges without needing to train the (too-expensive) full model. \cref{fig:synthetic_task} illustrates other advantages of this batching approach; this is also discussed further in \cref{sec:discussion}.

One approach is to train a smaller network to identify which edges are significant. It is not obvious \emph{a priori} that attention scores learned from the smaller network will estimate those in the larger network, but we present an experimental study verifying that attention scores are surprisingly consistent as the network size reduces.
We also give theoretical indications that narrow networks are capable of expressing the same attention scores as wider networks of the same architecture.

 {\bf Our approach.} We first train a small-width network in order to estimate pairwise attention score patterns, which we then use to sparsify the graph and train a larger network. We first train the graphs without edge attributes. This reduces the complexity of Exphormer to $\mathcal{O}(md+nd^2)$ and then by training a much smaller width $d_s\ll d$  network, reduces the time and memory complexity by at least a factor of ${d}/{d_s}$. We also introduce two additions to the model to improve this consistency.
Training this initial network can still be memory-intensive, but as the small width implies the matrix multiplications are small, it is practical to train this initial model on a CPU node with sufficient RAM (typically orders of magnitude larger than available GPU memory), without needing to use distributed computation.
Once this initial model is trained, the attention scores can be used in creating a sparse graph, over which we train the second network.
These initial attention scores %
can be used as edge features for the second network.

As mentioned previously, we use the attention scores obtained from the trained low-width network to sparsify the graph. By selecting a fixed number $c$ of edges per attention layer for each node, we reduce the complexity of each layer to $\mathcal O(nd^2+ndc)$. This sparsification alleviates the effect of a large number of edges, and allows for initial training with a larger degree expander graph, since most of the expander edges will be filtered for the final network. This sparsification differs from conventional graph sparsification algorithms (for MPNNs) in three ways. First, we use expander edges, self-loops, and graph edges and sparsify the combination of these patterns together. Second, this sparsification is layer-wise, which means that in a multi-layer network the attention pattern will vary from layer to layer. Finally, our sampling uses a smaller network trained on the same task, identifying important neighbors based on the task, instead of approaches independent of the task such as sampling based on PageRank or a neighbor's node degree. 

Another advantage of this approach is that the fixed number of neighbors for each node enables regular matrix calculations instead of the edge-wise calculations used by \citet{kreuzer2021rethinking, shirzad2023exphormer}, greatly improving the speed of the model. After this reduction, batching can be done based on the edges over different layers, enabling Transformers to be effectively batched while still effectively approximating the main Transformer model, enabling modeling long-range dependencies. In batching large graphs, naive implementations of sampling without replacement from attention edges with varying weights can be very slow. This is especially true if the attention scores are highly concentrated on a small number of neighbors for most of the nodes.
We use reservoir sampling \citep{efraimidis2006weighted}, enabling parallel sampling with an easy, efficient GPU implementation, improving the sampling process significantly.

We only use the Transformer part of the Exphormer model, not the dual MPNN+Transformer architecture used by \cite{shirzad2023exphormer, rampavsek2022recipe}. Unlike the Exphormer approach, we do not assume that the expander graph is of degree $\mathcal{O}(1)$; we can see this as interpolating between MPNNs and full Transformers, where smaller degree expander graphs mostly rely on the graph edges and are more similar to MPNNs, while higher degree expander graphs can resemble full attention, in the most extreme case of degree $n-1$ exactly recovering a full Transformer.

To summarize, the contributions of this paper are as follows:
1) We experimentally and theoretically analyze the similarity of attention scores for networks of different widths, and propose two small architectural changes to improve this similarity.
2) We propose layer-wise sparsification, by sampling according to the learned attention scores, and do theoretical analysis on the sparsification guarantees of the attention pattern.
3) Our two-phase training process allows us to scale Transformers to larger datasets, as it has significantly smaller memory consumption, while maintaining competitive accuracy.

\begin{figure}[t!]
\centering
\includegraphics[width = 5.5in]{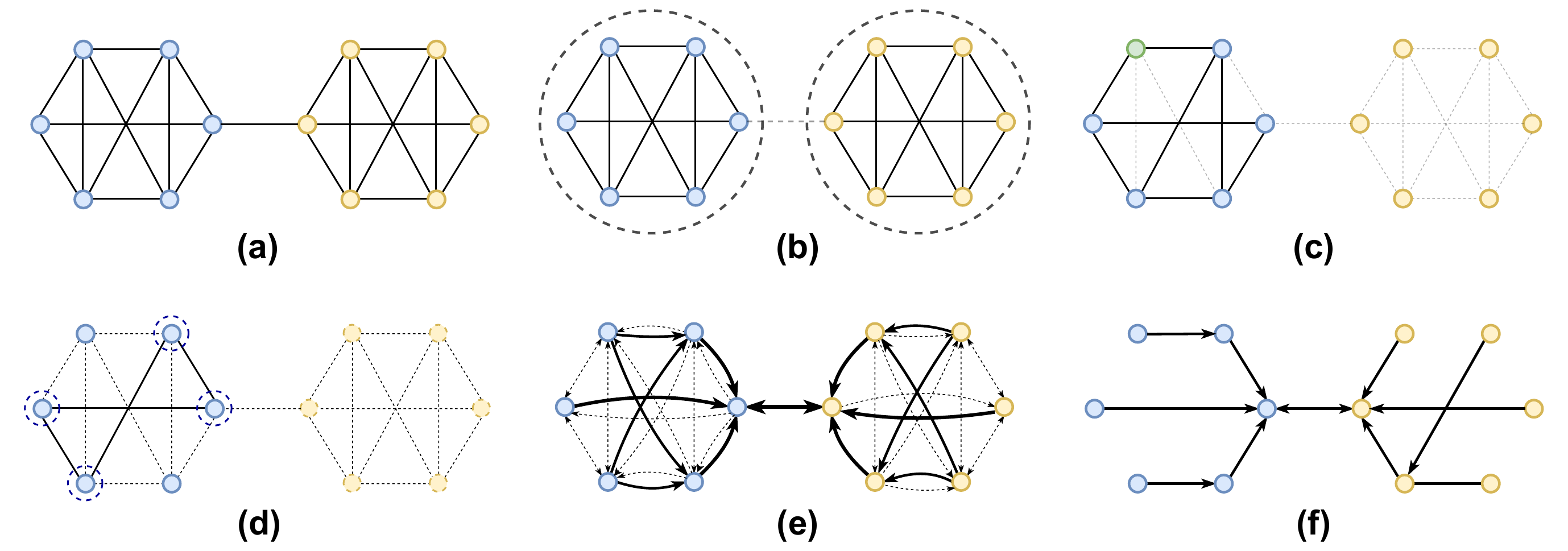}
\caption{Figure (a) shows a very simple synthetic graph where each node has a binary classification task of determining whether there exists a node of the opposite color in the same connected component. This task requires learning long-range dependencies. Figure (b) shows a natural clustering of the graph. This clustering would mean no node can do its task if models are trained only on one cluster at a time. Figure (c) shows a neighbor sampling starting from the green node, where random sampling fails to select the single important edge that bridges to the different-colored nodes. Figure (d) shows a random subset sampling strategy, where the task is solvable if and if only the two sides of the bridge between the two colors get selected. If we increase the size of each cluster, while keeping just one edge between two colors, the probability of selecting the bridge in any batch goes to zero, and thus the training will fail in this scenario. (e) shows attention scores between the nodes if trained with an attention-based network. Dashed lines have near zero attention scores, and thicker lines indicate a larger attention score. Knowing these attention scores will mean each node with just one directional edge can do the task perfectly. The attention edges are shown in (f). In case two nodes are equally informative; selecting either of them leads to the correct result.}
\label{fig:synthetic_task}
\end{figure}

\section{Related Work}

{\bf Graph Transformer Architectures.}
Attention mechanisms were proposed in early (message-passing) Graph Neural Network (GNN) architectures such as Graph Attention Networks (GAT)~\citep{velickovic2018graph}, where they guide node aggregation among neighbors, without using positional encodings. GraphBert~\citep{zhangGraphBert} finds node encodings based on the underlying graph structure. Subsequent work has proposed full-fledged graph Transformer models that generalize sequence Transformers~\citep{DwivediBresson21} and are not limited to message passing between nodes of the input graph; these include Spectral Attention Networks (SAN)~\citep{kreuzer2021rethinking}, Graphormer~\citep{Ying2021DoTR}, GraphiT~\citep{MialonCSM21}, etc. GraphGPS~\citep{rampavsek2022recipe} combines attention mechanisms with message passing, allowing the best of both worlds.

{\bf Efficient Graph Transformers.}
Several recent works have proposed various scalable graph transformer architectures. NAGphormer~\citep{nagphormer22} and Gophormer~\citep{gophormer2021} use a sampling-based approach. On the other hand, Difformer~\citep{wu2023difformer} proposes a continuous time diffusion-based transformer model. Exphormer~\citep{shirzad2023exphormer} proposes a sparse graph that combines the input graph with edges of an expander graph as well as virtual nodes. They show that their model works better than applying other sparse Transformer methods developed for sequences. Another work, NodeFormer~\citep{wu2022nodeformer}, which is inspired by Performer~\citep{ChoromanskiLDSG21}, uses the Gumbel-Softmax operator as a kernel to efficiently propagate information among all pairs of nodes. SGFormer~\citep{wu2024simplifying} shows that just using a one layer transformer network can sometimes improve the results of GCN-based networks and the low memory footprint can help scale to large networks.
Perhaps most conceptually similar to our work is Skeinformer~\citep{ChenZHJJY22}, which uses sketching techniques to accelerate self-attention.

{\bf Sampling and batching techniques.} Some sampling-based methods have been used to alleviate the problem of ``neighborhood explosion.'' For instance, sampling was used in GraphSAGE~\citep{hamilton2017inductive}, which used a fixed-size sample from a neighborhood in the node aggregation step. GraphSAINT~\citep{ZengZSKP20} scales GCNs to large graphs by sampling the training graph to create minibatches.

{\bf Other.} Expander graphs were used in convolutional networks by \citet{PrabhuVN18}.

\section{Preliminaries and Notation}

{\bf Exphormer.} 
\textsc{Exphormer} is an expander-based sparse attention mechanism for graph transformers that uses $O(|V|+|E|)$ computation, where $G=(V,E)$ is the underlying input graph.
Exphormer creates an interaction graph $H$ that consists of three main components: edges from the input graph, an overlaid expander graph, and virtual nodes (which are connected to all the original nodes).

For the \emph{expander graph} component, Exphormer uses a constant-degree random expander graph, with $\mathcal O(n)$ edges. Expander graphs have several useful theoretical properties related to spectral approximation and random walk mixing, which allow the propagation of information between pairs of nodes that are distant in the input graph $G$ without explicitly connecting all pairs of nodes. The expander edges introduce many alternative short paths between the nodes and avoid the information bottleneck that can be caused by the virtual nodes.

{\bf Our model.}
We use $H$ to denote the attention pattern,
and $\mathcal N_H(i)$ the neighbors of node $i$ under that pattern. Let $\mathbf{X} = (\mathbf{x}_1, \mathbf{x}_2, \dots, \mathbf{x}_n) \in \mathbb{R}^{d\times n}$ be the matrix of $d$-dimensional embeddings for all of the $n$ nodes.
Our primary ``driver'' is then $h$-head attention:
using $\odot$ for element-wise multiplication,
\[
 \textsc{Attn}_H(\mathbf{X})_{:,i} = \mathbf{x}_i + \sum_{j=1}^h \mathbf{V}_i^j \cdot 
 \sigma\left( \left(\mathbf{E}^j  \odot \mathbf{K}^j \right)^T \mathbf{Q}_i^j + \mathbf B^j \right)
,\]

where $\mathbf{V}_i^j = \mathbf{W}_V^j \mathbf{X}_{\mathcal{N}_H(i)}$,
$\mathbf{K} = \mathbf{W}_K^j \mathbf{X}_{\mathcal{N}_H(i)}$, and $\mathbf{Q}_i^j = \mathbf{W}_Q^j \mathbf{x}_i$,
are linear mappings of the node features for the neighbors $\mathbf X_{\mathcal N_H(i)}$,
and
$\mathbf{E}^j = \mathbf{W}_E^j \mathcal{E}_{\mathcal{N}_H(i)}$
and $\mathbf B^j = \mathbf W_B^j \mathcal E_{\mathcal N_H(i)}$
are linear maps of the edge features $\mathcal E$,
which is a $d_E \times \lvert \mathcal N_H(i) \rvert$ matrix of features for the edges coming in to node $i$. Exphormer uses learnable edge features for each type of added edge, and original edge features for the graph's edges. If the graph does not have any original edge features, it uses a learnable edge feature across all graph edges. Edge features help the model distinguish the type of attention edges. Here, $\sigma$ is an activation function. In both Exphormer and our work the activation function is $\ReLU$.

In the absence of edge features, which is the case for most of the transductive datasets, including the datasets that have been used in this paper, $\mathcal{E}_e$ for any attention edge $e$ can have one of three possible representations, and so $\mathbf{E}^j$ can be computed more simply by first mapping these three types of edge features with $\mathbf{W}_E^j$ for head $j$, and then replacing the mapped values for each edge type. This simple change reduces the complexity of the Exphormer from $\mathcal{O}(md^2+nd^2)$ to $\mathcal{O}(md+nd^2)$.

Compared to prior work, we introduce $\mathbf B^j$ as a simpler route for the model to adjust the importance of different edge types. Considering Exphormer as an interpolation between MPNNs and full Transformers, the $\mathbf B^j$ model has an easier path to allow for attention scores to be close to zero for all non-graph attention edges, without restricting the performance of the attention mechanism on graph edges. Consequently, it can function roughly as an MPNN (similar to GAT) by zeroing out the non-local attention paths.
We use $d_E = d$, and have each layer output features of the same width as its input, so that each of the $\mathbf W_\cdot^j$ parameter matrices except for $\mathbf W_B^j$ are $d \times d$, and $\mathbf W_B^j$ is $d \times 1$.

As a simple illustration that $\mathbf{E}^j$ is insufficient to allow near-zero attention scores, thus highlighting the importance of $\mathbf B^j$, note that if the columns of $\mathbf{K}$ and $\mathbf{Q}$ are distributed independently and uniformly on a unit ball (e.g., under a random initialization), there is no vector $\mathbf{E}^j$ which is identical for all edges of an expander graph that can make the attention scores for all the expander edges near-zero.

{\bf Our network compared to Exphormer.} 
We use Exphormer as the base model because it provides us the flexibility to adjust the sparsity of the attention graph and to interpolate between MPNNs and full Transformers. Exphormer can model many long-range dependencies that are not modeled by MPNNs and are very expensive to model in a full Transformer. For example, one cannot train a full Transformer model in the memory of a conventional GPU device for a dataset such as Physics, which has a graph on just 34K nodes.
In our instantiation of Exphormer, we add self-loops for every node and use $d/2$ random Hamiltonian cycles to construct our expander graph as described in \citep[Appendix C.2]{shirzad2023exphormer}.
We do not add virtual nodes in our networks. (Even so, the resulting network is still a universal approximator; \citealp[Theorem E.3]{shirzad2023exphormer}).
Although the best known results for Exphormer combine sparse attention with MPNNs, in this work, we avoid the MPNN component for scalability reasons.
We also make two additional changes; see  \cref{sec:normalizing_v,sec:variable_temp}.

\section{Method}

Our method consists of a two-phase training process. The first phase trains a model we call the {\em Attention Score Estimator Network}, whose goal is to estimate the attention scores for a larger network. This model is not particularly accurate; its only goal is for each node to learn which neighbors are most important. The learned attention scores for each layer of the first network are then used to construct sparse interaction graphs for each layer in a second model, which is trained (with hyperparameter tuning for the best results) and serves as the final predictor.

{\bf Attention Score Estimator Network.} For this network, we use a width of $4$ or $8$, with just one attention head, in our training. We tune the other hyperparameters in order to have a converged training process with reasonably high accuracy, but we do not spend much time optimizing this network as it is sufficient to learn the  important neighbors for each node, i.e., edges with high attention scores. This network will be trained with as many layers as the final network we want to train.
Because it is so narrow, it has many fewer parameters and hence much less memory and time complexity, making it cheaper to train.
Moreover, we only need to do this training once per number of layers we consider, conditioned on the fact that the training converges, even if the final model has a large number of hyperparameters.
Compared to Exphormer, we use a much higher-degree expander graph: $30$ to $200$ instead of the $6$ used for most transductive graphs by \citet{shirzad2023exphormer}.
As most of the considered datasets do not have edge features,
we use a learnable embedding for each type of edge (graph edge, expander edge, or self-loop).
We also make two small changes to the architecture and the training process of this model, discussed below.
\Cref{sec:attention_score_analysis} shows experimentally that the low-width network is a good estimator of the attention scores for a large-width network.

{\bf Normalizing V.} 
\label{sec:normalizing_v}
Having a smaller attention score,
$\alpha_{ij} < \alpha_{i j'}$,
does not necessarily mean that $j$'s contribution to $i$'s new features is smaller than that of $j'$:
if $\lVert \mathbf V_j \rVert \gg \lVert \mathbf V_{j'} \rVert$,
the net contribution of $j$ could be larger.
Although Transformers typically use layer normalization, they do not typically do so after mapping $X$ to $\mathbf{V}$. We normalize the rows of $\mathbf{V}$ to have the same vector sizes for all nodes. In our experiments, normalizing to size one reduced performance significantly; however, adding a learnable global scale $s$,
so that $\mathbf V_i$ becomes $\frac{s {\mathbf{V}_i}}{{\lvert\lvert \mathbf{V_i} \rvert\rvert_2}}$,
maintained performance while making attention scores more meaningful.

\paragraph{Variable Temperature}
\label{sec:variable_temp}
One of the side goals is to have sharper attention scores, guiding the nodes to get their information from as few nodes as possible. Using temperature in the attention mechanism can do this, where logits will be divided by a temperature factor $\tau$ before being fed into a softmax. Normal attention corresponds to $\tau = 1$; smaller $\tau$ means sharper attention scores. However, setting the temperature to a small value from the beginning will make the random initialization more significant, and increase the randomness in the training process. Instead, we start with $\tau=1.0$ and gradually anneal it to $0.05$ by the end of the training. We set an initial phase for $\lambda$ epochs where we use $\tau=1$; this lets the model learn which neighbors are more important for each node slowly. We multiply $\tau$ with a factor $\gamma$ after each epoch, obtaining a temperature in epoch $t>\lambda$ of $\max(\gamma^{t-\lambda}, 0.05)$. We use $\lambda=5$  and $\gamma=0.99$ or $0.95$ depending on how fast the learning converges.

\begin{figure}[t!]
\centering
\includegraphics[width = 4.5in]{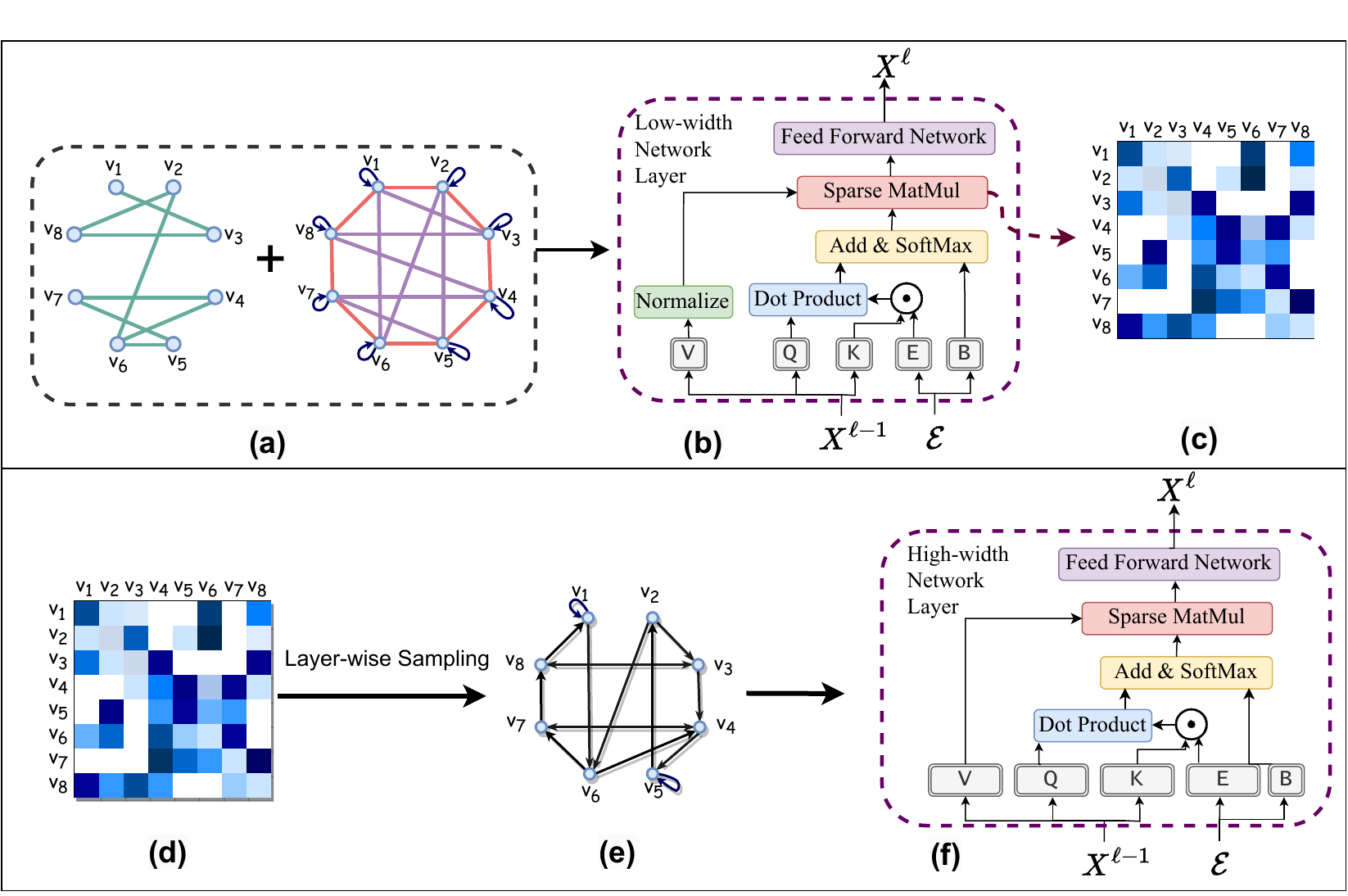}
\caption{Steps of our method.
\textbf{(a)} The attention mechanism for the attention score estimator network combines graph edges with an expander graph and self-loops. The expander graphs are constructed by combining a small number of Hamiltonian cycles -- here two, in red and in purple -- then confirming the spectral gap is large enough.
\textbf{(b)} Self-attention layers in the estimator network use this sparse attention mechanism; its self-attention layers normalize $\mathbf{V}$. 
\textbf{(c, d)} Attention scores are extracted from this network for each layer, and used to sample, in \textbf{(e)}, a sparse directed graph, which becomes the attention graph for the final network \textbf{(f)}. This network, with a much larger feature dimension, does not normalize $\mathbf{V}$.}
\label{fig:method}
\vspace{-10pt}
\end{figure}

{\bf Sparser Attention Pattern.} The memory and time complexity of Exphormer is linearly dependent on the number of edges. Also, with a small number of layers, the expander degree should be high enough to ensure a large enough receptive field for each node in order to learn the long-range dependencies. Not all these edges are equally important, and many of them will have a near-zero effect on the final embedding of each node. Reducing the number of edges can alleviate memory consumption. Additionally, a sparser pattern lets us use batching techniques for the larger graphs.
In this work, we analyze how effectively the sparser model can work and up to what factor we can sparsify. For each layer, e.g., $\ell$, we select a $\deg_{\ell}$ as a fixed degree for each node and sample without replacement according to the attention score estimator network's attention scores in each epoch of training or evaluation. 
Having the same degree for each node's attention pattern also means that attention can be calculated using (much-more-optimized) standard matrix multiplications, rather than the propagation techniques used in Exphormer and SAN \citep{kreuzer2021rethinking}.

To sparsify the graph, in each epoch, we sample a new set of edges according to the learned attention scores from the smaller network. The reason why we do this rather than a simpler strategy such as selecting top-scored edges is that in many cases, several nodes can have very similar node features. If we assume nodes $u_{1}, u_2, \dots, u_p$ from the neighbors of node $v$ have almost the same features, and if the attention scores for these nodes are $\alpha_1, \alpha_2, \dots, \alpha_p$, any linear combination of $\sum_{i=1}^p \alpha_i = \alpha$ will lead to the same representation for node $v$. If features are exactly the same, $\alpha$ will be divided between these nodes, and even if $\alpha$ is large, each node's attention score from $v$ can be small. By sampling, we have a total $\alpha$ chance of selecting any of the nodes $u_{1:p}$. In each epoch, we re-sample a new set of edges for each node from its original neighborhood.

{\bf Faster Sampling Using Reservoir Sampling.}
Sampling without replacement using default library calls is very slow, especially if few neighbors dominate the attention scores. We instead use reservoir sampling \citep{efraimidis2006weighted}, which is GPU-friendly and parallelizable. For reservoir sampling of $k$ neighbors from the neighborhood of node $i$, with attention scores $\mathbf{a} = (a_1, a_2, \cdots, a_{|\mathcal{N}_H(i)|})$, we first take a uniform random sample $\mathbf{u} = (u_1, u_2, \cdots, u_{|\mathcal{N}_H(i)|})$, where the $u_i$ are i.i.d.\ samples from $\mathrm{Uniform}(0, 1)$. Then we calculate $\frac{1}{\mathbf{a}} \odot \log (\mathbf{u})$ with element-wise multiplication, and select the indices with the top $k$ values from this list. Selecting $k$-th rank from $n$ values and pivoting has a worst-case $\mathcal{O}(n)$ time algorithm, which is much faster than the $\mathcal O(nk)$ worst case time for trial-and-error. Pseudocode is given in \cref{alg:reservoir_sampling}. The GPU-friendly version of this can be implemented by sampling for nodes in parallel, but requires forming a regular matrix for the attention scores. This can be done by extending each attention score vector to the maximum degree, or selecting a value $k' \gg k$ and first sampling $k'$ and selecting the top $k'$ attention scores from each node, making sure that the sum of the rest of the neighbor's attention scores are very near to zero. Then by forming a rectangular attention matrix, uniform sampling and element-wise multiplications are much faster on GPU, and sampling from the entire batch is much more efficient.
\hamed{I can't find the theorem related to the version of reservoir sampling we do}

\begin{algorithm}[h] 
\caption{Reservoir Sampling from a Node's Neighborhood}
\label{alg:reservoir_sampling}
\begin{algorithmic}[1]
\Require{Attention scores $\mathbf{a} = a_{i, \mathcal{N}_H(i)}^{(\ell)}$, number of neighbors to sample: $\deg_{\ell}$}
\Ensure{List of $\deg_{\ell}$ neighbors of node $i$}
\Statex
\Function{ReservoirSample}{$\mathbf{a}, \deg_{\ell}$}
    \State{$\mathbf{u} \sim \mathrm{Uniform}(0, 1)^{\abs{\mathcal{N}_{H(i)}}}$}
    \State \Return {$\operatorname{argtop_{\deg_\ell}}(\frac{1}{\mathbf{a}} \odot \log(\mathbf{u}))$}
\EndFunction
\end{algorithmic}
\end{algorithm}

{\bf Batching.} Each batch starts with a random subset of ``target'' nodes $B$.
These are the nodes whose last-layer representations we will update in this optimization step.
To calculate these representations,
we need keys and values based on the previous layer's representations for the relevant neighbors of each target node (again, sampling neighbors from the graph augmented by an expander graph).
To approximate this,
we sample $\deg_L$ neighbors for each target node.
Then we have a set of
at most $\lvert B \rvert (\deg_L + 1)$ nodes
whose representations we need to calculate in layer $L-1$;
we repeat this process,
so that in layer $\ell$ we need to compute representations for up to
$ \mathcal{Q}\llo \leq \min(|B|\prod_{i=\ell+1}^{L}(\deg_i+1), n)$ query nodes,
with $\lvert \mathcal Q\llo \rvert \deg_\ell$ attention edges.
Pseudocode is given in \cref{alg:batching}.

When the number of layers $L$ and degree $\deg_\ell$
are not too large,
this batching can be substantially more efficient than processing the entire graph.
Moreover, compared to other batching techniques, our approach selects neighbors \emph{according to their task importance}.
Except for optimization dynamics in the training process corresponding to minibatch versus full-batch training,
training with batches is identical to training with the entire sparsified graph;
if we choose a large $\deg_\ell$ equal to the maximum degree of the augmented graph,
this is exactly equivalent to SGD on the full graph,
without introducing any biases in the training procedure.
This is in stark contrast to previous approaches,
as illustrated in \cref{fig:synthetic_task}.
Unlike these prior approaches,
which typically use the full graph at inference time,
we can run inference with batch size as small as one
(trading off memory for computation).

\begin{algorithm}[H] 
\caption{Neighborhood Sampling for a Batch of Nodes}
\label{alg:batching}
\begin{algorithmic}[1]
\Require{Attention scores in each layer: $\mathbf{a} = \left\{a_{i, j}^{(\ell)} \mid \forall i \in V, j\in \mathcal{N}_H(i), , 1\leq \ell \leq L\right\}$, number of neighbors to sample in each layer: $\mathbf{deg} = \left\{\deg_1, \cdots ,\deg_L\right\}$, and a batch of nodes $B \subseteq V$} 
\Ensure{$\mathcal{Q}^{(\ell)}, \mathcal{K}^{(\ell)}, \mathcal{V}^{(\ell)}$, query, key, and value nodes in each layer}
\Statex
\Function{SampleNeighborhood}{$B, \mathbf{a}, \mathbf{deg}$}
    \State {$\mathcal{V}^{(L+1)} \gets B$}
    \For{$\ell \gets L$ to $1$}
        \State {$\mathcal{Q}^{(\ell)} \gets \mathcal{V}^{(\ell+1)}$}
        \For{$i \gets i \in \mathcal{Q}^{(\ell)}$}
            \State{$\mathcal{K}_i^{(\ell)} \gets \textsc{ReservoirSample}(\mathbf{a}_{i, \mathcal{N}_H(i)}, \deg_{\ell}$)}
        \EndFor
        \State{$\mathcal{K}^{(\ell)} \gets \bigcup_{i\in \mathcal{Q}^{\ell}} \mathcal{K}_i^{(\ell)}$}
        \State{$\mathcal{V}^{(\ell)} \gets \mathcal{Q}^{(\ell)} \bigcup  \mathcal{K}^{(\ell)}$}
    \EndFor
    \State \Return {$\left\{\left(\mathcal{V}^{(\ell)}, \mathcal{Q}^{(\ell)}, \mathcal{K}^{(\ell)}\right) \mid  1 \leq \ell \leq L\right\}$}
\EndFunction
\end{algorithmic}
\end{algorithm}

{\bf Fixed Node Degree Layers.}
Sparse matrix operations are not yet nearly as efficient as dense operations on GPU devices. Exphormer and SAN use a \texttt{gather} operation, which is memory-efficient but not time-efficient on a GPU \citep{zaheer2020big}. By normalizing the degree, instead of having $|\mathcal{Q}\llo| \deg_\ell$ separate dot products between the query and key vectors, we can reshape the key vectors to be of size $|\mathcal{Q}\llo|\times \deg_\ell \times d$ and the query is of shape $|\mathcal{Q}\llo|\times d$. Now the dot product of query and key mappings can be done using $|\mathcal{Q}\llo|$, $\deg_\ell \times d$ by $d \times 1$ matrix multiplications. This same size matrix multiplication can be done using highly optimized batch matrix multiplication operations in e.g.\ PyTorch and Tensorflow \citep{paszke2019pytorch, tensorflow2015-whitepaper}.

\subsection{Theoretical Underpinnings}
We first study the approximability of a network with a smaller hidden dimension or width. Formally, suppose that the width of a wide network is $D$.
Then there exists a network with narrow dimensions for $\W_Q$ and $\W_K$, of dimension $\mathcal{O}(\frac{\log n}{\eps^2}) \times D$ instead of $D \times D$, whose attention scores agree with those of the wide network up to $\mathcal{O(\eps)}$ error (\cref{thrm:narrow_attention}).
This reduction helps with the most intensive part of the calculation; others are linear with respect to the number of nodes $n$.
While this is not the model we use in practice, \citet[Section 4]{shirzad2024compression} explore some scenarios common in graph Transformers that allow for the existence of ``fully'' narrow networks with accurate attention scores.
They support these claims with experiments that show compressibility for some datasets we use.
This is an existence claim;
we will justify experimentally that in practice,
training a narrow network does approximate attention scores well.

We then study the sampling procedure of our sparsification method. Under certain assumptions, we show that sampling roughly $\mathcal O(n \log n / \eps^2)$ entries of the attention matrix $A$ (corresponding to sampling this many edges in the graph) suffices to form a matrix $B$ with $\norm{A - B}_2 \le \eps \norm{A}_2$, if we can access the entries of $A$ (\cref{thrm:sampling}).
We cannot actually access the matrix $A$,
but we do have attention scores $A'$ from a narrow network.
We show that if the entries of $A$ are not seriously under-estimated by $A'$, the same bound on the number of samples still holds (\cref{prop:sampling_noisy}).

\begin{table}[!hp]
\caption{Comparison of our model with other GNNs on six homophilic datasets. The reported metric is accuracy for all datasets.}
    \centering
    \label{tab:homophilicgraphs}
    \setlength\tabcolsep{5pt} %
    \scalebox{0.9}{
    \small
    \begin{tabular}{lcccccc}
    \toprule
         {\bf Model} & {\bf Computer}  & {\bf Photo} & {\bf CS}  & {\bf Physics} & {\bf WikiCS} & {\bf ogbn-arxiv} \\
         \midrule
         \textsc{GCN}& 89.65 $\pm$ 0.52&  92.70 $\pm$ 0.20 & 92.92 $\pm$ 0.12& 96.18 $\pm$ 0.07 & 77.47 $\pm$ 0.85 & 71.74 $\pm$ 0.29 \\
         \textsc{GraphSAGE}& 91.20 $\pm$ 0.29&  94.59 $\pm$ 0.14 & 93.91 $\pm$ 0.13& 96.49 $\pm$ 0.06 & 74.77 $\pm$ 0.95 & 71.49 $\pm$ 0.27\\
         \textsc{GAT}&   90.78 $\pm$ 0.13& 93.87 $\pm$ 0.11& 93.61 $\pm$ 0.14 & 96.17 $\pm$ 0.08 & 76.91 $\pm$ 0.82 & 72.01 $\pm$ 0.20\\
         \textsc{GraphSAINT}& 90.22 $\pm$ 0.15& 91.72 $\pm$ 0.13& 94.41 $\pm$ 0.09& 96.43 $\pm$ 0.05 & - & 68.50 $\pm$ 0.23\\
         \textsc{NodeFormer}& 86.98 $\pm$ 0.62 &  93.46 $\pm$ 0.35 & 95.64 $\pm$ 0.22 &  96.45 $\pm$ 0.28 &  74.73 $\pm$ 0.94 & 59.90 $\pm$ 0.42\\
         \textsc{GraphGPS}& 91.19 $\pm$ 0.54 &   95.06 $\pm$ 0.13 &  93.93 $\pm$ 0.12 &  97.12 $\pm$ 0.19 & 78.66 $\pm$ 0.49 & 70.92 $\pm$ 0.04\\
         \textsc{GOAT}& 90.96 $\pm$ 0.90 & 92.96 $\pm$ 1.48 &  94.21 $\pm$ 0.38 &  96.24 $\pm$ 0.24 & 77.00 $\pm$ 0.77 & 72.41 $\pm$ 0.40 \\
         \midrule
         \textsc{Exphormer+GCN}& 91.59 $\pm$ 0.31& 	95.27 $\pm$ 0.42 & 95.77 $\pm$ 0.15 & 97.16 $\pm$ 0.13 & 78.54 $\pm$ 0.49 & 72.44 $\pm$ 0.28\\
         \textsc{Exphormer}*& 91.16 $\pm$ 0.26& 95.36 $\pm$ 0.17& 95.19 $\pm$ 0.26&96.40 $\pm$ 0.20 & 78.19 $\pm$ 0.29 & 71.27 $\pm$ 0.27\\
        \midrule
         \textsc{Spexphormer} & 	91.09 $\pm$ 0.08 & 95.33 $\pm$ 0.49 & 95.00 $\pm$ 0.15 & 96.70 $\pm$ 0.05 & 78.2 $\pm$ 0.14 & 70.82 $\pm$ 0.24 \\
         \midrule
         Avg. Edge Percent & 7.6\% & 8.2\% & 12.8\%& 11.3\% & 8.6\% &13.7\% \\
  \bottomrule
    \end{tabular}
    }
\end{table}

\section{Experimental Results}
\begin{figure*}[ht]
\captionsetup[subfloat]{farskip=-1pt,captionskip=-1pt}
\centering
\subfloat[][]{\includegraphics[width = 1.37in]{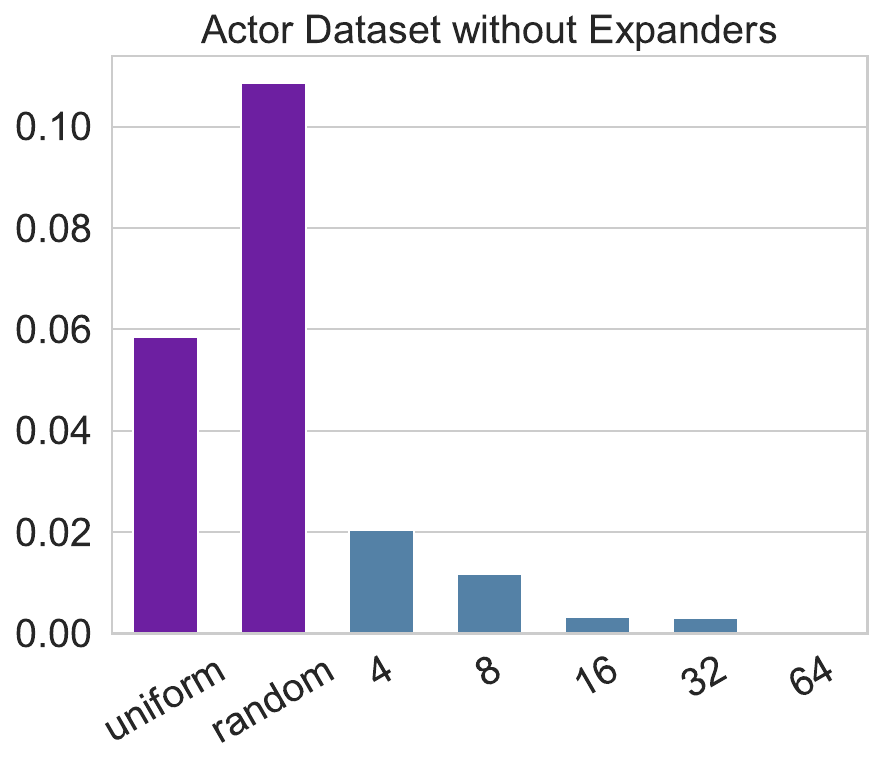}} 
\subfloat[][]{\includegraphics[width = 1.37in]{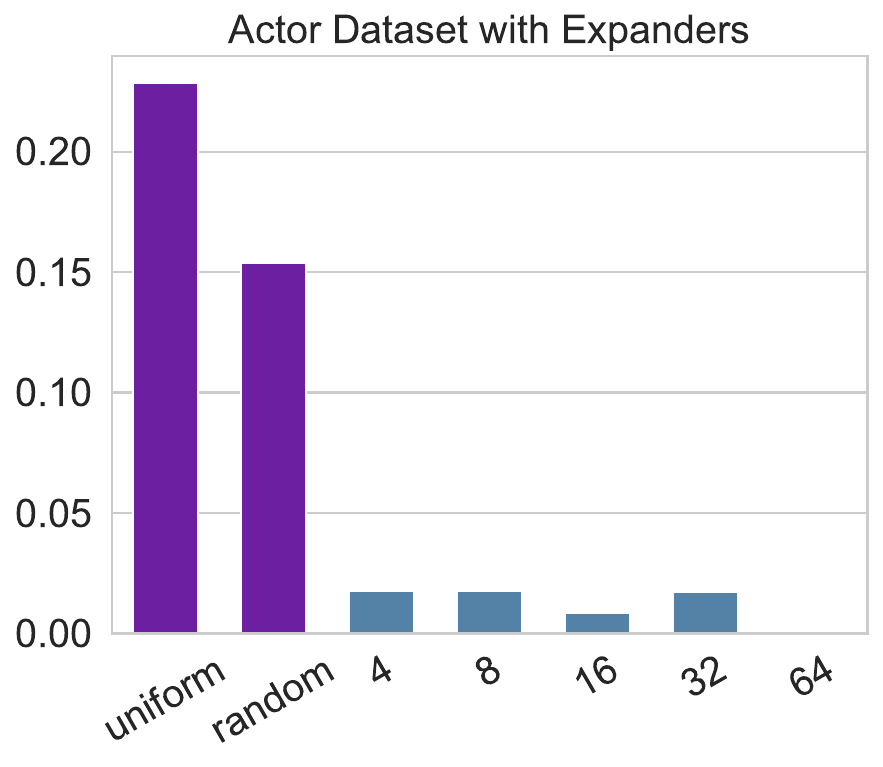}}
\subfloat[][]{\includegraphics[width = 1.37in]{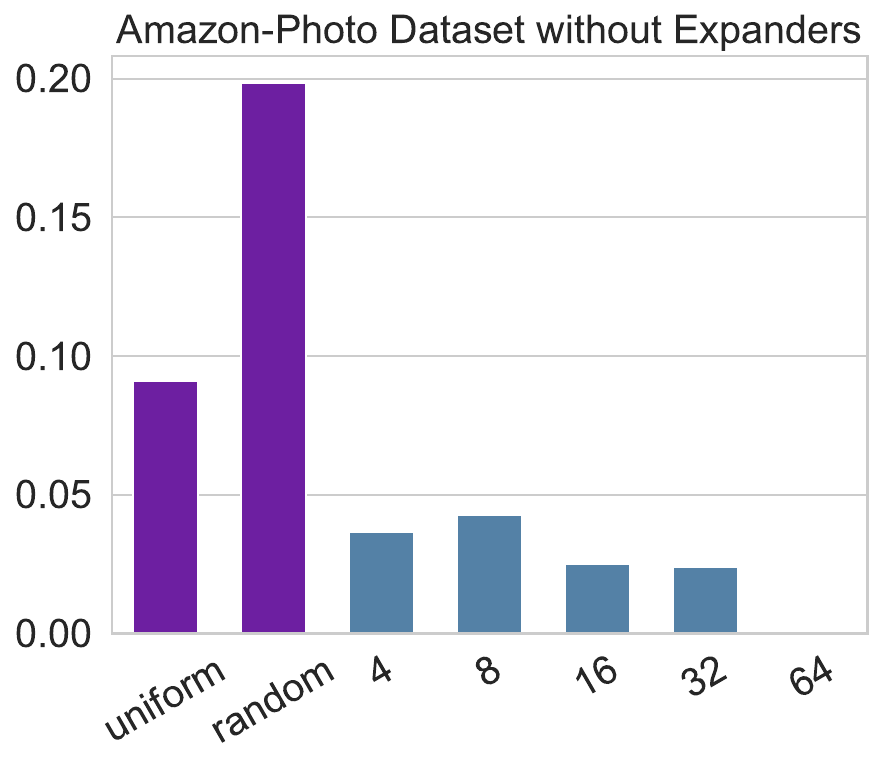}}
\subfloat[][]{\includegraphics[width = 1.37in]{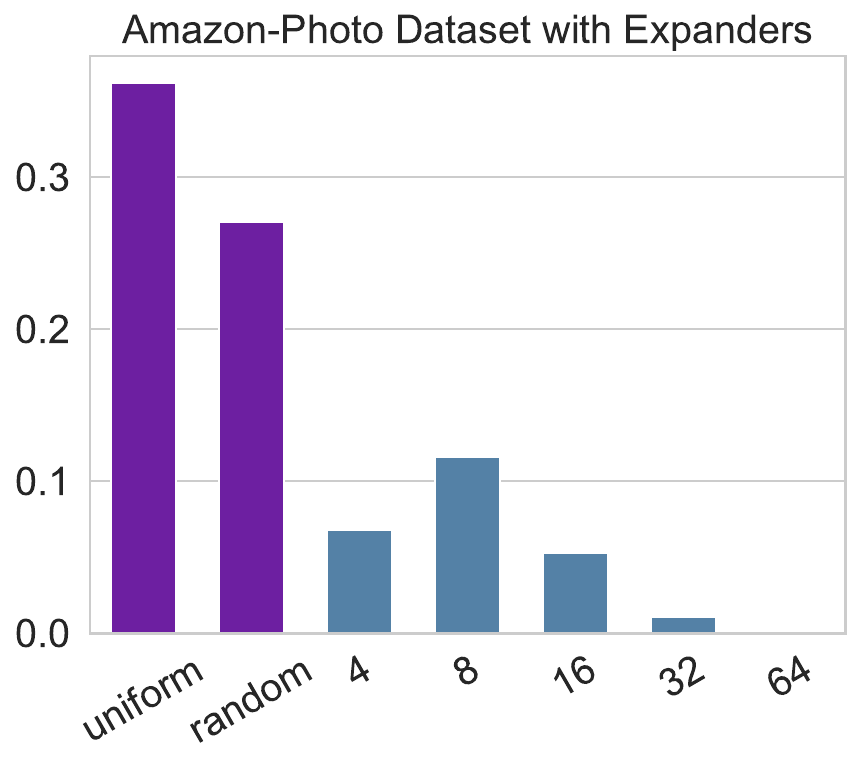}}
\caption{Energy distance between the attention scores of various networks to a network of width 64. ``Uniform'' refers to the baseline placing equal scores to each neighbor, while ``random'' refers to the baseline with uniformly distributed logits. The remaining bars refer to networks trained on the appropriately labeled width.}
\label{fig:energy_dists}
\end{figure*}

\subsection{Attention Score Estimation} \label{sec:attention_score_analysis} To show how well the smaller network estimates the attention scores for a larger network, we conduct experiments on two smaller datasets, where we can reasonably train the full network at higher width for many runs in order to estimate the distribution of the attention scores. To this end, we use the Actor \citep{lim2021large} and Photo \citep{shchur2018pitfalls} datasets. We train the network for hidden dimensions $h$ varying from 4 to 64 for both datasets.
For each $h$ we train the network $100$ times. 
We consider the distribution of attention scores for each node, and estimate the energy distance (\citealp{energy-distance}; an instance of the maximum mean discrepancy, \citealp{energy-rkhs}) for that node's attention scores across each pair of $h$ sizes.

We ran this experiment in two scenarios: first, with just graph edges, and then, by adding expander and self-loop edges. It might be that the model, just by examining the category of the edges, may give a lower score to one type, making distributions seem more similar despite not identifying a small number of important neighbors as we want. However, in the presence of only one type of edge, the model can still consistently estimate which nodes should have a higher attention score.

We compare attention scores from our model with the uniform distribution on the neighbors (each neighbor of node $i$ has score $\frac{1}{d_i}$), and to a distribution with logits uniform over $[-8, 8]$. The choice of $8$ here is because in the network we clip the logits with an absolute value higher than $8$. \Cref{fig:energy_dists} shows that even width-$4$ networks provide far superior estimates of attention scores than these baselines.

In \cref{sec:attention_score_analysis_appendix}, we extend our analysis with several experiments: examining pairwise energy distances between all pairs of hidden dimensions as well as uniform and random distributions, providing layer-wise results (\cref{sec:pairwise_dists}), analyzing the sharpness or smoothness of attention scores across layers (\cref{sec:entropy}), assessing their similarity between layers (\cref{sec:inter_layer}), and measuring precision, recall, density, and coverage in estimating the attention scores of the larger network using a smaller one (\cref{sec:prdc}). Additionally, we investigate the sum of top-$k$ attention scores (\cref{sec:top-k}) and evaluate the role of different edge types in learning representations (\cref{sec:edge_type}). Our key insights are as follows:

\textbf{{\textit{Insight 1.}}} Attention scores from a network with a smaller hidden dimension serve as a good estimator for the attention scores in a network with a higher hidden dimension.

\textbf{{\textit{Insight 2.}}} Attention scores are smoother in the first layer, and become sharper in subsequent layers.

\textbf{{\textit{Insight 3.}}} The attention scores in the layers after the first are consistently very similar to one another, but distinct from the attention scores in the first layer.

\textbf{{\textit{Insight 4.}}} The sum of the top-$k$ attention scores is substantially lower than one for many nodes, even for relatively large $k$ values such as $10$.
\begin{table}[!htp]
\caption{Comparison of our model with other GNNs on five heterophilic datasets. The reported metric is ROC-AUC ($\times$100) for the Minesweeper, Tolokers, and Questions datasets, and accuracy for all others.}
    \centering
    \label{tab:heterophilicgraphs}
    \setlength\tabcolsep{5pt} 
    \scalebox{0.87}{
    \small
    \begin{tabular}{lcccccc}
    \toprule
         {\bf Model} & {\bf Actor} & {\bf Minesweeper}  & {\bf Tolokers} & {\bf Roman-Empire} & {\bf Amazon-Ratings} & {\bf Questions}\\
         \midrule
         \textsc{GloGNN}& 36.4 $\pm$ 1.6& 51.08 $\pm$ 1.23 & 73.39 $\pm$ 1.17 & 59.63 $\pm$ 0.69 &  36.89 $\pm$ 0.14 & 65.74 $\pm$ 1.19\\
         \textsc{GCN}& 33.23$\pm$1.16& 89.75 $\pm$ 0.52 & 83.64 $\pm$ 0.67 & 73.69 $\pm$ 0.74 & 48.70 $\pm$ 0.63 & 76.09 $\pm$ 1.27\\
         \textsc{GraphGPS}&  37.1 $\pm$ 1.5 &  90.63 $\pm$ 0.67 &  83.71 $\pm$ 0.48 & 82.00 $\pm$ 0.61 &  53.10 $\pm$ 0.42 & 71.73 $\pm$ 1.47 \\
         \textsc{NAGphormer}& - & 84.19 $\pm$ 0.66 & 78.32 $\pm$ 0.95 & 74.34 $\pm$ 0.77 &  51.26 $\pm$ 0.72 & 68.17 $\pm$ 1.53\\
         \textsc{NodeFormer}& 36.9 $\pm$ 1.0 & 86.71 $\pm$ 0.88 & 78.10 $\pm$ 1.03 & 64.49 $\pm$ 0.73 & 43.86 $\pm$ 0.35 & 74.27 $\pm$ 1.46\\
         \textsc{GOAT}& - & 81.09 $\pm$ 1.02 &  83.11 $\pm$ 1.04 & 71.59 $\pm$ 1.25 &  44.61 $\pm$ 0.50 &  75.76 $\pm$ 1.66\\
         \midrule
         \textsc{Exphormer+GAT}& 38.68 $\pm$ 0.38 &  90.74 $\pm$ 0.53 & 83.77 $\pm$ 0.78 & 89.03 $\pm$ 0.37 & 53.51 $\pm$ 0.46 & 73.94 $\pm$ 1.06\\
         \textsc{Exphormer}*& 39.01 $\pm$ 0.69 & 92.26 $\pm$ 0.56 & 	83.53 $\pm$ 0.28 & 84.91 $\pm$ 0.25 & 46.80 $\pm$ 0.53 & 73.35 $\pm$ 1.78\\
        \midrule
         \textsc{Spexphormer} & 	38.59 $\pm$ 0.81 & 	90.71 $\pm$ 0.17 & 83.34 $\pm$ 0.31  & 87.54 $\pm$ 0.14 & 50.48 $\pm$ 0.34 & 73.25 $\pm$ 0.41\\
          \midrule
         Avg. Edge Percent& 5.8\% & 17.8\% & 8.9\% & 31.1\% & 15.3\% & 13.8\%\\
  \bottomrule
    \end{tabular}
    }
\end{table}

\begin{minipage}{\textwidth}
  \begin{minipage}[b]{0.49\textwidth}
    \centering
    \includegraphics[width=2.42in]{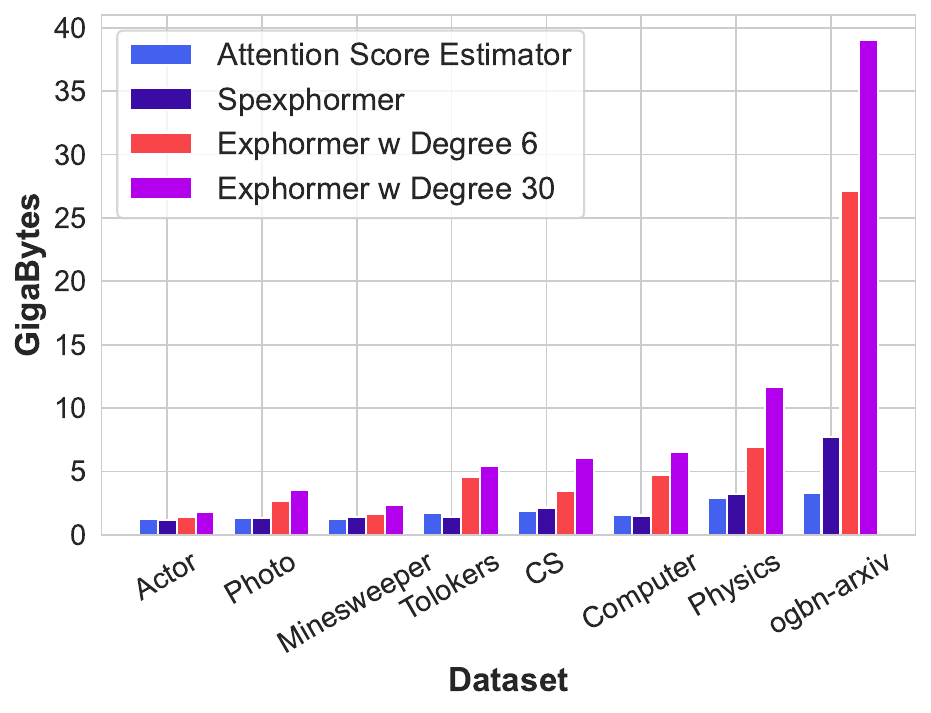}
    \vspace{-0.1in}
    \captionof{figure}{Memory usage comparison: Attention Score Estimator network and Spexphormer vs. Exphormer with expander degrees 6 and 30. Exphormer with degree 30 for the ogbn-arxiv dataset could not fit into the memory of a 40GB GPU device, and thus the number here is a lower bound.}
    \label{fig:memory}
  \end{minipage}
  \hfill
  \begin{minipage}[b]{0.49\textwidth}
    \centering
    
    \small
    \scalebox{0.7}{
    \begin{tabular}{lccc}
    \toprule
         {\bf Model} & {\bf ogbn-proteins} & {\bf Amazon2M}  & {\bf Pokec}*
         \\
         \midrule
         \textsc{MLP} & 72.04 $\pm$ 0.48 & 63.46 $\pm$ 0.10 & 60.15 $\pm$ 0.03  \\
\textsc{GCN} & 72.51 $\pm$ 0.35 & 83.90 $\pm$ 0.10 & 62.31 $\pm$ 1.13  \\
\textsc{SGC} & 70.31 $\pm$ 0.23 & 81.21 $\pm$ 0.12 & 52.03 $\pm$ 0.84   \\
\textsc{GCN-NSampler} & 73.51 $\pm$ 1.31 & 83.84 $\pm$ 0.42 & 63.75 $\pm$ 0.77 \\
\textsc{GAT-NSampler} & 74.63 $\pm$ 1.24 & 85.17 $\pm$ 0.32 & 62.32 $\pm$ 0.65 \\
\textsc{SIGN} & 71.24 $\pm$ 0.46 & 80.98 $\pm$ 0.31 & 68.01 $\pm$ 0.25 \\
\textsc{NodeFormer} & 77.45 $\pm$ 1.15 & 87.85 $\pm$ 0.24 & 70.32 $\pm$ 0.45 \\
\textsc{SGFormer} & 79.53 $\pm$ 0.38 & 89.09 $\pm$ 0.10 & 73.76 $\pm$ 0.24 \\
\textsc{Spexphormer} & \textbf{80.65 $\pm$ 0.07} & \textbf{90.40 $\pm$ 0.03}& \textbf{74.73 $\pm$0.04} \\
  \midrule
  \multicolumn{4}{c}{Memory Information for \textsc{Spexphormer}} \\
  Memory (MB) & 2232 & 3262& 2128\\
  Batch Size & 256 & 1000& 500 \\
  Hidden Dimension & 64& 128 & 64 \\
  Number of layers & 2 & 2 & 2\\
  Number of Parameters & 79,224& 300,209 & 83,781 \\
  \bottomrule
    \end{tabular}
    }
      \captionof{table}{Comparative results on large graph datasets, with ROC-AUC($\times$100) reported for the ogbn-proteins dataset and accuracy for all others. GPU memory usage, batch sizes, hidden dimensions used to obtain these numbers, and the total number of parameters have been added at the bottom of the table.}
      \label{tab:largegraphs}
    \end{minipage}
  \end{minipage}

\subsection{Model Quality}
We conduct experiments on twelve medium-sized graphs, including six homophilic datasets: CS, Physics, Photo, Computer \citep{shchur2018pitfalls}, WikiCS \citep{mernyei2020wiki}, and ogbn-arxiv \citep{ogbPaper}; and six heterophilic datasets: Minesweeper, Tolokers, Roman-empire, Amazon-ratings, Questions \citep{platonov2023critical}, and Actor \citep{lim2021large}. 

For the CS, Physics, Photo, and Computer datasets, we use a random train/validation/test split of 60\%/20\%/20\%. For WikiCS and ogbn-arxiv we follow the standard data split provided by the original source. For the Actor dataset, we use a 50\%/25\%/25\% split following \cite{wu2022nodeformer}. For the Minesweeper, Tolokers, Roman-empire, Amazon-ratings, and Questions datasets, we use the standard split from \cite{platonov2023critical}. Results for these experiments are provided in \cref{tab:homophilicgraphs,tab:heterophilicgraphs}. The \textsc{Exphormer} model presented in the tables refers to the attention mechanism of \textsc{Exphormer} without incorporating any MPNN components. Interestingly, the results on the \textit{Roman-Empire} and \textit{Amazon-Ratings} datasets revealed that removing certain edges led to better performance compared to simply adding an expander layout.

In these medium-sized datasets, we are able to train the full Exphormer model. Our goal is to determine the extent of performance reduction when using two memory-efficient networks to estimate the original network. Results show that the two memory-efficient networks can efficiently estimate the original network, enabling us to scale the {Exphormer} to larger graph datasets. We compare the maximum required memory of the attention score estimator and final networks with that of the corresponding {Exphormer} model in \cref{fig:memory}.

We then experiment on large graph datasets: ogbn-proteins, Amazon2M \citep{ogbPaper}, and Pokec \citep{takac2012data}. The results provided in \cref{tab:largegraphs} demonstrate superior performance of our model despite limited memory constraints. We follow the standard data split for the ogbn-proteins dataset and follow \cite{wu2024simplifying} for the dataset split on the Amazon2M and Pokec datasets, with 10\%/10\%/80\% and 50\%/25\%/25\% train/validation/test ratios. We emphasize that this split differs from the original dataset split used by many other works, making those numbers incomparable.

In all our experiments, we train the smaller network once, and then for the second network, we always use the same initial network's learned attention scores. Attention scores are collected from the network training step with the highest validation accuracy.

We use a subset of the following models in each of our tables as baselines, depending on the type of the dataset and scalability level of the models, GCN \citep{kipf2016semi}, GraphSAGE \citep{hamilton2017inductive}, GAT \citep{velickovic2018graph}, GraphSAINT \citep{ZengZSKP20}, Nodeformer \citep{wu2022nodeformer}, Difformer \citep{wu2023difformer}, SGFormer \citep{wu2024simplifying}, GraphGPS \citep{rampavsek2022recipe}, GOAT \citep{kong2023goat}, GloGNN \citep{li2022finding}, SGC \citep{wu2019simplifying}, NAGphormer \citep{nagphormer22}, Exphormer \citep{shirzad2023exphormer}, and SIGN \citep{frasca2020sign}. We borrow most of the baseline numbers in the tables from \citet{wu2024simplifying, deng2024polynormer}. \hamed{Should we explain why we're not comparing with the Polynormer?}

\subsection{Ablation Studies}
We benchmark the effect of different parts of the model in \cref{tab:ablation}.
Spexphormer-uniform, rather than sampling based on the estimated attention scores, samples uniformly from the augmented graph;
this is always worse than attention-based sampling, but the gap is larger for some datasets than others.
Spexphormer-max takes the edges with the highest attention scores,
rather than sampling;
this again performs somewhat worse across datasets.
Spexphormer w.o.\ temp uses a constant temperature of 1 in the initial attention score estimator network;
Spexphormer w.o.\ layer norm removes our added layer normalization.
These changes are smaller, and in one case layer normalization makes the results worse.
Across the four datasets, however, it seems that both temperature and layer norm help yield more informative and sparser attention scores.

\begin{table}
\centering
\caption{Ablation studies on two homophilic and two heterophilic datasets. Metrics: accuracy for Photo and Computer, ROC-AUC (×100) for Tolokers and Minesweeper. For the initial network, we report the result for the network used for training the Spexphormer and thus, there is no confidence interval for them.}
\label{tab:ablation}
\small
\scalebox{0.8}{
\begin{tabular}{lcccc}
\toprule
Model/Dataset               & Computer  & Photo        & Minesweeper  & Tolokers      \\
\midrule
Initial Network & 85.23 & 91.70 & 85.67 & 80.16 \\
\midrule
Spexphormer-uniform  & 86.65 $\pm$ 0.46 & 94.21 $\pm$ 0.22 & 84.15 $\pm$ 0.22 & 82.56 $\pm$ 0.17  \\
Spexphormer-max & 89.31 $\pm$ 0.31 & 95.07 $\pm$ 0.20 & 87.92 $\pm$ 0.26 & 80.85 $\pm$ 0.23 \\
Spexphormer w.o. temp  & 89.05 $\pm$ 0.35 & 95.30 $\pm$ 0.16& 90.02 $\pm$ 0.02 & 83.34 $\pm$ 0.13 \\
Spexphormer w.o. layer norm &  89.70 $\pm$ 0.25 & 94.91 $\pm$ 0.18 & 89.65 $\pm$ 0.10 & \bf{84.06 $\pm$ 0.10} \\   
\midrule
Spexphormer   & \bf{91.09 $\pm$ 0.08} & \bf{95.33 $\pm$ 0.49}& \bf{90.71 $\pm$ 0.17} & 83.34 $\pm$ 0.13  \\
\bottomrule
\end{tabular}
}
\end{table}

\section{Conclusion \& Limitations}
We analyzed the alignment of the attention scores among models trained with different widths. We found that the smaller network's attention score distributions usually align well with the larger network's. We also theoretically analyzed the compressibility of the larger Graph Transformer models. Based on these observations, we used a sampling algorithm to sparsify the graph on each layer. As a result of these two steps, the model's memory consumption reduces significantly, while achieving a competitive accuracy. This strategy also lets us use novel batching techniques that were not feasible with expander graphs of a large degree. Having a regular degree enables using dense matrix multiplication, which is far more efficient with current GPU and TPU devices.

While our method successfully scales to datasets with over two million nodes, it relies on large CPU memory for the attention score estimation for these datasets.
For extremely large datasets, this is still infeasible without highly distributed computation.
Estimated attention scores can be shared and used for training various networks based on attention scores, however, so this only needs to only be computed once per dataset and depth.
An area for potential future work is to combine sampling with simultaneous attention score estimation in a dynamic way, scaling this estimation to larger graphs.

\begin{ack}
This work was supported in part by the Natural Sciences and Engineering Resource Council of
Canada,
the Fonds de Recherche du Québec - Nature et technologies (under grant ALLRP-57708-2022),
the Canada CIFAR AI Chairs program,
the BC DRI Group, Calcul Québec, Compute Ontario, and the Digital
Resource Alliance of Canada.
Honghao Lin was supported in part by a Simons Investigator Award, NSF CCF-2335412, and a CMU Paul and James Wang Sercomm Presidential Graduate Fellowship.
\end{ack}

\bibliography{bibliography}
\bibliographystyle{apalike}

\newpage

\appendix
\onecolumn

\section{Notation Table}

\begin{table}[h]
\label{tab:notations}
\caption{A summary of the notation used in this paper. The hat notation always refers to a compressed network equivalent of a vector or matrix from the reference network.}
\centering
\scalebox{0.9}{
\begin{tabular}{l|l} 
\toprule
Notation & Definition  \\ 
\hline
    $n$ & The number of nodes in the graph \\
    $m$ & The number of attention edges in total, including graph and expander edges \\
    $d$     &     Hidden dimension of a narrow network     \\
    $D$     &     Hidden dimension of the original large graph  \\
    $L$     &     The total number of layers  in the network    \\
    $\ell$     &     Arbitrary layer index     \\
    $\mathbf{V}$   &     Value mapping of the vectors in the attention mechanism     \\
    $\mathbf{Q}$   &     Query mapping of the vectors in the attention mechanism \\
    $\mathbf{K}$   &     Key mapping of the vectors in the attention mechanism \\
    $\mathbf{W}_\cdot^{(\ell)}$     &     Weight matrix of mapping such as key, query, value, edge features, or bias in layer $\ell$     \\
    $\widehat{\mathbf{W}}_\cdot^{(\ell)}$     &  Low dimensional network's weight matrix for a mapping in layer $\ell$     \\
    $\M_\cdot$ & A linear mapping matrix (usually from the higher dimension to the smaller)\\
    $\ReLU$     &    Rectified Linear Unit    \\
    $\Ho\llo$ & Output of layer $\ell-1$ from the reference network \\
    $\Hb\llo$ & A low-rank estimation of $\Ho\llo$ \\
    $\Hh\llo$ & Output of layer $\ell-1$ from a compressed network \\
    $h_i\llo$ & column $i$ of matrix $\Ho\llo$ \\
    $a_{ij}^{(\ell)}$ & The Attention score between nodes $i$ and $j$ in layer $\ell$\\
    $\hat{a}_{ij}^{(\ell)}$ & The attention score between nodes $i$ and $j$ in layer $\ell$ from a smaller network\\
\bottomrule
\end{tabular}
}
\end{table}

\section{Dataset Descriptions} \label{sec:datasetdesc}
Below, we provide descriptions of the datasets on which we conduct experiments. A summarized statistics of these datasets have been provided in \cref{table:datasets}.

\paragraph{Amazon datasets}
Amazon Computers and Amazon Photo are Amazon co-purchase graphs. Nodes represent products purchased. An edge connects a 
pairs of products purchased together.  Node features are bag-of-words encoded reviews of the products. Class labels are the product category.

\paragraph{Amazon-Ratings}
The Amazon-ratings is an Amazon co-purchasing dataset. Each node represents a product and the edges are between the nodes purchased together frequently. Node features are the average of word embeddings from the product description. The task is to predict the average rating of the product.

\paragraph{Amazon2M}
Amazon2M dataset is a graph from the co-purchasing network. Each node represents an item. Edges between items represents products purchased together. The node features are generated from the product description. The node labels are from the top-level categories the product belongs to.

\paragraph{WikiCS}
WikiCS contains pages from Wikipedia. Each node represents an article from Wikipedia related to the Computer Science field. Edges represent the hyperlinks between the articles. The node features are the average of the word embeddings from the articles. The task is to classify the nodes into ten different branches of the field. 

\paragraph{Actor dataset}
The actor dataset is created by the actor-only subgraph of a larger graph of actor, director, writer, and film co-occuring on a Wikipedia page, limited to English-language films.  Each node corresponds to an actor. Edges denote co-occurence on a Wikipedia page. Node features are based on the terms in the actor's page. The prediction task is categorizing into one of five categories \citep{pei2020geom}.

\paragraph{Roman-Empire}
This dataset is a graph constructed from the ``Roman Empire'' article from Wikipedia. Each node is a word from this text. Two words are connected to each other if they follow each other in the text, or they are connected in the dependency tree of the sentence. The task is to predict the syntactic role of the word in the sentence. Graph is highly sparse and heterophilic.

\paragraph{Coauthor datasets} 
The datasets, CS and Physics are co-authorship graphs from Microsoft Academic Graph.  The nodes represent the authors
and an edge connects two authors who share a paper.  The node features are the keywords in the papers.
The class represents the active area of study for the author.

\paragraph{ogbn-arxiv} \citep{ogbPaper}
The ogbn-arxiv dataset is from OGBN datasets. The nodes represents the papers and edges represent the citations between the papers. 
Nodes are 128-dimensional feature vector that is an average of the embeddings of words in the title and abstract.
The prediction task is to identify the category of the 40 subject areas.

\paragraph{ogbn-proteins dataset}
The ogbn-proteins dataset is an undirected graph with edge weights and types based on species. The nodes represent proteins from eight different species. Edges indicate various biologically meaningful associations between the proteins (e.g., co-expression, homology etc.).
The edges are eight-dimensional, with each dimension having a value from [0,1] indicates the confidence score. The prediction task is a multi-label binary classification among 112 labels --- to predict the presence of protein functions. The performance measurement is the average of ROC-AUC scores across the 112 tasks.

\paragraph{Minesweeper}
The dataset is a graph representation of the 100x100 grid from the Minesweeper game. A node represents a cell and the edges connect a node to its eight neighboring cells. 20\% of the nodes are marked as mines. The features of the nodes are the one-hot encoding of the mines among the neighbors. For 50\% of the nodes the features are unknown and indicated by a separate binary feature.

\paragraph{Tolokers}
Tolokers is a graph representation of the workers in a crowd-sourcing platform, called Toloka. Each node represents a worker. Two nodes are connected if the workers have worked on the same task. Node features are based on the worker's task performance statistics and other profile information. The task is to predict which nodes have been banned for a project.

\paragraph{Questions}
This dataset is derived from the Yandex Q question-answering platform, focusing on interactions among users interested in the topic of medicine from September 2021 to August 2022. Nodes represent users, and edges denote answers given to another user’s questions. Node features include fastText-based embeddings of user descriptions, supplemented by a binary indicator for missing descriptions. The task is to predict user activity status at the end of the period.

\paragraph{Pokec}
Pokec is a large-scale social network dataset. Nodes represents users of the network. Nodes features include profile data like geographical region, age etc. The task is to predict the gender of users based on the graph.

\begin{table}[htp]
    \centering
    \caption{Dataset statistics. The reported number of edges is the number of directed edges, which will be twice the number of actual edges for the undirected graphs.}
    \fontsize{8.5pt}{8.5pt}\selectfont
    \setlength\tabcolsep{6pt} %
    \scalebox{1}{
    \begin{tabular}{lcccccc}
    \toprule
         {\bf Dataset} & {\bf Nodes} & {\bf Edges} & {\bf Average Degree} & {\bf Node Features} & {\bf Classes} &{\bf Metric} \\
Amazon Photo & 7,487 & 238,162 & 31.13 & 745 & 8 & Accuracy \\
Coauthor Physics & 34,493 & 495,924 & 14.38 & 8,415 & 5 & Accuracy \\
Amazon Computer & 13,381 & 491,722 & 35.76 & 767 & 10 & Accuracy \\
Coauthor CS & 18,333 & 163,788 & 8.93 & 6,805 & 15 & Accuracy \\
WikiCS & 11,701 &  431,726 & 36.90 & 300 & 10 & Accuracy \\
ogbn-arxiv & 169,343 & 2,332,486 & 13.77 & 128 & 40 & Accuracy\\
\midrule
Actor & 7,600 & 33,391 & 4.39 & 932 & 5 & Accuracy\\
Minesweeper & 10,000 & 78,804 & 7.88 & 7 & 2 & AUC\\ 
Tolokers & 11,758 & 1,038,000 & 88.28 & 10 & 10 & AUC \\
Roman-Empire &  22,662 & 65,854 & 2.91 & 300 & 18 & Accuracy \\
Amazon-Ratings & 24,492 & 186,100 & 7.60 & 300 & 5 & Accuracy \\
Questions & 48,921 & 307,080 & 6.28 & 301 & 2 & AUC \\
\midrule 
Pokec & 1,632,803 & 30,622,564 & 18.75 & 65 & 2 & AUC\\
ogbn-proteins & 132,534 & 79,122,504 & 597.00 & 8 & 112 & AUC\\
Amazon2M & 2,449,029 & 123,718,280 & 50.52 & 100 & 47 & AUC\\
    \bottomrule
    \end{tabular}
    }
     \label{table:datasets}
\end{table}

\section{More Experiments}

\subsection{Time-Memory Trade-off}  
One advantage of our method is its ability to trade time for memory without sacrificing accuracy. \cref{fig:mem_time_tradeoff} illustrates this trade-off on two datasets: ogbn-proteins and arxiv. In these experiments, all hyperparameters are kept constant, with the only variation being the batch size. The results demonstrate that memory usage and runtime can be adjusted without introducing bias into the training process. 

It is important to note that in random subset batching, the average degree of nodes and the number of edges included in the training process are closely related to the batch size. A very small batch size relative to the graph size can randomly exclude a significant portion of the graph's edges during training, potentially ignoring critical edges without considering their importance.

\begin{figure*}[h]
\captionsetup[subfloat]{farskip=-1pt,captionskip=-1pt}
\centering
\subfloat[][]{\includegraphics[width = 2.7in]{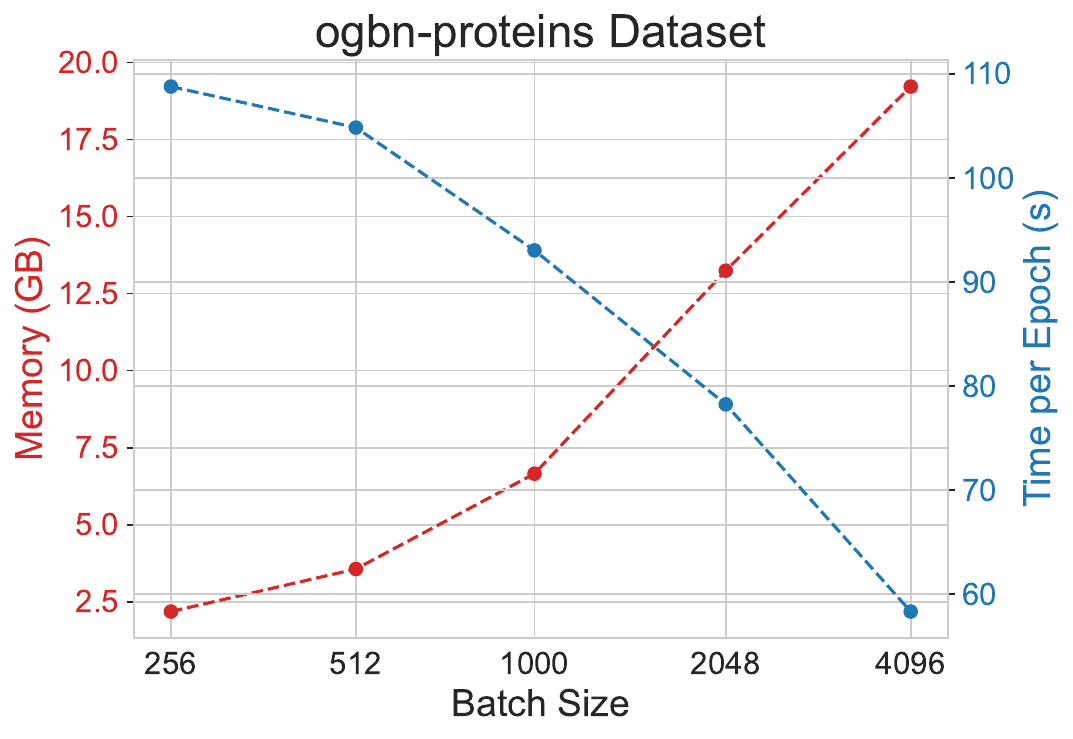}} 
\hspace{0.1 in}\subfloat[][]{\includegraphics[width = 2.7in]{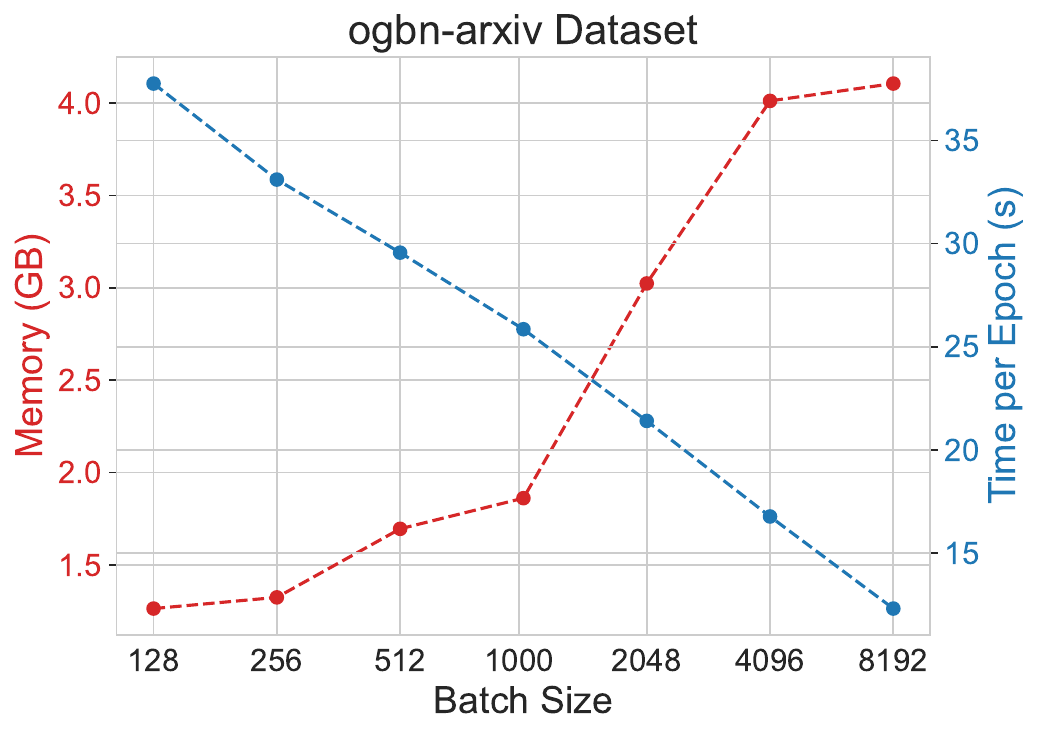}}
\caption{The memory and runtime trade-off for the ogbn-proteins and ogbn-arxiv datasets. The plot demonstrates that memory and time can be effectively exchanged in our approach. The reported runtime includes the whole process of preprocessing the batches, train, and validation on validation and test sets. All experiments were conducted on a V100 GPU with 32GB of memory.}
\label{fig:mem_time_tradeoff}
\end{figure*}

\subsection{Neighborhood Expansion}  
The level of neighborhood expansion significantly impacts the efficiency of our model. As described in \cref{alg:batching}, neighborhood expansion begins from the final nodes for which we require representations in the final layer and proceeds backward through the layers, sampling neighbors based on attention scores at each layer. 

We analyze the number of nodes across four datasets: Amazon-Photo, Coauthor-CS, Minesweeper, and Tolokers, to observe how the number of nodes increases as we trace back through the layers. This experiment is conducted with varying sampling degrees per layer, and the results are summarized in \cref{fig:neighborhood_expansion}. In all experiments, we assume that representations are needed for $10$ final nodes. We sample $100$ times of these $10$ random seed nodes and plot average and standard deviations of the neighborhood node counts.

The process has an upper bound, which is the total number of nodes in the graph. As the number of sampled nodes approaches this limit, the likelihood of encountering new nodes decreases. We compare these results with full-neighborhood sampling methods, as employed in k-hop neighborhood-induced subgraph techniques, and demonstrate that in the presence of expander graphs, this neighborhood can rapidly encompass the entire graph. The impact of limited neighborhood sampling becomes even more pronounced on extremely large graphs.

\begin{figure}[h]
    \centering
    \includegraphics[width=\linewidth]{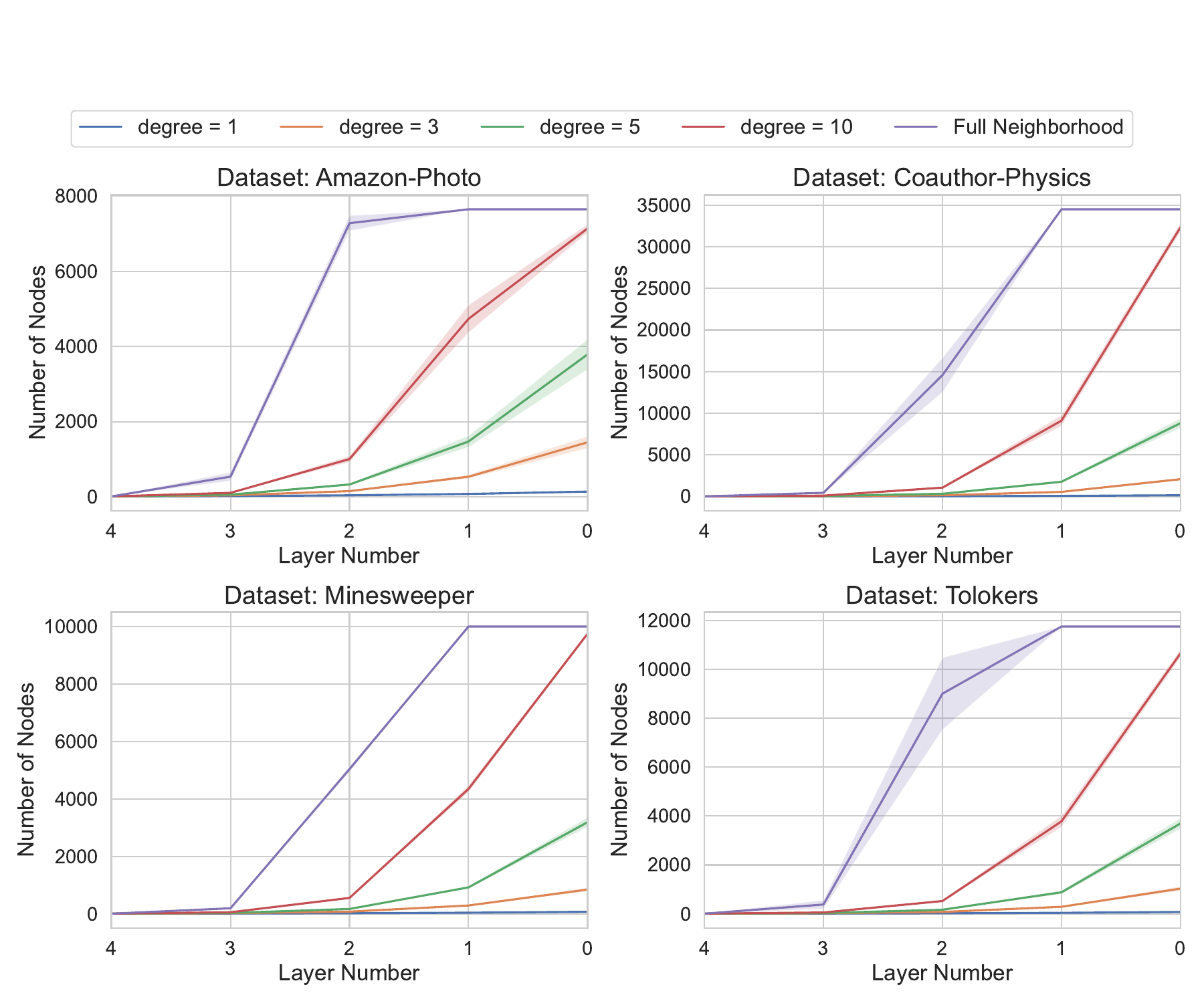}
    \caption{The neighborhood expansion of the graph is analyzed to determine the number of nodes required in each layer to obtain representations for $10$ nodes in the final layer. This is compared between Spexphormer's degree-based sampling and full-neighborhood selection. The shadowed regions in the plot represent the $95\%$ confidence intervals, calculated from $100$ iterations of sampling the ten final nodes.}
    \label{fig:neighborhood_expansion}
\end{figure}

\subsection{Memory and Runtime with Graph Size}  
Inspired by \cite{finkelshtein2024learning}, we compare the memory and runtime of our method to a Graph Convolutional Network (GCN) during a forward pass. In many real-world scenarios, the average degree of nodes is not constant and tends to scale with the graph size. One advantage of our method is its ability to subsample neighborhoods for each node, identifying a small yet representative set of neighbors. While GCN is a more computationally efficient network, we demonstrate that, even with a small but superlinear growth in neighborhood size, the memory and runtime requirements of GCN can surpass those of our method, which employs a sparse but regular self-attention layer with a fixed neighborhood size.  

In these experiments, we evaluate different growth factors for the GCN and varying neighborhood sampling sizes for our sparse self-attention method. For these comparisons, no batching is used; the entire process operates on the whole graph. Both models consist of a single layer of the corresponding network followed by a linear layer that maps the values to dimension $1$. We use a hidden dimension of $128$. In this setup, the GCN has approximately $16K$ parameters, while the self-attention layer in our method has about $65K$ parameters.  

We vary the number of nodes from $10K$ to $50K$, using an Erdős-Rényi distribution with specified probabilities $p$, denoted as $ER(n, p)$. Here, $n$ represents the number of nodes, and $p$ is the probability that any pair of nodes is independently connected. In the GCN model input, $p$ varies with $n$ and can also be viewed as a function of $n$. The average degree of a node in this model is $pn$. For the Spexphormer model, we sample $d$-regular graphs as input. Node features are drawn from $\mathcal{N}(0, I_{128})$, where $I_{128}$ is the $128$-dimensional identity matrix.

The results, shown in \cref{fig:runtime_memory_synth}, indicate that except for a constant average degree in the GCN model, the memory and runtime growth rates are higher for the GCN under all other configurations. For sufficiently large and dense graphs, our method proves to be significantly more efficient in both memory and runtime.

\begin{figure*}[h]
\captionsetup[subfloat]{farskip=-1pt,captionskip=-1pt}
\centering
\subfloat[][]{\includegraphics[width = 2.8in]{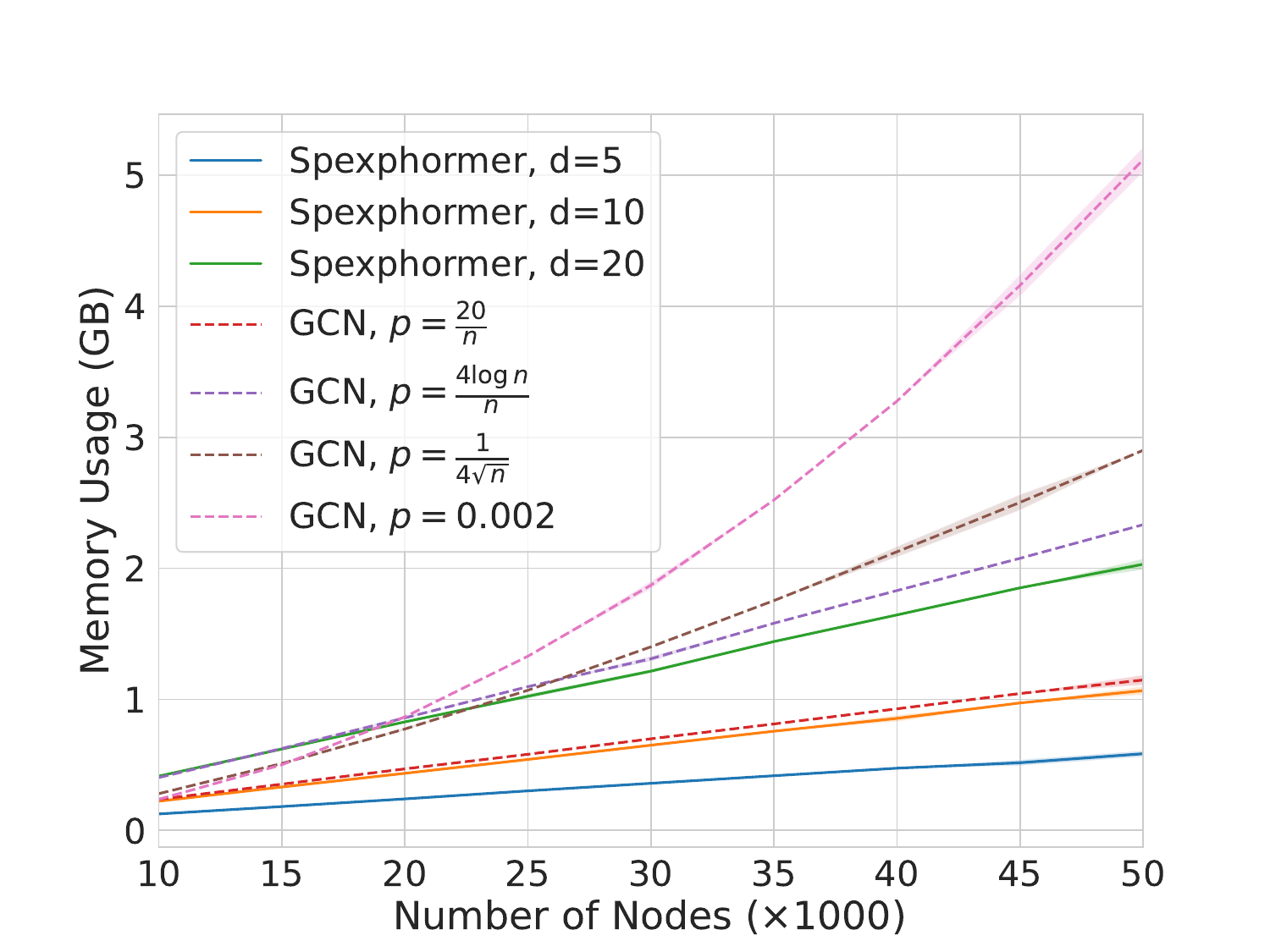}} 
\subfloat[][]{\includegraphics[width = 2.8in]{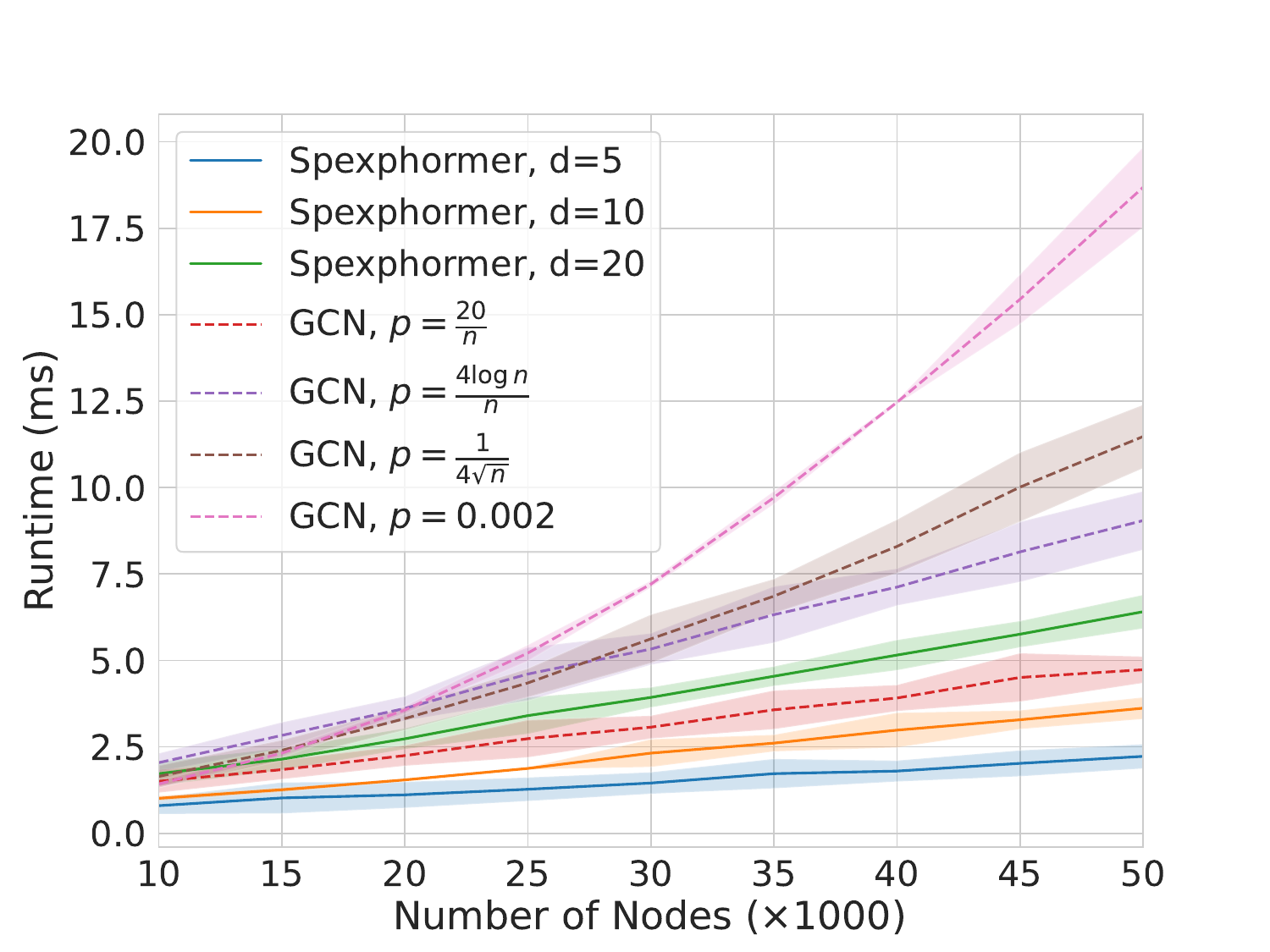}}
\caption{The memory and runtime comparison between our model and the GCN demonstrates that our model, with sparsification, significantly outperforms even a very simple GCN model on a forward pass.}
\label{fig:runtime_memory_synth}
\end{figure*}

\subsection{Accuracy, Memory, and Runtime with Sampling Degree}  
For four datasets—Tolokers, Minesweeper, Amazon-Photo, and Coauthor-CS—we analyze how accuracy, memory usage, and runtime change as the sampling degree is varied. In this experiment, all hyperparameters are fixed except for the sampling degree $deg_{\ell}$, which is kept consistent across all layers to simplify the analysis. The results are shown in \cref{fig:runtime_memory_by_deg}, where we plot both the Accuracy/AUC results and the memory/runtime metrics.  

For more heterophilic datasets, larger neighborhood sampling generally improves performance; however, the improvement becomes marginal beyond a certain point, while memory usage and runtime continue to increase linearly. For homophilic datasets, a very small neighborhood size is sufficient, and increasing the neighborhood size further does not provide noticeable benefits.

\begin{figure*}[h]
\captionsetup[subfloat]{farskip=-1pt,captionskip=-1pt}
\centering
\subfloat[][Tolokers AUC]{\includegraphics[width = 2.18in]{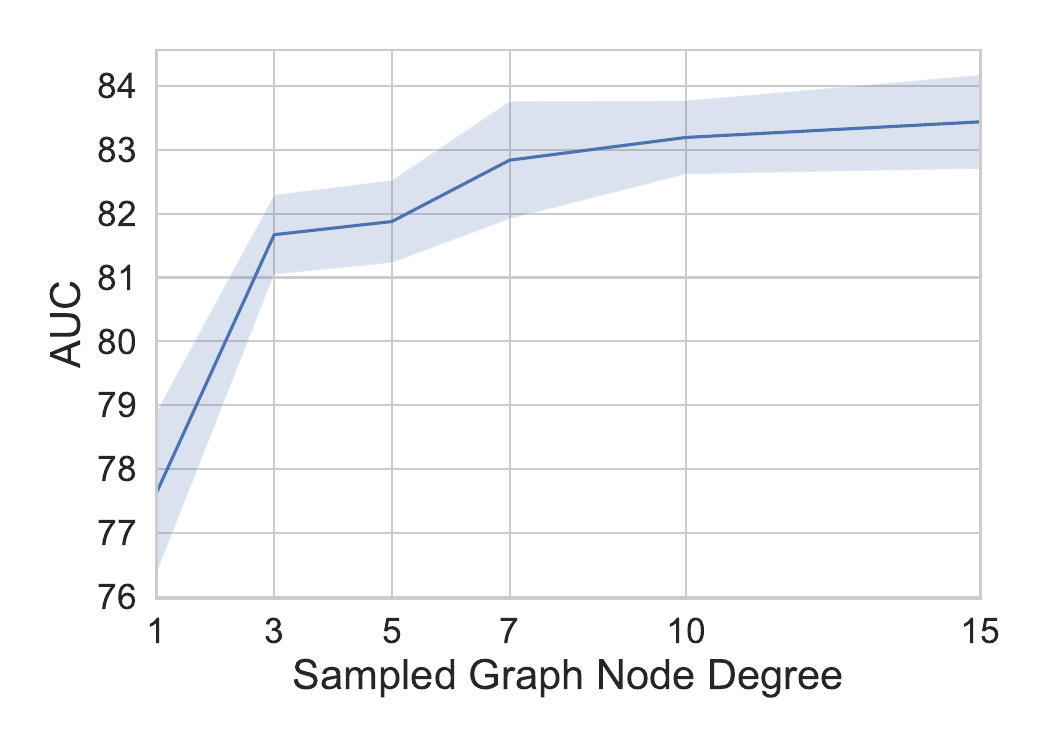}} 
\subfloat[][Tolokers memory]{\includegraphics[width = 1.55in]{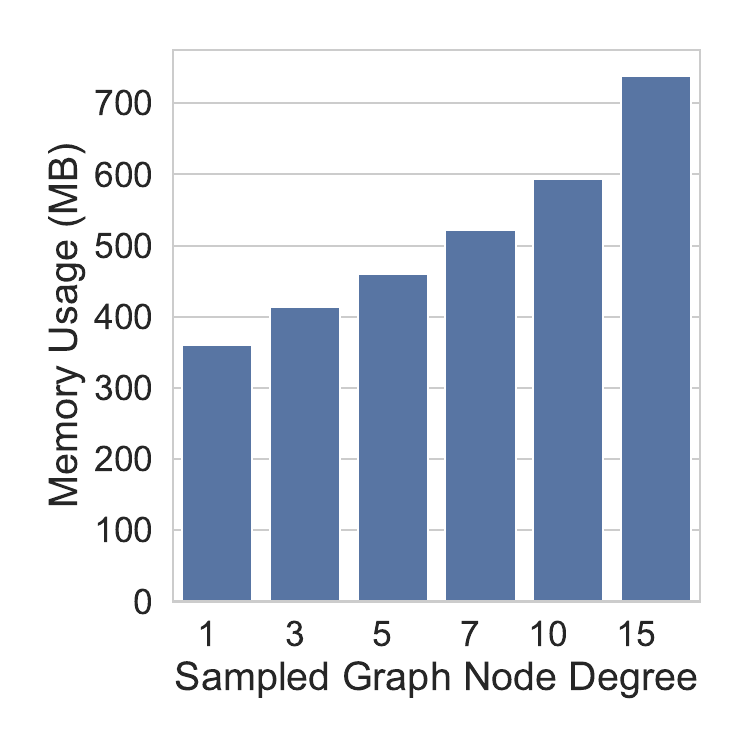}}
\subfloat[][Tolokers runtime]{\includegraphics[width = 1.55in]{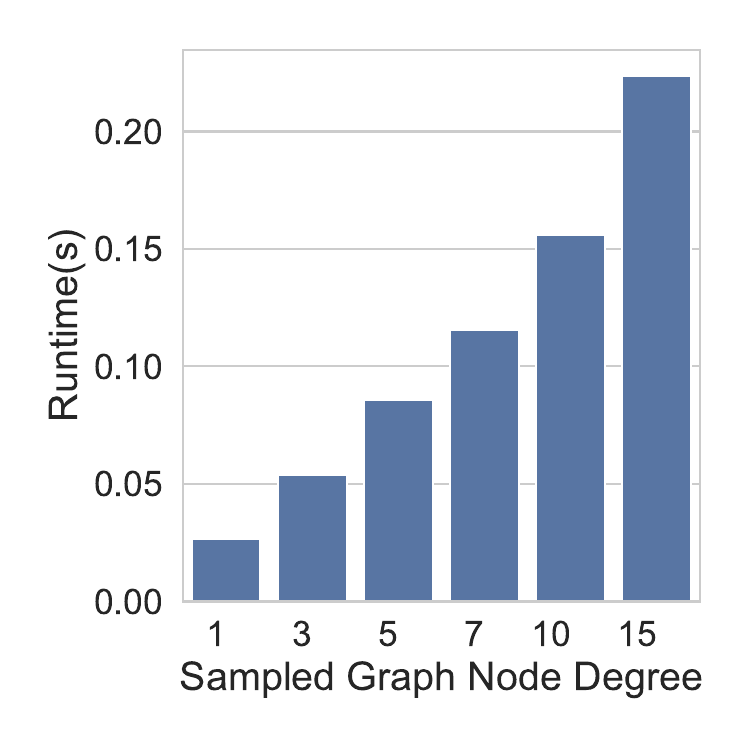}}
\\
\subfloat[][Minesweeper AUC]{\includegraphics[width = 2.18in]{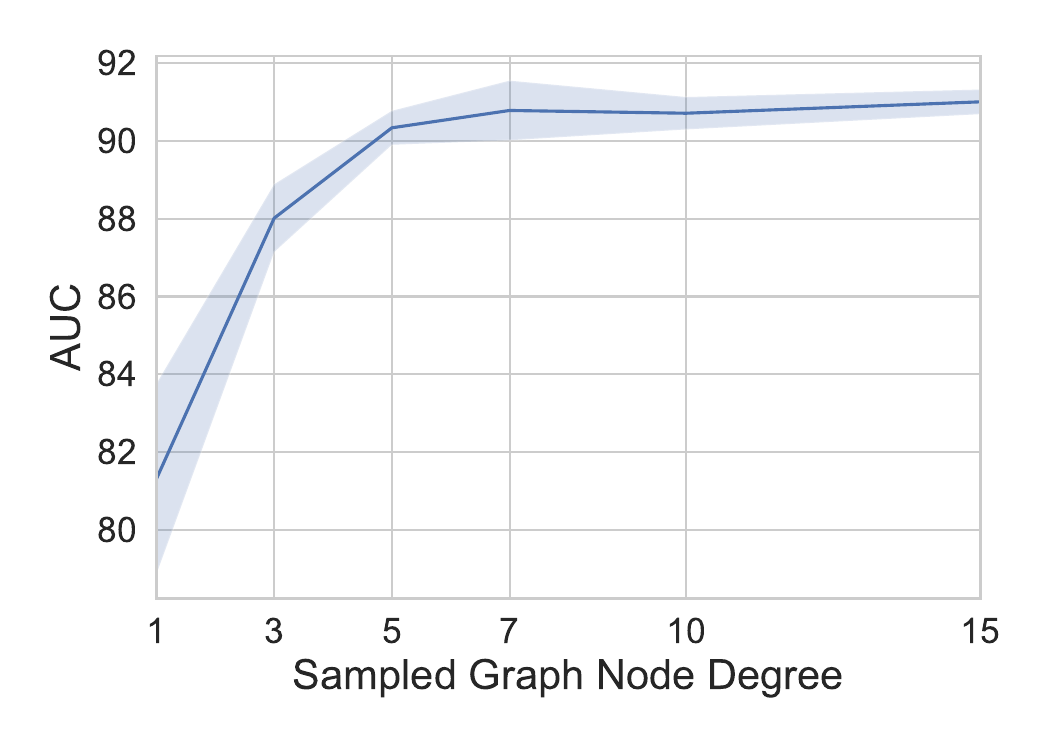}} 
\subfloat[][Minesweeper memory]{\includegraphics[width = 1.55in]{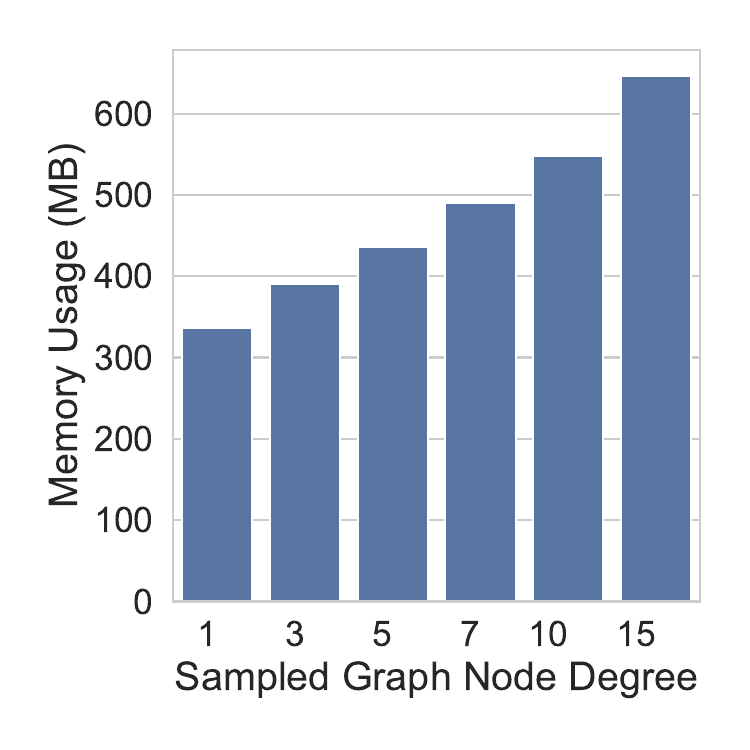}}
\subfloat[][Minesweeper runtime]{\includegraphics[width = 1.55in]{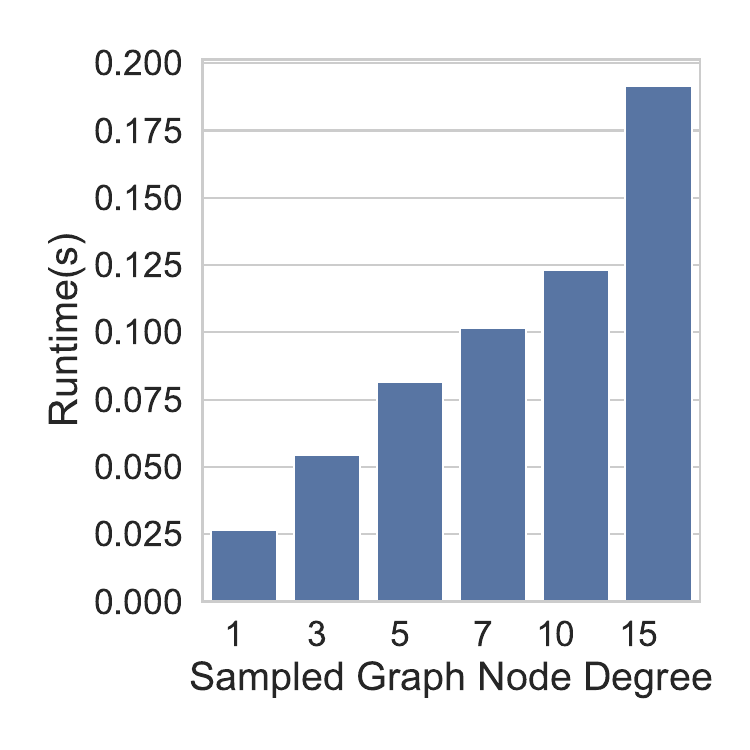}}
\\
\subfloat[][Photo accuracy]{\includegraphics[width = 2.18in]{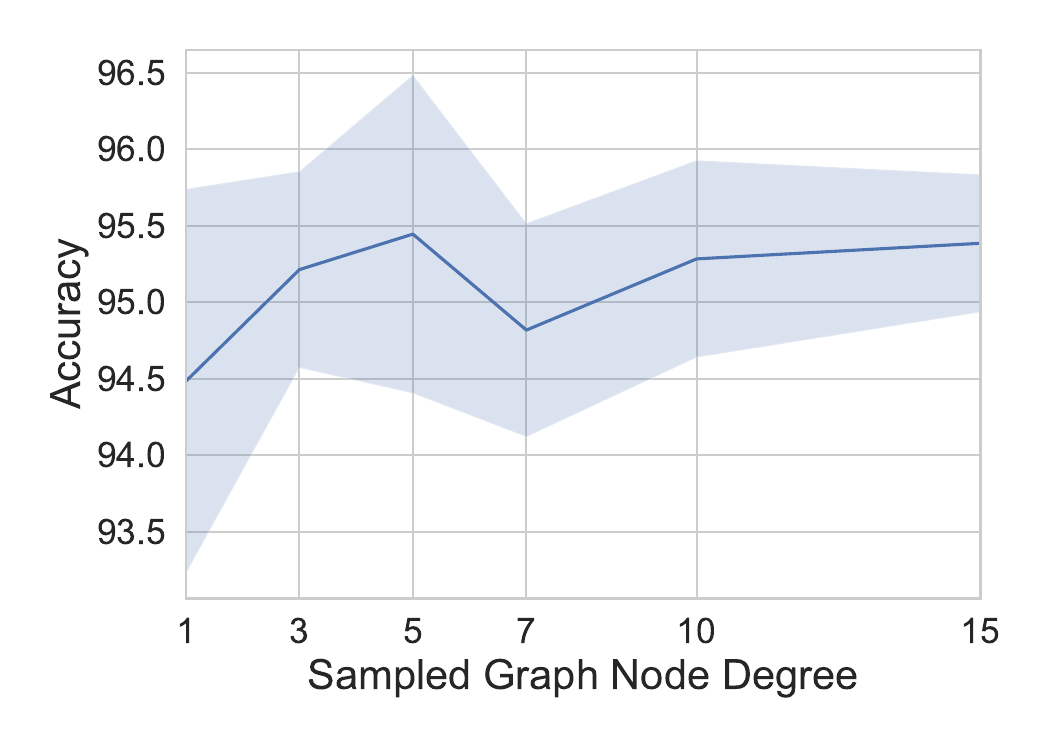}} 
\subfloat[][Photo memory]{\includegraphics[width = 1.55in]{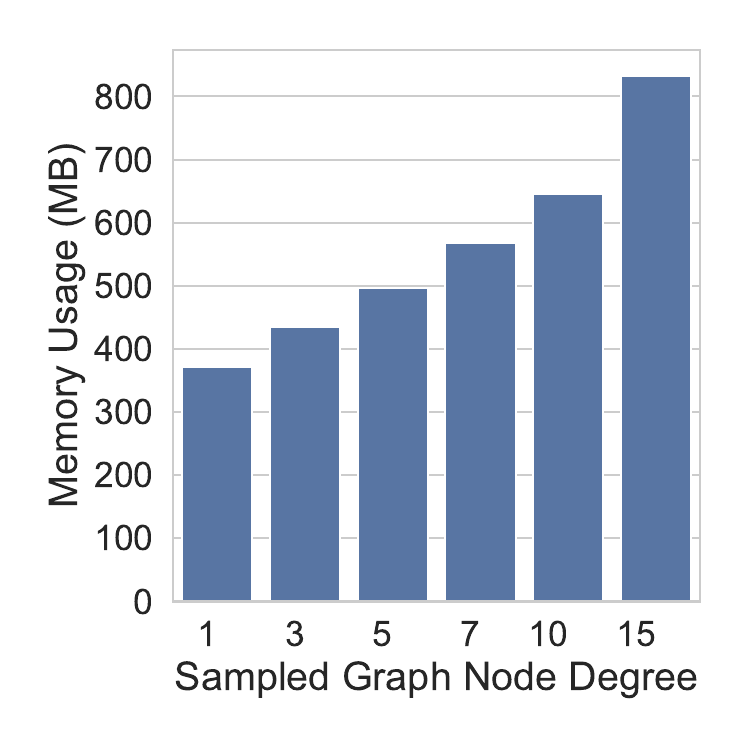}}
\subfloat[][Photo runtime]{\includegraphics[width = 1.55in]{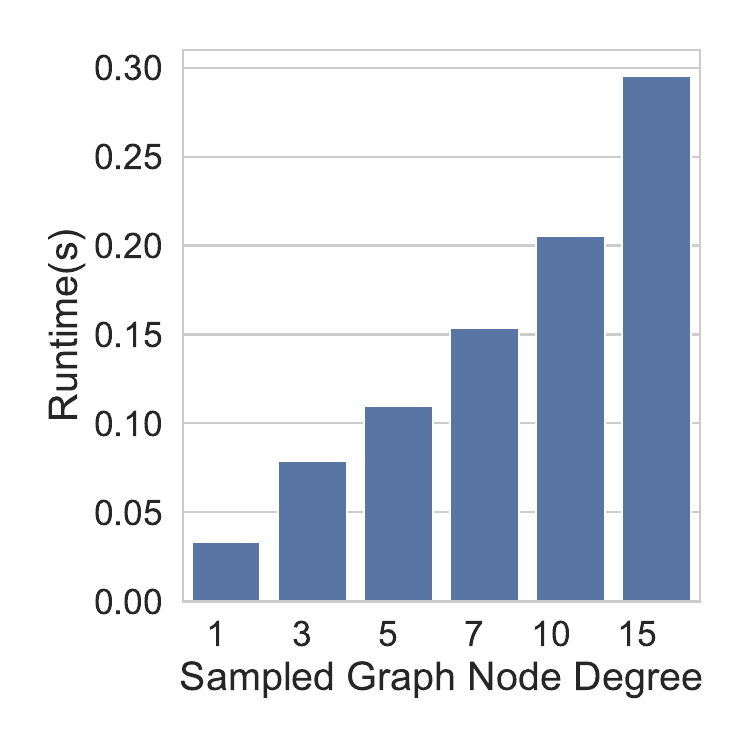}}
\\
\subfloat[][CS accuracy]{\includegraphics[width = 2.18in]{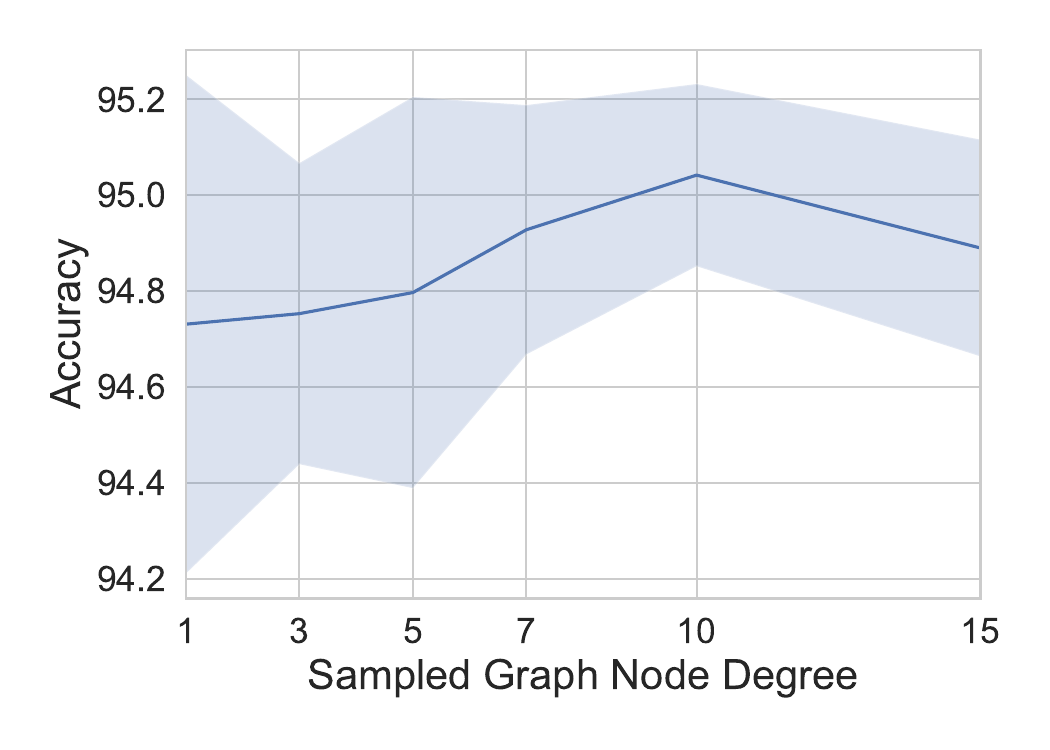}} 
\subfloat[][CS memory]{\includegraphics[width = 1.55in]{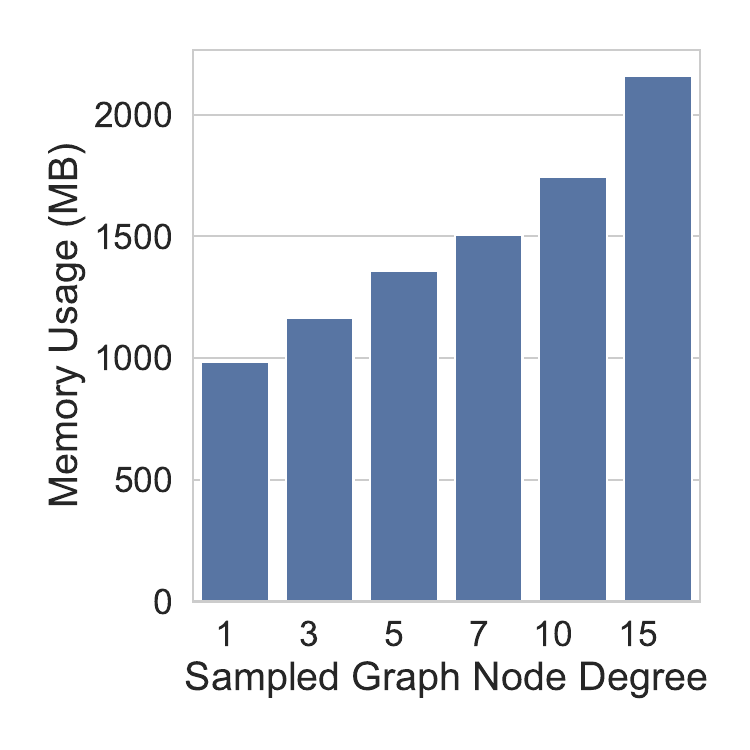}}
\subfloat[][CS runtime]{\includegraphics[width = 1.55in]{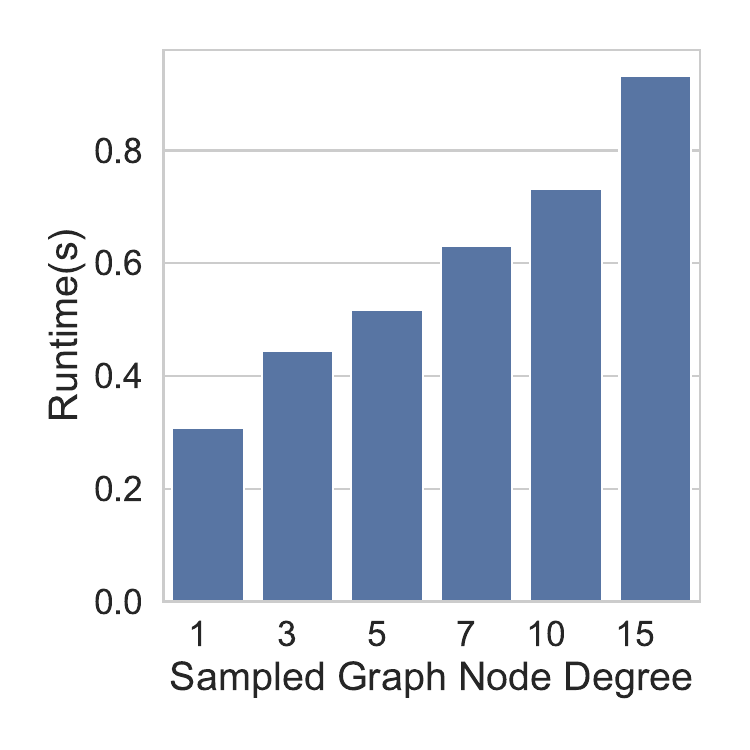}}
\caption{AUC and accuracy results, along with memory and runtime analysis, are presented for four datasets: two homophilic datasets (Amazon-Photo and Coauthor-CS) and two heterophilic datasets (Tolokers and Minesweeper). Larger sampling degrees generally lead to better results; however, for the homophilic datasets, even a very small neighborhood size can yield substantial performance. Increasing the sampling degree increases memory and runtime requirements accordingly.}
\label{fig:runtime_memory_by_deg}
\end{figure*}

\section{Experiment Details}

\subsection{Hyperparameters}

In our networks, we use a higher expander degree than what was used in the \textsc{Exphormer} paper. Since many of these edges will get a small attention score, a higher attention score increases the receptive field of the nodes, letting the final network be able to sample from wider options and have better access to long-range dependencies. We also noticed, the attention scores in the first layer are usually more flat than the other layers and so we usually sample more edges from the first attention layer. For the comparisons both on the results and the memory we have given the same expander degree to the Exphormer and the ogbn-arxiv dataset could barely fit into a 40GB GPU memory device with higher expander degree. For the attention score estimator network, we do not use dropout, and we only use one attention head in these networks. The number of layers is always equal between both networks. 

We use AdamW optimization algorithm in all our networks and use a cosine learning rate scheduler with it. We use weight decay of $1e-3$ in all networks. We use layer norm in attention score estimator networks to keep attention scores more meaningful, but use a batch norm for better results in the final \textsc{Spexphormer} model.
Other key hyperparameters can be found in \cref{tab:hyperparamshomophily,tab:hyperparamsheter,tab:hyperparamslarge}.

\begin{table}[ht]
\centering
\caption{Hyperparameters used for training the networks for homophilous datasets.}
\label{tab:hyperparamshomophily}
\scalebox{0.9}{
\begin{tabular}{l|cccccc} 
\toprule
{\bf Hyperparameter}    & {\bf OGBN-Arxiv} & {\bf Computer}   & {\bf Photo} & {\bf CS} & {\bf Physics} & {\bf WikiCS} \\ \hline
\multicolumn{7}{c}{Attention Score Estimator}  \\
\hline
$L$       & 3       & 4       & 4           & 4       & 4   & 4  \\
$d_s$& 8 & 4 & 4 & 4 & 4 & 4\\
Num Epochs  & 200     & 200     & 200& 200& 200 & 100\\
Learning Rate      & 0.01    & 0.1   & 0.001       & 0.002& 0.001 & 0.1 \\
\hline
\multicolumn{7}{c}{Final Spexphormer Network} \\
\hline
$d_l$& 96      & 80      & 56& 64& 64 & 64\\
$\deg_\ell$& [6, 6, 6] & [5, 5, 5, 5]& [5, 5, 5, 5]& [5, 5, 5, 5]& [5, 5, 5, 5] & [8, 5, 5, 5]\\
Number of Heads   & 2       & 2       & 2           & 2       & 2   & 2     \\
Learning Rate           & 0.01    & 0.001   & 0.01& 0.002& 0.001 & 0.001 \\
Num Epochs              & 600     & 150     & 100& 120& 80 & 100 \\
Dropout                 & 0.3     & 0.5& 0.5& 0.4     & 0.4  & 0.5\\ 
\bottomrule
\end{tabular}
}
\end{table}

\begin{table}[ht]
\centering
\caption{Hyperparameters used for training the networks for heterophilic datasets.}
\label{tab:hyperparamsheter}
\scalebox{0.8}{
\begin{tabular}{l|cccccc} 
\toprule
{\bf Hyperparameter}       &{\bf Actor} &{\bf Minesweeper} &{\bf Tolokers} & {\bf Roman-Empire} & {\bf Amazon-ratings} & {\bf Questions} \\ \hline
\multicolumn{4}{c}{Attention Score Estimator} \\
\hline
$L$         & 3    & 4 & 4 & 4 & 4 & 4\\
$d_s$& 4 & 4 & 4 & 4& 4 & 4 \\
Num Epochs &100 & 100 & 200 & 200 & 100 & 100\\
Learning rate & 0.01 & 0.01 & 0.01& 0.01 & 0.01& 0.01\\
\hline
\multicolumn{7}{c}{Final Spexphormer Network} \\
\hline
$d_l$ &32 & 32 & 32& 40 & 64 & 32\\
$\deg_\ell$& [2, 2, 2] & [12,5,5,5] & [12, 10, 10, 10] & [12, 10, 10, 10] & [8, 5, 5, 5] & [5, 5, 5, 5]  \\
Number of Heads &4 & 4 & 4 & 2 & 2 & 2\\
Learning Rate &0.01 & 0.01 & 0.01 &0.03 &0.01 & 0.01\\
Num Epochs &100 & 80 & 200 & 200 & 200 & 80\\
Dropout &0.5 & 0.2 & 0.25 & 0.1 & 0.1 & 0.5\\ 
\bottomrule
\end{tabular}
}
\end{table}

\begin{table}[ht]
\centering
\caption{Hyperparameters used for training the networks for the large graphs datasets.}
\label{tab:hyperparamslarge}
\begin{tabular}{l|lll} 
\toprule
{\bf Hyperparameter}       &{\bf ogbn-proteins} &{\bf Amazon2M} &{\bf Pokec}\\ \hline
\multicolumn{4}{c}{Attention Score Estimator} \\
\hline
$L$         & 2    & 2 & 2 \\
$d_s$& 8 & 8 & 8\\
expander degree & 200 & 30 & 30\\
Num Epochs &150 & 150 & 150\\
Learning rate & 0.01 & 0.01 & 0.01\\
\hline
\multicolumn{4}{c}{Final Spexphormer Network} \\
\hline
$d_l$ & 64& 128 & 64 \\
$\deg_\ell$& [50, 30] & [10,10]& [20, 20] \\
Number of Heads & 1& 1 & 1\\
Learning Rate &0.005 & 0.001 & 0.01\\
Num Epochs &200 & 200 & 300 \\
Dropout &0.1& 0.2 & 0.2\\ 
Batch size & 256& 1000& 500 \\
GPU Memory& 2232MB& 3262MB& 2128MB\\
\bottomrule
\end{tabular}

\end{table}

\subsection{Hardware}
For all trainings of the medium-sized graph datasets and the final network training of the large-sized graphs, we used GPUs of type A100 with 40GB memory, and V100, both 32GB and 16GB versions. While these are powerful GPUs, we have always monitored the GPU memory usage for computational efficiency, ensuring that no more than 8GB is used for whole graph training and no more than 4GB of GPU memory is used with batching. Training with even less memory is feasible with smaller batch sizes.

For calculating the attention scores on the large graph datasets, we have used CPU devices Intel Xeon E5-2680 v4, with 500GB of memory. Except for the Amazon2M dataset, for the other datasets 200GB of memory would be sufficient.

\section{Theory}

In this section, we theoretically analyze the compressibility of the Graph Transformer architecture and also sparsification guarantees using the attention score estimator network.

For simplification, we use the following formulation of a single head Transformer network:

\[
h_i\llh = \sum_{j=1}^{deg_i} a_{ij}^{(l)}\mathbf{V}_j^{(\ell)},
\]

\[
 h_i\lln = \W_2^{(\ell)} \left(\sigma \left(\W_1^{(\ell)} \left( h_i\llh\right)\right)\right),
\]

\[
a_{ij}^{(l)} = \frac{\exp\left({\mathbf{K}_j^{(\ell)} \cdot \mathbf{Q}_i^{(\ell)}}\right)}{\sum_{u \in \mathcal{N}_H(i)} \exp\left({\mathbf{K}_u^{(\ell)} \cdot \mathbf{Q}_i^{(\ell)}}\right)},
\]

where, $\V\llo = \W_V\llo h\llo$, $\Q\llo = \W_Q\llo h\llo$, $\K\llo = \W_K\llo h\llo$, and $\sigma$ can be any 1-Lipchitz activation function, such as $\ReLU$, which has been used in practice in our networks.  We remove the normalization parts from the architecture but assume that in all steps for all vectors, $\norm{X_i}_2, \norm{h_i\llh}_2, \norm{h_i\llo}_2 \leq \sqrt{\alpha}$, and all linear mapping $\W_\cdot$ matrices' operator norm is bounded by a constant $\beta$. The first assumption is realistic because of the layer-norm applied between the layers in real-world architectures. The second assumption is also justified as the operator norms are near $2$ in the initialization of the network by the default PyTorch initialization and during the optimization we expect the operator norm to not increase drastically from the initialization. \todo{Hamed: can we cite anything about this?} Also, we assume $h^{(0)} = X$, which is the input features. For a simpler notation, we will use $D$ for a hypothetical large network hidden dimension in this analysis, and $d$ is the hidden dimension of the narrow network. For simplicity, in our analysis, we assume $X \in \R^{n \times D}$. In case each node has less than $D$ features, we can concatenate them with zeros.

\subsection{On the Compressibility of the Graph Transformer}

Our approach uses a narrow network to estimate the attention scores. We want to show if we have a large network with good accuracy for a task on a graph, we can have a less complex network that can work on the same input graph and the error of this network is bounded by $\mathcal{O}(\eps)$ from the large network.

The most memory/time-intensive part of a Transformer architecture is its attention score calculation part. The rest of the sections are node/token-wise and linear with respect to the number of nodes. The attention score estimation part of a full-Transformer layer requires $\mathcal{O}(n^2d)$ operations and $\mathcal{O}(md)$ operators are required for a sparse Transformer with $m$ attention edges. In the main Exphormer network, this would also be more intensive as the edge features mappings require $\mathcal{O}(md^2)$ operations, but since we replace edge feature mappings with edge embeddings by their type, this part in case we do not have other edge features is $\mathcal{O}(md)$, but $m$ still can be $\omega(n)$, and it will be the most computationally-intensive part. 

Assume we have a large network with $L$ layers, where $L$ is $\mathcal{O}(1)$, and hidden dimension $D$, we will show that there is a similar network with $L$ layers where the attention score calculation matrices $\W_Q, \W_K \in \mathbb{R}^{D \times d}$, and all other matrices are of the same size and $d$ is $\mathcal{O}(C^L\frac{\log n}{\epsilon^2})$, where $C$ is a constant based on $\alpha$ and $\beta$. For this proof we use the distributional Johnson-Lindenstrauss transform lemma \citep{johnson1984extensions}:

\begin{lemma}[Johnson-Lindenstrauss Transform Lemma (\textit{JLT})]
Assume $0 < \epsilon, \delta < \frac{1}{2}$ and any positive integer $D$, if $d = \mathcal{O}(\frac{\log(1/\delta)}{\epsilon^2})$, there exist a distribution over matrices $\mathbf{M} \in \mathbb{R}^{d \times D}$ that for any $x \in \mathbb{R}^{D}$ and $\norm{x} = 1$:

$$
\operatorname{Pr}(\norm{\M x} - 1 > \epsilon) < \delta 
$$
\end{lemma}

The following corollary is an immediate conclusion from the \textit{JLT}. 

\begin{corollary}
    Assume $0 < \epsilon, \delta < \frac{1}{2}$ and any positive integer $D$, if $d = \mathcal{O}(\frac{\log(1/\delta)}{\eps^2})$, there exist a distribution over matrices $\M \in \R^{d \times D}$ that for any $x, y \in \R^{D}$:

$$
\operatorname{Pr}((1-\eps) \norm{x - y} < \norm{\M x- \M y}  < (1 + \eps) \norm{x - y}) < \delta 
$$
\end{corollary}
This can derived by replacing $x$ from \textit{JLT} with $\frac{x - y}{\norm{x-y}}$.

From this, we can derive another corollary about the dot product of the vectors in low-dimensional space. 

\begin{corollary} [JLT-dot product]
\label{cor:jltdot}
    Assume $0 < \epsilon, \delta < \frac{1}{2}$ and any positive integer $D$, if $d = \mathcal{O}(\frac{\log(1/\delta)}{\eps^2})$, there exist a distribution over matrices $\M \in \R^{d \times D}$ that for any $x, y \in \R^{D}$, and $\norm{x}, \norm{y} \leq \sqrt{\alpha}$:

$$
\operatorname{Pr}((1-\eps\alpha) x\tp y <  x\tp\M\tp\M y  < (1 + \eps\alpha) x\tp y) < \delta 
$$
\end{corollary}

For the proof see \citep[Corollary 2.1]{Lec09}. As a result of this corollary, if we have $m$ pairs of vectors $(x_i, y_i)$, and for each $i$ $\norm{x_i}_2, \norm{y_i}_2 \leq \sqrt{\alpha}$ of $\sqrt{\alpha}$, and $d = \mathcal{O}(\frac{\log(m)}{\eps^2})$, there exists an $\M$ such that for all these pairs $\abs{x_i\tp\M\tp\M y_i -  x_i\tp y_i}  <  \eps\alpha$. The proof can be done using a union bound over the error from Corollary~\ref{cor:jltdot}. Also, in our case where $m$ is the number of edges, we know that $m \leq n^2$, thus we can also say $d = \mathcal{O}(\frac{\log(n)}{\eps^2})$.

\begin{theorem}
\label{thrm:narrow_attention}
Assume we have a Transformer network $\mathcal{T}$ with arbitrary large hidden dimension $D$, $L=O(1)$ layers, and in this network, in all layers, we have $\norm{h_\cdot}_2 \leq \sqrt{\alpha}$, and $\norm{\W_\cdot}_{op} \leq \beta$. There exists a Transformer $\widehat{\mathcal{T}}$, that for any layer $\W_Q$ and $\W_K$ are in $\R^{d \times D}$ for a $d=\mathcal{O}(\frac{\log n}{\eps^2})$, with a sufficiently small $\eps$, and for all $i \in [n]$, $\norm{\mathcal{T}(X)_i - \widehat{\mathcal{T}}(X)_i}_2 = \mathcal{O(\eps)}$. And furthermore, for any attention score $\frac{a_{ij}\llo}{\ah_{ij}\llo} = 1 + \mathcal{O}(\eps).$
\end{theorem}

\begin{proof}
In the proof we use hat notation, $\widehat{\square}$, for the vectors and matrices from $\widehat{\mathcal{T}}$, for example, $\hat{h}^{(\ell)}$ are the outputs of layer $\ell$, and $\Wh_\cdot$ are the weight matrices for this network. In all layers for both networks $\W_V, \W_1$, and $\W_2$, are of the same size, so we set $\Wh_V = \W_V$, $\Wh_1 = \W_1$, and $\Wh_2 = \W_2$.

For the proof, we want to find $\eps^{(0)}, \cdots, \eps^{(L)}$ in a way that for any $v$ in layer $\ell$, $\abs{h_v^{(\ell)} - \hat{h}_v^{(\ell)}} < \eps^{(\ell)}$. We will find these bounds inductively, starting from the first layer. We have $\eps^{(0)} = 0$, as both networks have the same input, and we want to bound $\eps^{(\ell+1)}$ based on $\eps^{(\ell)}$. 

We have $\Q\llo = \W_Q\llo \Ho\llo$, $\K\llo = \W_K\llo \Ho\llo$ and assume 
$\Bar{\Q}\llo = \W_Q\llo \Hh\llo$, $\Bar{\K}\llo = \W_K\llo \Hh\llo$. 
Because of the operator norm of matrices $\W_Q$ and $\W_K$, for each $i$ we have $\norm{q_i\llo - \Bar{q}_i\llo} \leq \eps\llo \beta$ and $\norm{k_i\llo - \Bar{k}_i\llo} \leq \eps\llo \beta$. Also, we have $\norm{q\llo_i}, \norm{k\llo_i} \leq \beta\sqrt{\alpha}$, thus $\norm{\Bar{q}_i\llo}, \norm{\Bar{k}_i\llo} \leq \beta(\eps\llo + \sqrt{\alpha})$. Now, for each pair of $i$ and $j$, we have:

\begin{align*}
    \abs{q_i\llo\cdot k_j\llo - \Bar{q}_i\llo \cdot \Bar{k}_j\llo} &= \abs{q_i\llo\cdot k_j\llo - \Bar{q}_i\llo \cdot k_j\llo + \Bar{q}_i\llo \cdot k_j\llo - \Bar{q}_i\llo \cdot \Bar{k}_j\llo} \\
    & \leq \abs{q_i\llo\cdot k_j\llo - \Bar{q}_i\llo \cdot k_j\llo} + \abs{\Bar{q}_i\llo \cdot k_j\llo - \Bar{q}_i\llo \cdot \Bar{k}_j\llo} \\
    & = \abs{(q_i\llo - \Bar{q}_i\llo) \cdot k_j\llo} + \abs{\Bar{q}_i\llo \cdot (k_j\llo - \Bar{k}_j\llo)} \\
   & \leq \norm{q_i\llo - \Bar{q}_i\llo} \norm{k_j\llo} + \norm{\Bar{q}_i\llo} \norm{k_j\llo - \Bar{k}_j\llo} \\
    & \leq \sqrt{\alpha}\beta\eps\llo + (\sqrt{\alpha}+\beta\eps\llo)\beta\eps\llo \\
    & = 2\sqrt{\alpha}\beta\eps\llo + (\beta\eps\llo)^2
\end{align*}

On the other hand, according to the ~\ref{cor:jltdot}, for a $0 < \eps < 1/2$ and $d=\mathcal{O}(\frac{\log(n)}{\eps^2})$ there exists a matrix $\M_{QK} \in \R^{d \times D}$, such that if we define $\Qh\llo = \M_{QK}\Bar{\Q}\llo$ and $\Kh\llo = \M_{QK}\Bar{\K}\llo$, $\abs{\Bar{q}_i\llo \cdot \Bar{k}_j\llo - \qh_i\llo \cdot \kh_j\llo} < \beta^2(\alpha+(\eps\llo)^2+2\sqrt{\alpha}\eps\llo)\eps$ for all $(i, j)$ pairs in the attention pattern. Note that we can define $\Wh_Q\llo = \M_{QK}\llo\W_Q\llo$, and $\Wh_K\llo = \M_{QK}\llo\W_K\llo$, both in $\R^{d \times D}$, as weights for the narrow attention score estimator network. With a triangle inequality we have $$\abs{q_i\llo \cdot k_i\llo - \qh_i\llo \cdot \kh_i\llo} < \beta^2(\alpha+(\eps\llo)^2+2\sqrt{\alpha}\eps\llo)\eps + 2\sqrt{\alpha}\beta\eps\llo + (\beta\eps\llo)^2.$$

By setting $\eps\llo \leq 1$, we have $$\abs{q_i\llo \cdot k_i\llo - \qh_i\llo \cdot \kh_i\llo} < \beta^2(\alpha+1+2\sqrt{\alpha})\eps + \beta(2\sqrt{\alpha}+\beta)\eps\llo.$$
Let us define $\eps_a = \beta^2(\alpha+1+2\sqrt{\alpha})\eps + \beta(2\sqrt{\alpha}+\beta)\eps\llo$, we have:

\begin{gather*}
\ah_{ij}^{(\ell)} = \frac{\exp(\qh_i\llo\cdot \kh_j\llo)}{\sum_{u \in \mathcal{N}_H(i)} \exp(\qh_i\llo \cdot\kh_u\llo)} \leq \frac{\exp(q_i\llo\cdot k_j\llo + \eps_a)}{\sum_{u \in \mathcal{N}_H(i)} \exp(q_i\llo\cdot k_j\llo -\eps_a)} \leq a_{ij}^{(\ell)}\exp(2\eps_a)
\\
\ah_{ij}^{(\ell)} = \frac{\exp(\qh_i\llo\cdot \kh_j\llo)}{\sum_{u \in \mathcal{N}_H(i)} \exp(\qh_i\llo\cdot \kh_u\llo)} \geq \frac{\exp (q_i^{(\ell)}\cdot k_j^{(\ell)} -\eps_a)}{\sum_{u \in \mathcal{N}_H(i)} \exp(q_i^{(\ell)} \cdot k_u^{(\ell)} +\eps_a)} \geq a_{ij}^{(\ell)}\exp(-2\eps_a)
\end{gather*}

Now we bound $\norm{h_i\llh - \hh_i\llh}$: 

\begin{align*}
   \norm{h_i\llh - \hh_i\llh} &= \norm{\sum_{j\in Nei(i)} a_{ij}\llo v_j\llo - \ah_{ij}\vh_j\llh} \\
    & = \norm{\sum_{j\in Nei(i)} a_{ij}\llo v_j\llo - \ah_{ij}\llo v_j\llo + \ah_{ij}\llo v_j\llo - \ah_{ij}\vh_j\llo} \\
    &= \norm{\sum_{j\in Nei(i)} (a_{ij}\llo - \ah_{ij}\llo) v_j\llo + \ah_{ij}\llo(v_j\llo - \vh_j\llo)} \\
    &= \norm{(v_j\llo - \vh_j\llo) + v_j\llo\sum_{j\in Nei(i)} (a_{ij}\llo - \ah_{ij}\llo)} \\
    & \leq \norm{v_j\llo - \vh_j\llo} + \norm{v_j\llo}\sum \abs{a_{ij}\llo - \ah_{ij}\llo} \\
    & \leq \eps\llo\beta + \sqrt{\alpha} \sum \max (1-\exp(-2\eps_a), \exp(2\eps_a) -1) a_{ij}\llo \\
    & \leq \eps\llo\beta + \sqrt{\alpha} (\exp(2\eps_a) -1),
\end{align*} 

and since $1+x < \exp(x) < 1+2x$ for $0<x<1$, if we have $\eps_a < 1$, we have 
\begin{equation}
    \norm{h_i\llh - \hh_i\llh} \leq \beta\eps\llo + 4 \sqrt{\alpha} \eps_a
\end{equation}

For the feed-forward network part, we know that this network is $\beta^2$-Lipschitz because $\W_1\llo$ and $\W_2\llo$ have maximum operator norm $\beta$ and $\sigma$ is a 1-Lipschitz activation function. Thus we have
\begin{equation*}
    \norm{h_i\lln - \hh_i\lln} \leq \beta^2(\beta\eps\llo + 4 \sqrt{\alpha} \eps_a) =
    (\beta^3 + 8\beta\alpha+4\beta^2\sqrt{\alpha})\eps\llo + 4\beta^2(\alpha\sqrt{\alpha} + 2\alpha + \sqrt{\alpha}) \eps.
\end{equation*}

Both $\beta^3 + 8\beta\alpha+4\beta^2\sqrt{\alpha}$ and $4\beta^2(\alpha\sqrt{\alpha} + 2\alpha + \sqrt{\alpha})$ are constants, and if we define them as $c_1$ and $c_2$, we have 

\begin{equation*}
    \eps\lln \leq c_1\eps\llo + c_2\eps
\end{equation*}

Given $\eps^{(0)} = 0$, as both networks get the same input, we have
\begin{align*}
    \eps^{(L)} &\leq c_1\eps^{(L-1)} + c_2\eps \\
    & \leq c_1(c_1\eps^{(L-2)} + c_2\eps) + c_2\eps \\
    & \cdots \\
    & \leq c_2\eps (c_1^{L-1} + \cdots + c_1) \\
    & = \frac{c_1(c_2^L-1)}{c_2-1}\eps
\end{align*}

While the error increases exponentially with the number of layers, when we have $L = O(1)$, then the error is bounded by a constant factor of chosen $\eps$. Now, we know that $\norm{\mathcal{T}(X)_i - \widehat{\mathcal{T}}(X)_i}_2 \leq \eps^{(L)} = \mathcal{O(\eps)}$. 

\end{proof}

While from the theorem it seems that the error is increasing exponentially by the layers, in practice the maximum number of layers used in this work is four with most large graph experiments using just two layers. Thus the constant factor will not be as large as it might look. Also, in real-world graphs usually, the columns of $X$ are not quite $n$ distinct vectors and many vectors would be equal or very similar to each other if we have $\kappa$ unique vectors in the first layer the complexity for the $d$ can be reduced to $\mathcal{O}(\frac{\log \kappa}{\eps^2})$. In the homophily graphs the representations $h\llo$ tend to converge to each other and thus again the number of unique vectors will be reduced letting us have smaller $d$, but these assumptions are not considered in the proof as we keep it general.

Although we have proved the existence of the $\widehat{\mathcal{T}}$, this does not mean that training with a gradient-based algorithm will necessarily lead to the introduced weights, but this gives at least the guarantee that such a network exists. However, on the other hand, it is also possible that the training process finds a set of weights that work better than the weights constructed in this proof.

Theorem~\ref{thrm:narrow_attention}, by narrowing the attention score calculation part, reduced the complexity from $\mathcal{O}(mD + nD^2)$ to $\mathcal{O}(md + nD^2)$, and for dense graphs or in scenarios we add denser expander graphs, where $m \gg n$, already the introduced network has a much lower complexity. However, our narrow network uses narrow hidden dimensions in all steps and has complexity $\mathcal{O}(md + nd^2)$. Proving the same guarantee along the whole network is not easy, if not impossible, without any further assumptions on $X$ and the large network.
\citet{shirzad2024compression} explores these settings further, in the presence of various additional assumptions.

\subsection{Analysis of the Sampling Process}
After training a network with a smaller width $d$, we sample the edges from the original graph and use them in the second-phase training with a large hidden width $D$. In this section, we shall analyze our sampling process. Formally, we model our process as follows. Suppose that $A$ is the attention score matrix with hidden width $D$, then we sample and rescale $s$ entries of $A$ to form a sparse matrix $B$ where the goal is the matrix $B$ can approximate $A$ well, i.e., $\norm{A - B}_2 \le \eps \norm{A}_2$. However, recall that we can not access the entries of $A$ precisely. Instead, we consider another attention score matrix $A'$, which corresponds to hidden width $d$. 

The first question is how many samples we indeed need to form the matrix $B$ that approximates $A$ well? To answer this, we have the following lemma for the attention score matrix $A$.

\begin{theorem}%
\label{thrm:sampling}
    Suppose that an $n \times n$ matrix $A$ satisfies the following conditions: 
    \begin{enumerate}
        \item For each $i$, we have %
        $\norm{A_{(i)}}_1 = 1$.
        \item $\max_j \norm{A^{(j)}}_1 = K$%
        \item Each column $A^{(j)}$ is $\ell$-sparse.
    \end{enumerate}
    Then, consider the sampling procedure that samples $s \ge s_0 = \mathcal{O}(n K \log n /(\eps^2 \|A\|_2^2)) = \mathcal{O}(n\ell\log n / (\eps^2 K))$ entries of $A$ with replacement:%
    \begin{enumerate}
        \item For each sample $B_t$, the probability that $B_t$ samples entry $A_{ij}$ is $p_{ij} = \frac{1}{n}\cdot \frac{|A_{ij}|}{\|A_{(i)}\|_1} = \frac{1}{n} |A_{ij}|$ (with a rescale factor $1/p_{ij}$, i.e., $B_{t}[i, j] = A_{ij} / p_{ij}$), and each $B_t$ only samples one entry of $A$. 
        \item Form the matrix $B = (B_1 + B_2 + \dots + B_s)/s$.
    \end{enumerate}
    Then, we have that with probability at least $9/10$, %
    \[
    \norm{A - B}_2 \le \eps \norm{A}_2.
    \]

\end{theorem}

To prove this lemma, we need the following matrix Bernstein inequality.

\begin{lemma} [Matrix Bernstein inequality]
    Consider a finite sequence $X_i$ of i.i.d.\ random $m \times n$ matrices, with $\mathbb{E}[X_i] = 0$ and $\Pr(\norm{X_i}_2 \le R) = 1$. Let $\sigma^2 = \max\{\norm{\mathbb{E}[X_i X_i^T]}_2,  \norm{\mathbb{E}[X_i^T X_i]}_2\}$. For some fixed $s \ge 1$, let $X = (X_1 + X_2 + \cdots + X_s) / s$, then we have that 
    \[
    \Pr [\|X\|_2 \ge \eps] \le (m + n) \cdot \exp\left(\frac{s \varepsilon^2}{-\sigma^2 + R\eps/3}\right)
    .\]
\end{lemma}

\begin{proof}
    We follow a similar proof strategy to that of~\citet{achlioptas2013near}. At a high level, the work of~\citet{achlioptas2013near} considers the matrix Bernstein inequality, whose tail bound 
    is dependent on the following two quantities:
    \begin{gather*}
    \sigma^2 = \max\{\norm{\mathbb{E}[(A - B_1)(A - B_1)^T]}, \norm{\mathbb{E}[(A - B_1)^T(A - B_1)]}\}
     \text{ and}\\
    R = \max \norm{A - B_1} \ \ \text{over all possible realizations of $B_1$}.
    \end{gather*}
    Here $B_1$ is the matrix that only samples one entry, and the final output is $B = (B_1 + B_2 + \cdots + B_s)/s$. 
    Instead, we consider the following quantities, 
    \[
    \tilde{\sigma}^2 = \max\left\{\max_i \sum_j A_{ij}^2/p_{ij}, \max_j \sum_i A_{ij}^2/p_{ij}\right\}
    \]
    \[
    \tilde{R} = \max_{ij} |A_{ij}|/p_{ij}.
    \] 
    It is shown in Lemma~A.2 of~\cite{achlioptas2013near} that $|\sigma/\tilde{\sigma} - 1| \le \frac{\|A\|_2^2}{\sum_i \|A_{(i)}\|_1^2} $ and $|R/\tilde{R} - 1| \le \frac{\|A\|_2}{\|A\|_1} $. From our condition on the matrix $A$, both of the upper bounds are at most $1$. %
    Hence, we only need to consider $\tilde{\sigma}$ and $\tilde{R}$. Back to our case, we have that $p_{ij} = \frac{1}{n} \cdot \frac{|A_{ij}|}{\|A_{(1)}\|_1} = \frac{1}{n} \cdot |A_{ij}|$, from this and the assumption of $A$ we have 
    \[
    \tilde{\sigma}^2 = n\cdot \max\left\{\max_i \sum_j |A_{ij}|, \max_j \sum_i |A_{ij}|\right\} \le n \cdot K %
    \]
    \[
    \tilde{R} = \max_{ij} |A_{ij}|/p_{ij} = n.
    \] 
    Hence, to make $\delta \le 0.1
    $, we only need to set $\eps' = \eps \|A\|_2$ in the Matrix Bernstein inequality and then we have $s \ge O(n K \log n / (\eps^2 \|A\|_2^2))$. Finally, note that if $\norm{A^{(j)}}_1 = K$, then we have $\norm{A}_2 \ge \norm{A e_j}_2 = \norm{A^{(j)}}_2 \ge K /\sqrt{\ell}$, which means that $n K \log n / (\eps^2 \|A\|_2^2) \le n \ell \log n / (\eps^2 K)$. 
\end{proof}

However, as mentioned, we can not access the value of the entries of $A$ but the entries of $A'$ (which corresponds to the trained network with a small hidden width $d$). We next show that even in the case where we sample the entries of $A$ from $A'$, we can still get the same order of the bound if the entries of $A$ are not under-estimated seriously in $A'$.

\begin{proposition}
\label{prop:sampling_noisy}
    Suppose that the matrices $A$ and $A'$ satisfy the condition in Theorem~\ref{thrm:sampling} and for every $i,j $ we have 
    \[
    |A'_{ij}| \ge \frac{1}{\alpha} |A_{ij}|
    \]
    for some sufficiently large constant $\alpha$. Then consider the same sampling procedure in Theorem~\ref{thrm:sampling} but sampling the entries of $A$ from the value of $A'$. Then, the guarantee in Theorem~\ref{thrm:sampling} still holds.
\end{proposition}

\begin{proof}
    We only need to note that from the assumption, the actual sampling probability $p'_{ij} \ge \frac{1}{\alpha} \cdot p_{ij}$ in Theorem~\ref{thrm:sampling}, hence it will increase the $\tilde{\sigma}^2$ and $\tilde{R}$ by at most $\alpha$ times, which means that we can increase $s$ by an $\alpha$ factor to make the error probability at most $0.1$.
\end{proof}

\section{Attention Score Analysis}
\label{sec:attention_score_analysis_appendix}
In \cref{fig:energy_dists}, we observed that the attention scores are relatively close to the reference attention scores. In this section, we provide further details on these experiments and offer additional analysis of the attention scores.  

For our experiments, we used an implementation of the Exphormer model with normalization on $V$ mappings and temperature adjustment for the attention scores. For each random seed, we selected the model with the best result on the validation set. We used an expander degree of $10$ for the Actor dataset and $30$ for Amazon-Photos. The difference in expander degrees is due to the significant variation in the average degree of nodes across the datasets. We aimed to balance the number of expander edges and graph edges since it has an impact on some of the experiments. In addition to the expander edges, we also included self-loops, which are necessary for the universal approximation theorem outlined by \cite{shirzad2023exphormer}.  

All networks in these experiments were trained with four layers. For each hidden dimension, we ran $100$ experiments with different initializations. The learning rate was adjusted for each hidden dimension to ensure more stable convergence. However, for smaller hidden dimensions, some experiments led to drastically lower accuracy results, which we did not exclude from the analysis. All results, including those with lower accuracy, were considered in our analysis.

\subsection{Preliminaries}  
Before presenting further experimental results, we provide a brief introduction to the metrics used.

For two random variables $X \sim \mathcal{P}$ and $Y \sim \mathcal{Q}$, both defined in $\mathbb{R}^d$ (or equivalently, defined by their cumulative distribution functions (CDFs) $F$ and $G$), we can define the following metrics:

\paragraph{Energy Distance}  
Energy distance is a metric used to measure the distance between two distributions \citep{energy-distance,energy-rkhs,rizzo2016energy}, and is defined as:
$$D^2(F, G) = 2\mathbb{E}[X - Y] - \mathbb{E}[X - X'] - \mathbb{E}[Y - Y'],$$  
where $X, X' \sim \mathcal{P}$ and $Y, Y' \sim \mathcal{Q}$, with all variables being independent of each other. This value is shown to be twice the Harald Cramer's distance \citep{cramer1928composition}, which is defined as:
$$\int (F(x) - G(x))^2 \, dx.$$

This metric is non-negative; however, an unbiased estimator based on samples may yield negative results.

Although the energy distance is a useful metric for identifying the distance between two probability distributions, it may not fully capture the variations between them. This issue becomes particularly relevant when measuring the performance of generative models, as it helps assess whether the generative model correctly approximates the real distribution. The following pairs of metrics provide finer-grained understanding of two different types of approximation.

\paragraph{Precision \& Recall \citep{sajjadi2018assessing}}  These metrics assess generative models by constructing a manifold for both real and generated data. This is done by forming a hypersphere around each data point, extending its radius to the $k$-th nearest neighbor, and then aggregating these hyperspheres. Precision measures the proportion of generated samples that fall within the real data manifold, while recall quantifies the fraction of real samples covered by the generated data manifold. These metrics correlate well with human judgments in the visual domain and are effective in detecting issues like mode collapse and mode dropping.

\paragraph{Density \& Coverage \citep{naeem2020reliable}}  
These metrics, similar to Precision and Recall, evaluate generative models by considering individual hyperspheres rather than aggregating them into a manifold. Density measures the average number of real hyperspheres that each generated sample falls into, while Coverage quantifies the fraction of real samples that fall into at least one generated hypersphere. These metrics have been shown to be more robust than the Precision and Recall metrics in certain scenarios.

For all these metrics, we first consider the distribution of attention scores for each individual node's neighborhood in a single layer, trained with a specific hidden dimension, represented as a vector. We then compare these distributions across different hidden dimensions or among different layers. Finally, we average the results over all nodes.

\begin{figure*}[]
\captionsetup[subfloat]{farskip=-1pt,captionskip=-1pt}
\centering
\subfloat[][]{\includegraphics[width = 2.7in]{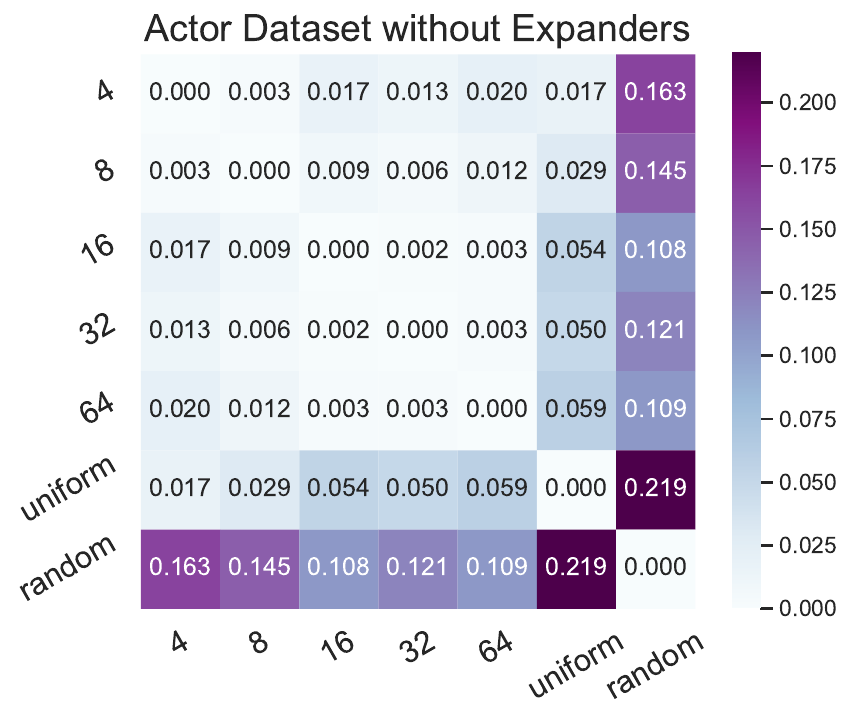}} 
\subfloat[][]{\includegraphics[width = 2.7in]{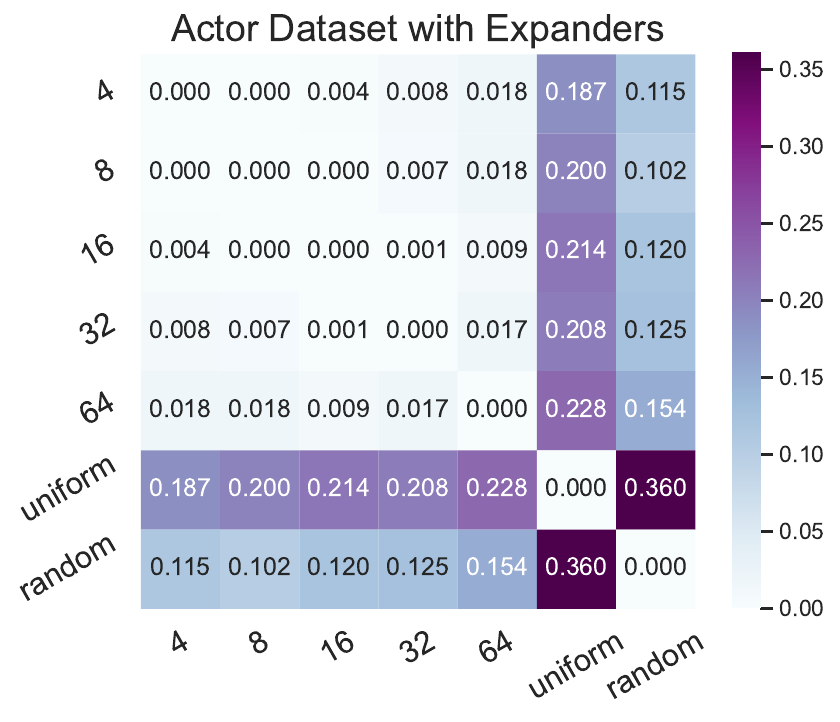}}\\
\subfloat[][]{\includegraphics[width = 2.7in]{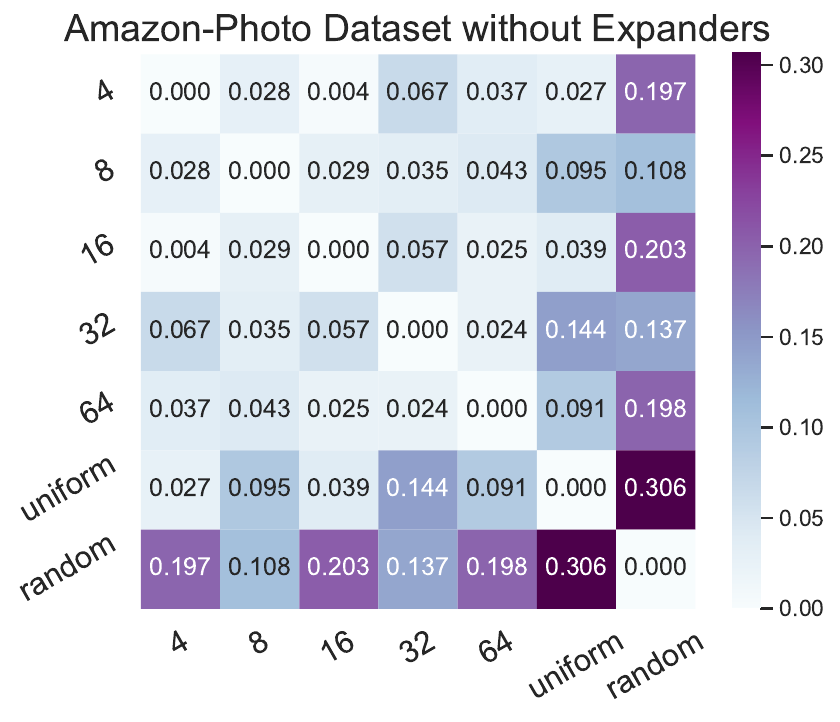}} 
\subfloat[][]{\includegraphics[width = 2.7in]{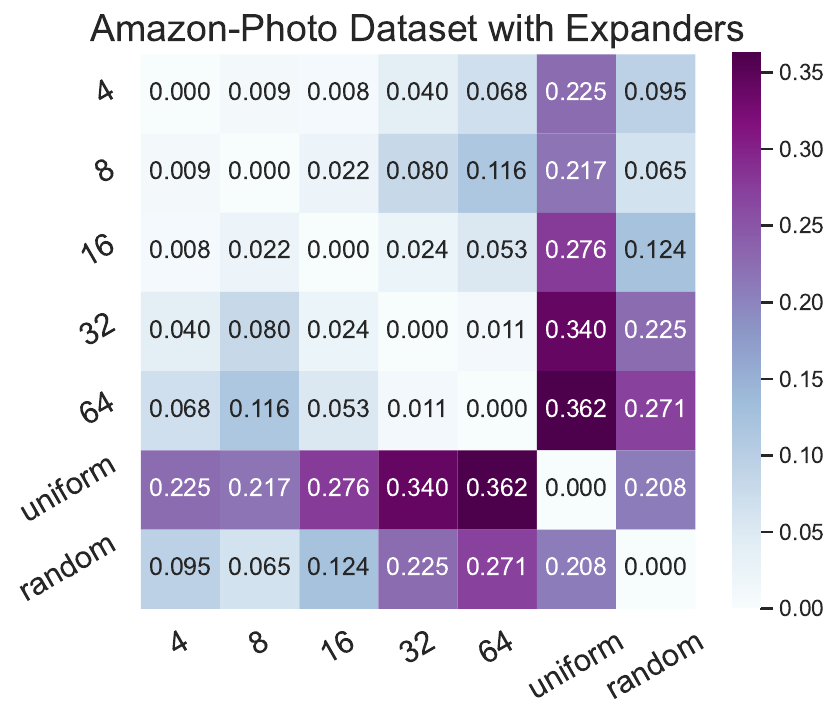}}

\caption{Pairwise energy distance across networks with different hidden dimensions, uniform distribution, and randomly generated attention scores.
}
\label{fig:pairwise_attention_score_edistance}
\end{figure*}

\begin{figure*}[]
\captionsetup[subfloat]{farskip=-1pt,captionskip=-1pt}
\centering
\subfloat[][]{\includegraphics[width = 2.7in]{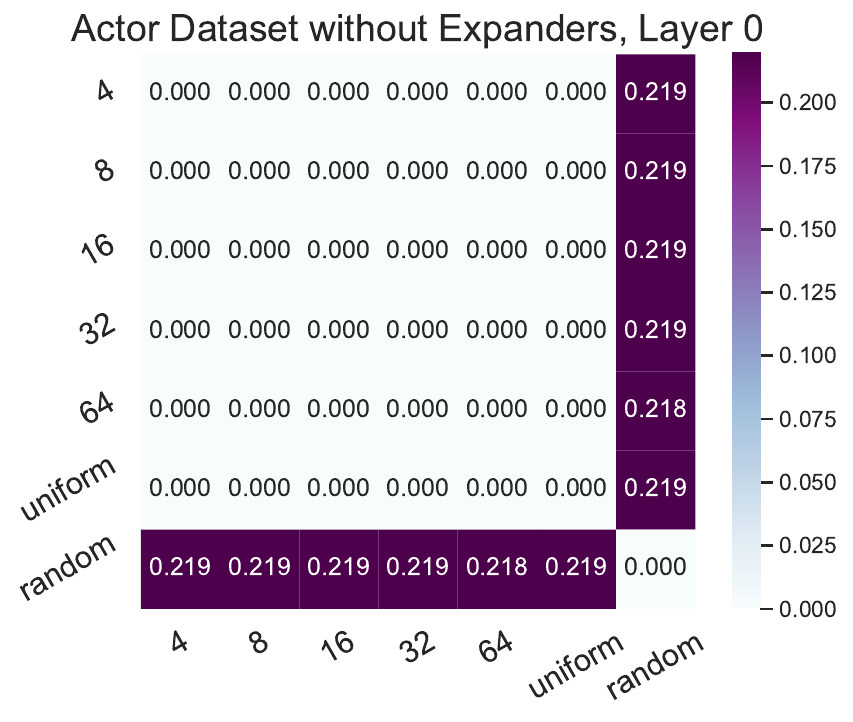}} 
\subfloat[][]{\includegraphics[width = 2.7in]{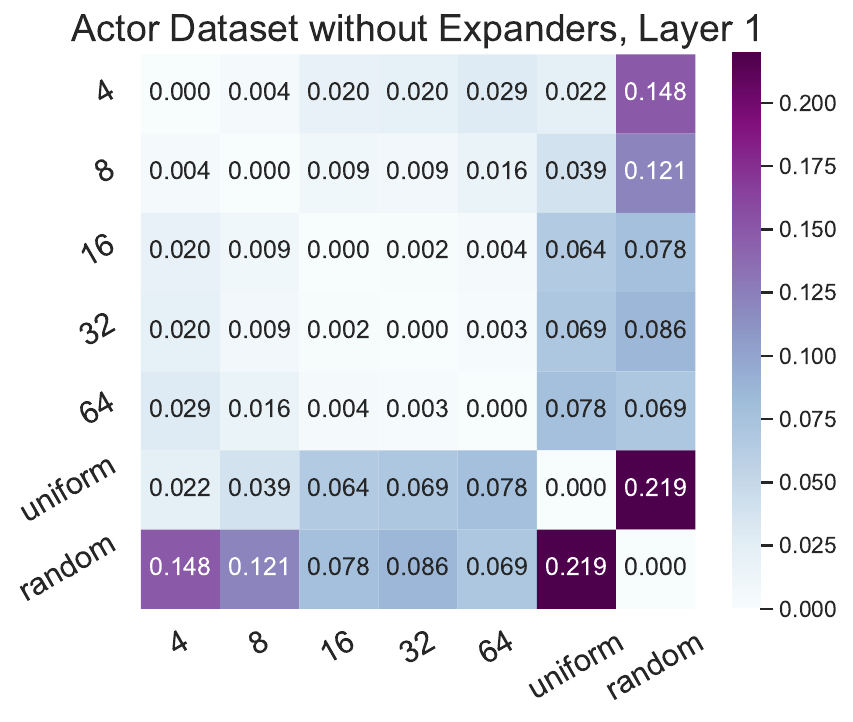}}\\
\subfloat[][]{\includegraphics[width = 2.7in]{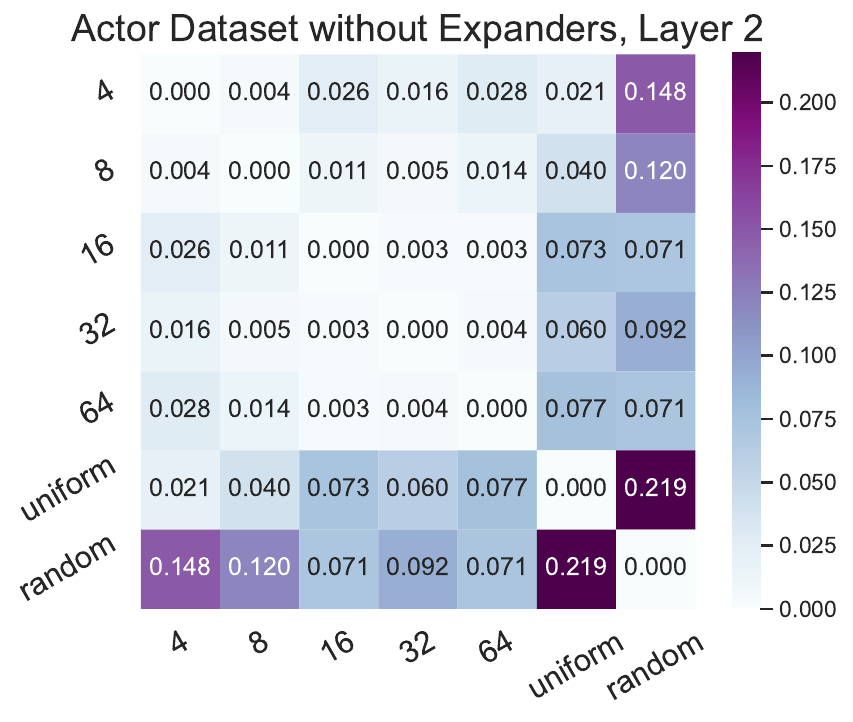}} 
\subfloat[][]{\includegraphics[width = 2.7in]{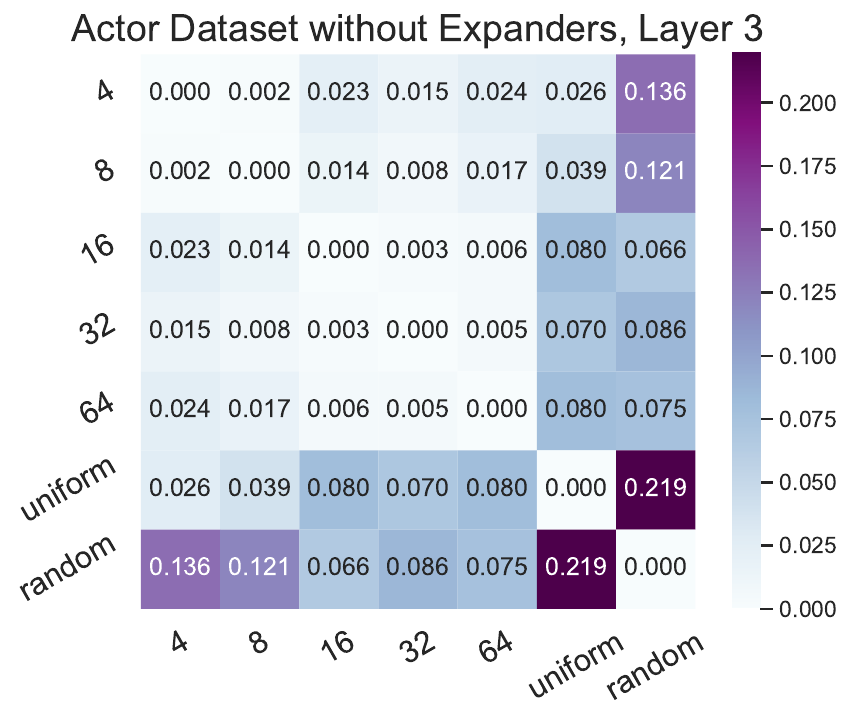}}
\caption{Pairwise energy distance between networks with different hidden dimensions, uniform distribution, and random attention scores for the Actor dataset, without the expander graph, on individual layers. This dataset has a very low average degree and it appears that almost always the first layer's attention scores are very similar to the uniform distribution.
}
\label{fig:pairwise_attention_score_edistances_w_layer_Actor_wo_exp}
\end{figure*}

\begin{figure*}[]
\captionsetup[subfloat]{farskip=-1pt,captionskip=-1pt}
\centering
\subfloat[][]{\includegraphics[width = 2.7in]{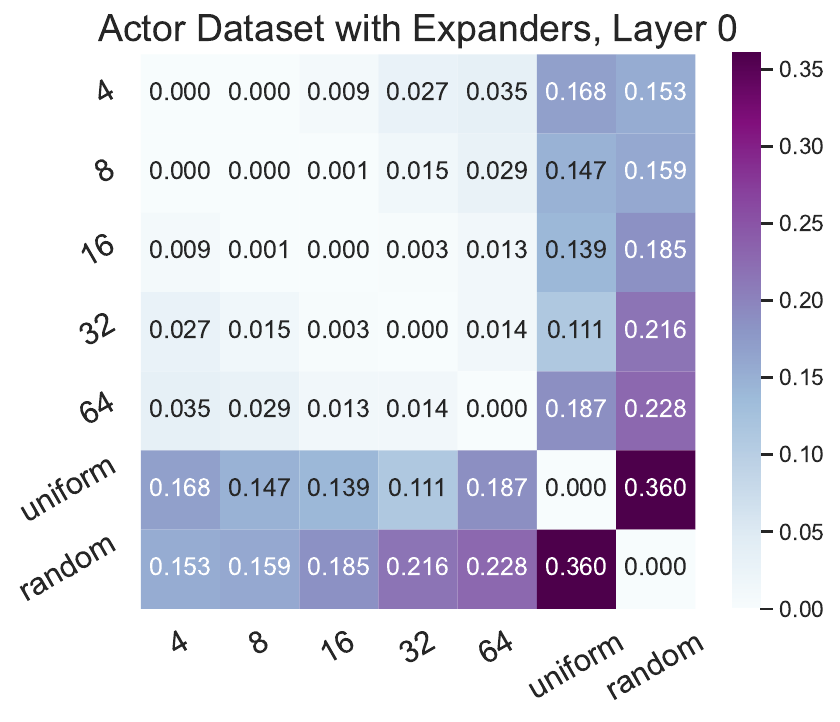}}
\subfloat[][]{\includegraphics[width = 2.7in]{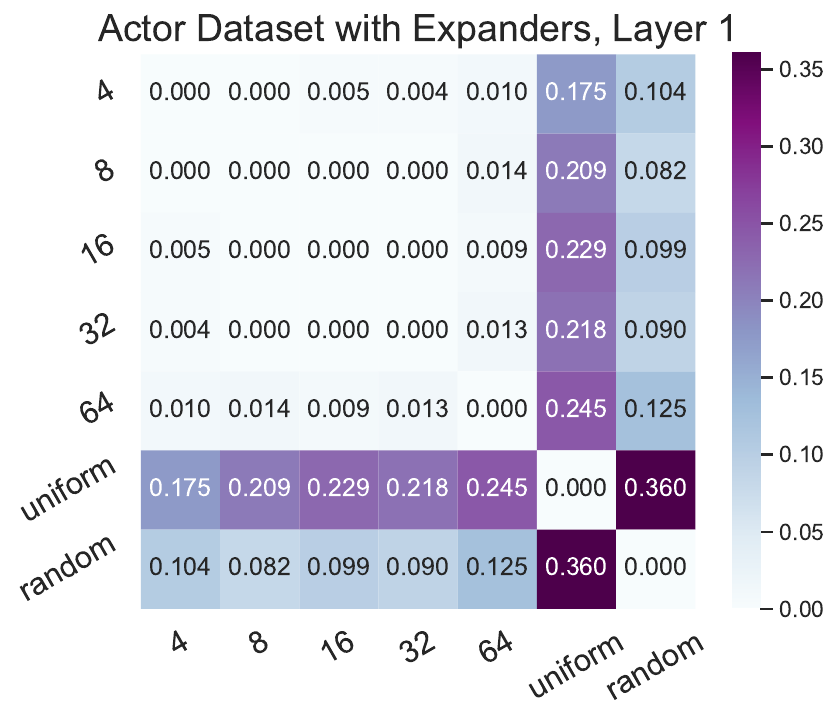}}\\
\subfloat[][]{\includegraphics[width = 2.7in]{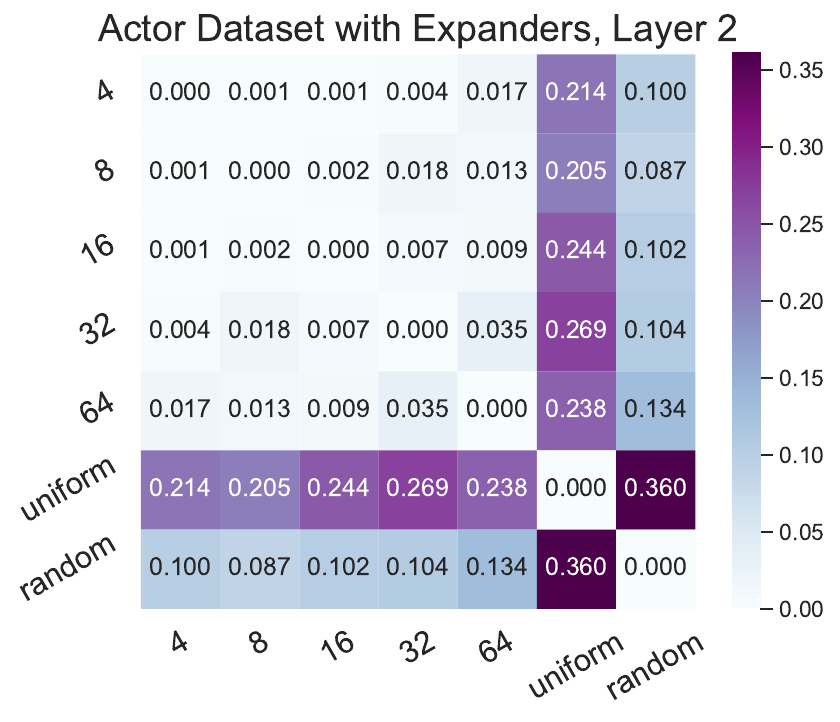}}
\subfloat[][]{\includegraphics[width = 2.7in]{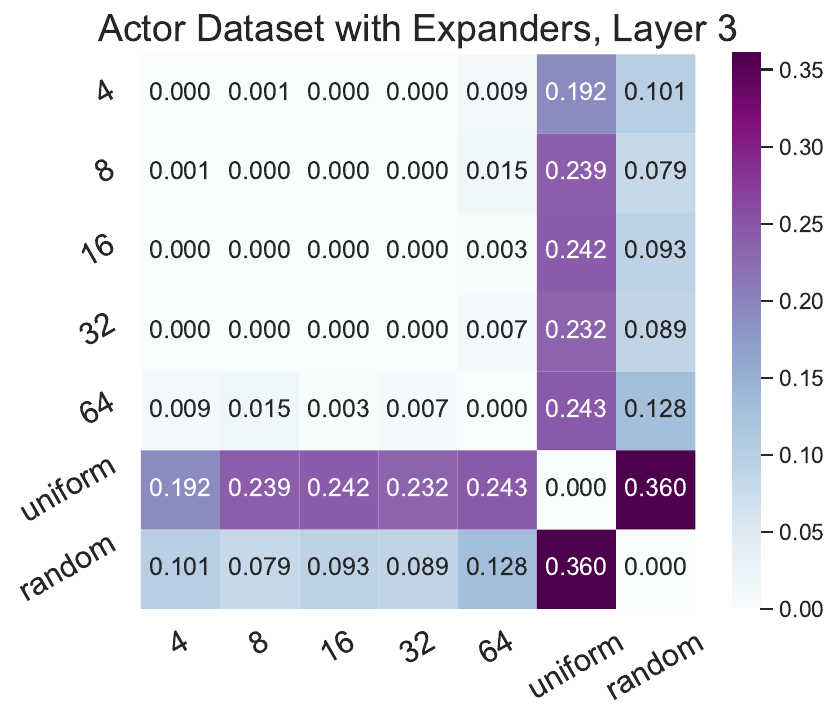}}

\caption{Pairwise energy distance between networks with different hidden dimensions, uniform distribution, and random attention scores for the Actor dataset, with the expander graph, on individual layers.}
\label{fig:pairwise_attention_score_edistances_w_layer_Actor_w_exp}
\end{figure*}

\begin{figure*}[]
\captionsetup[subfloat]{farskip=-1pt,captionskip=-1pt}
\centering
\subfloat[][]{\includegraphics[width = 2.7in]{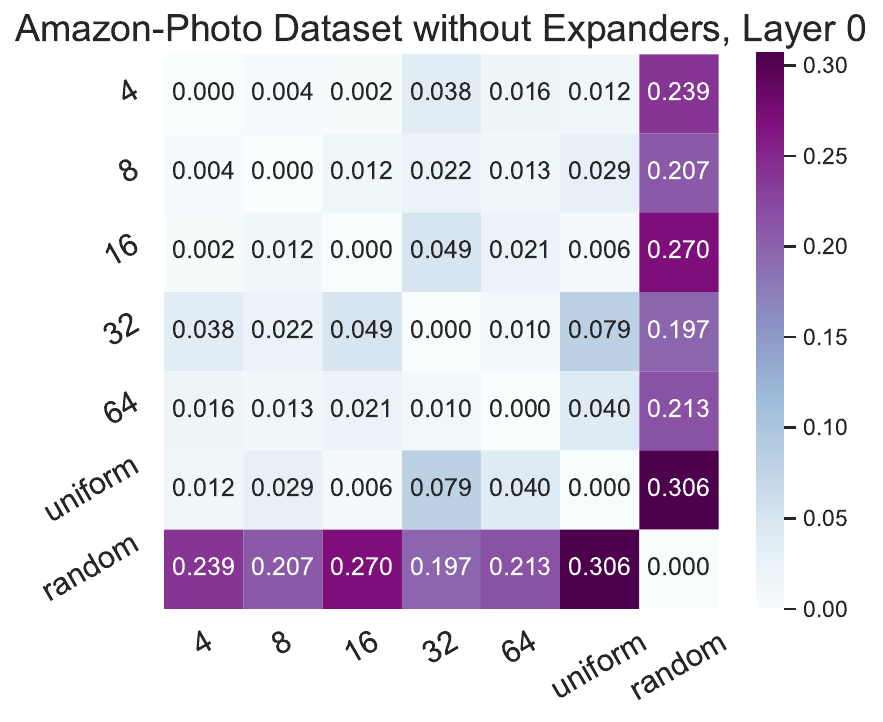}} 
\subfloat[][]{\includegraphics[width = 2.7in]{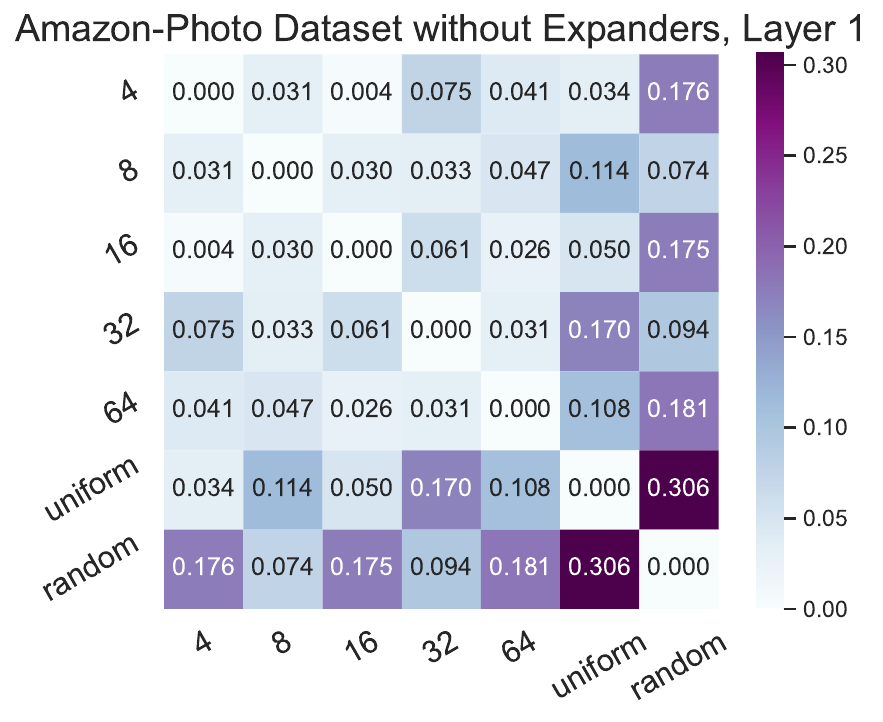}}\\
\subfloat[][]{\includegraphics[width = 2.7in]{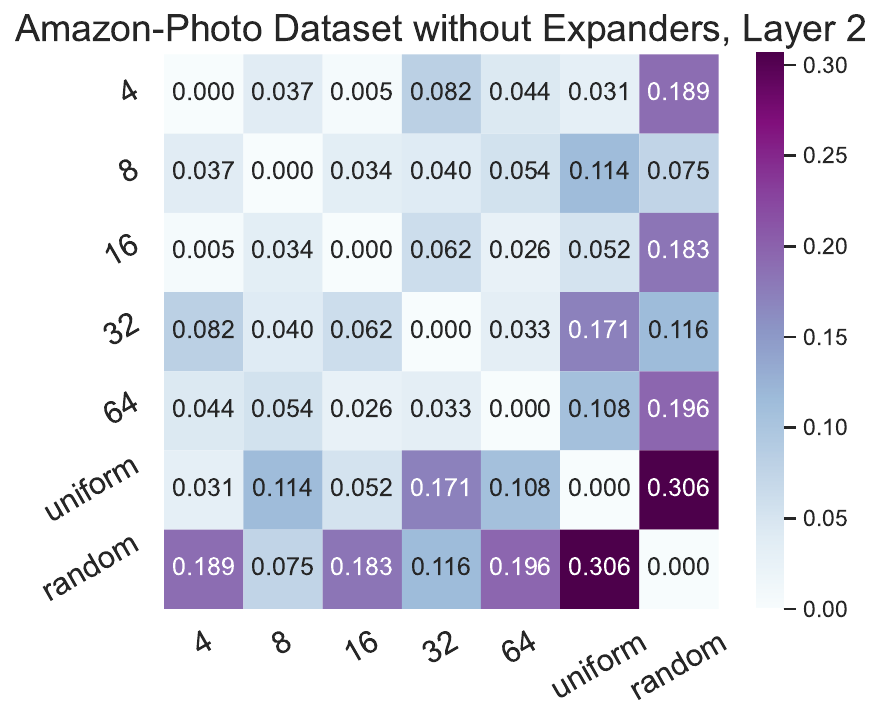}} 
\subfloat[][]{\includegraphics[width = 2.7in]{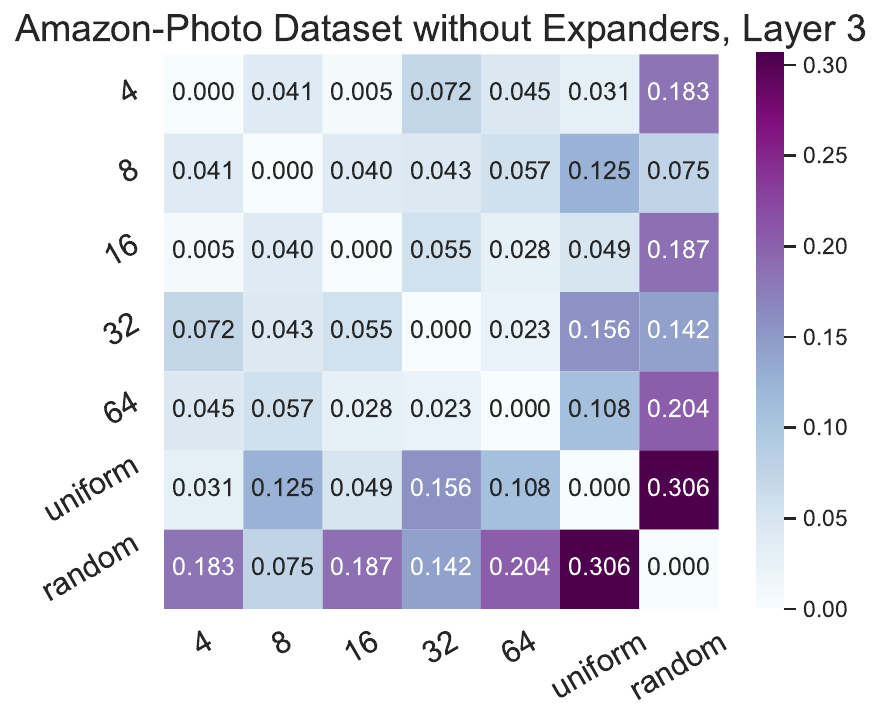}}
\caption{Pairwise energy distance between networks with different hidden dimensions, uniform distribution, and random attention scores for the Amazon-Photo dataset, without the expander graph, on individual layers.}
\label{fig:pairwise_attention_score_edistances_w_layer_photo_wo_exp}
\end{figure*}

\begin{figure*}[]
\captionsetup[subfloat]{farskip=-1pt,captionskip=-1pt}
\centering
\subfloat[][]{\includegraphics[width = 2.7in]{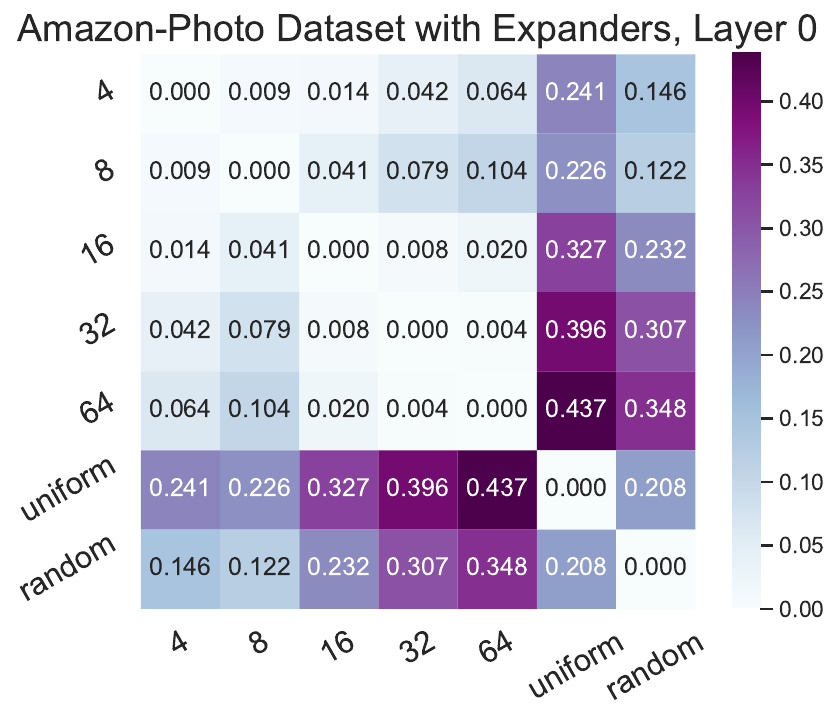}}
\subfloat[][]{\includegraphics[width = 2.7in]{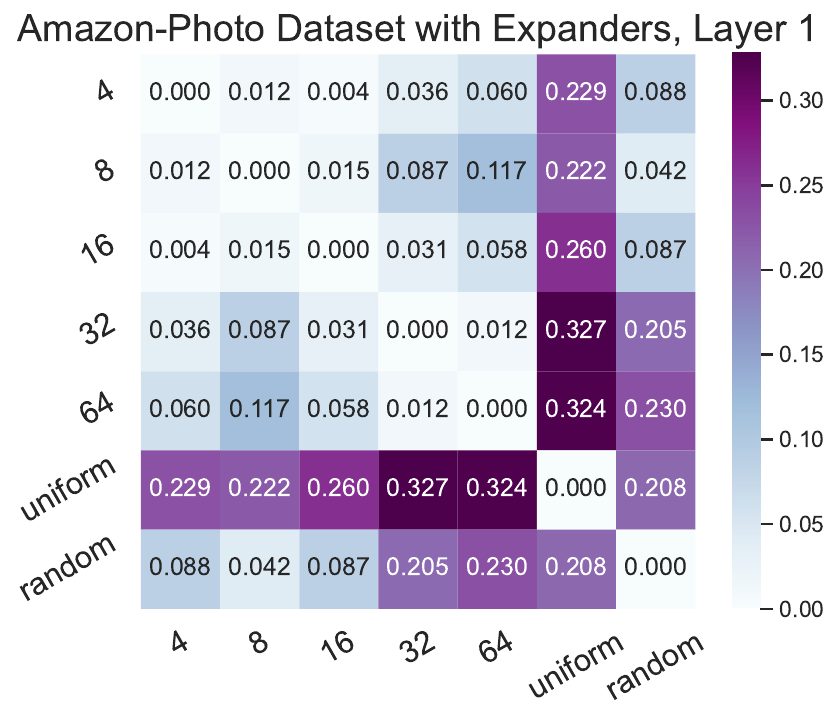}}\\
\subfloat[][]{\includegraphics[width = 2.7in]{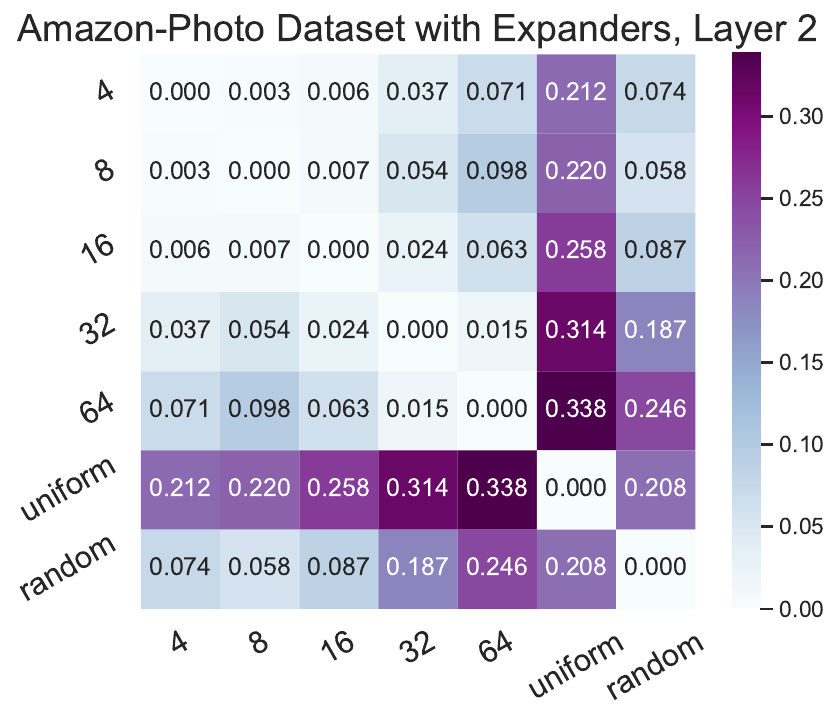}}
\subfloat[][]{\includegraphics[width = 2.7in]{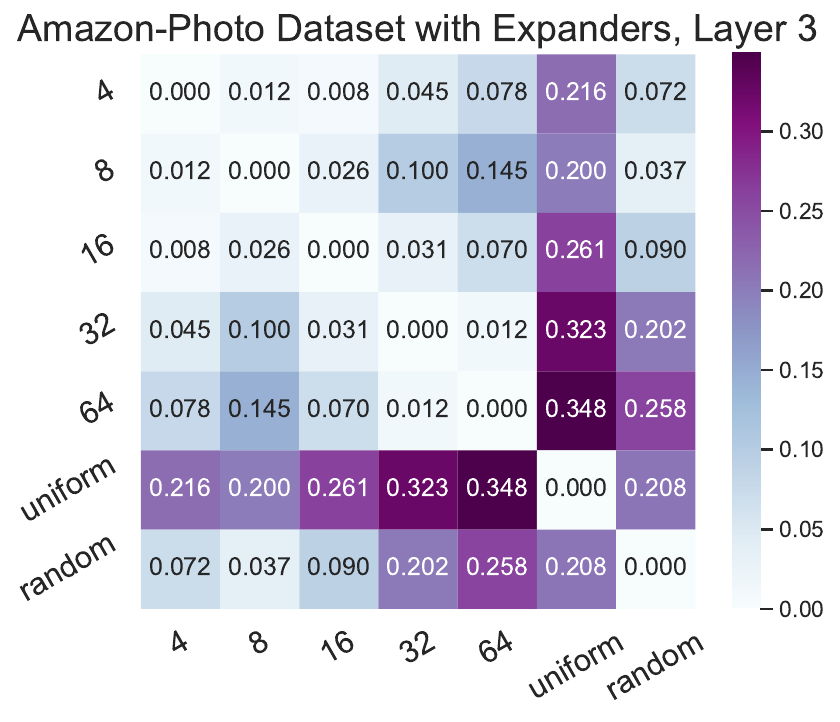}}

\caption{Pairwise energy distance between networks with different hidden dimensions, uniform distribution, and random attention scores for the Amazon-Photo dataset, with the expander graph, on individual layers.}
\label{fig:pairwise_attention_score_edistances_w_layer_photo_w_exp}
\end{figure*}

\subsection{Pairwise Distances}
\label{sec:pairwise_dists}
While we demonstrated the energy distances from the reference hidden dimension of $64$ in \cref{fig:energy_dists}, it is also valuable to examine all pairwise distances. We present these pairwise distances in \cref{fig:pairwise_attention_score_edistance}. Additionally, these distances may vary layer by layer, so it is insightful to explore how these distances change across different layers of the network. To this end, we provide the results in \cref{fig:pairwise_attention_score_edistances_w_layer_Actor_wo_exp,fig:pairwise_attention_score_edistances_w_layer_Actor_w_exp,fig:pairwise_attention_score_edistances_w_layer_photo_wo_exp,fig:pairwise_attention_score_edistances_w_layer_photo_w_exp}.  

These experiments consistently show that attention scores obtained from different hidden dimension sizes are close to each other. In contrast to uniform sampling, or randomly generated attention scores, this distribution provides a much better reference for drawing neighborhood samples when the goal is to select nodes based on their attention score importance.

\begin{tcolorbox}[colback=yellow!10!white, colframe=yellow!50!black, title=\faLightbulb\ Insight 1]
Attention scores from a network with a smaller hidden dimension serve as a good estimator for the attention scores in a network with a higher hidden dimension.
\end{tcolorbox}

\begin{figure*}[]
\captionsetup[subfloat]{farskip=-1pt,captionskip=-1pt}
\centering
\subfloat[][]{\includegraphics[width = 2.5in]{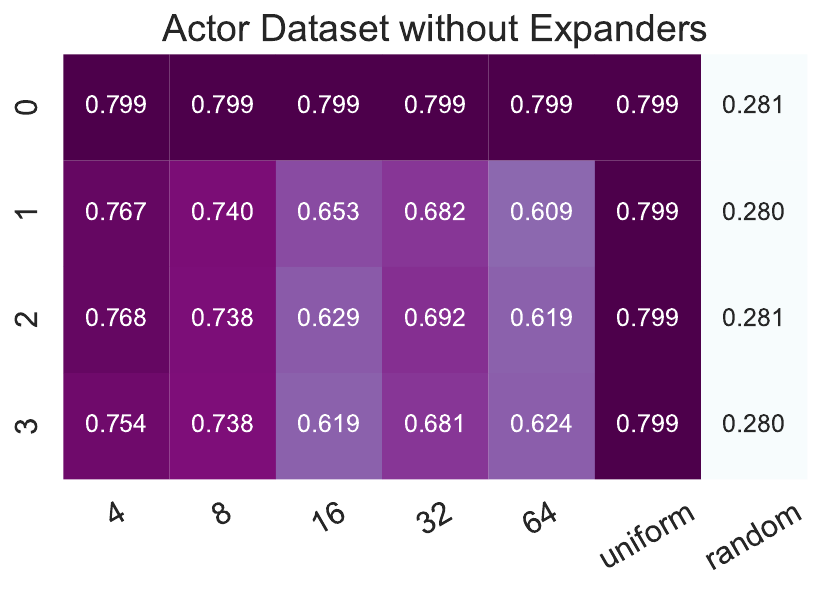}}
\hspace{1em}\subfloat[][]{\includegraphics[width = 2.5in]{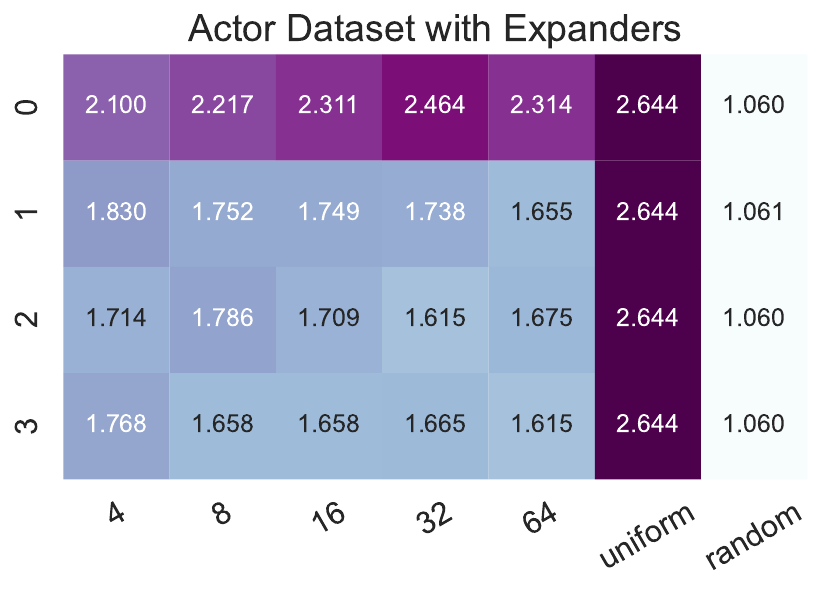}}\\
\subfloat[][]{\includegraphics[width = 2.5in]{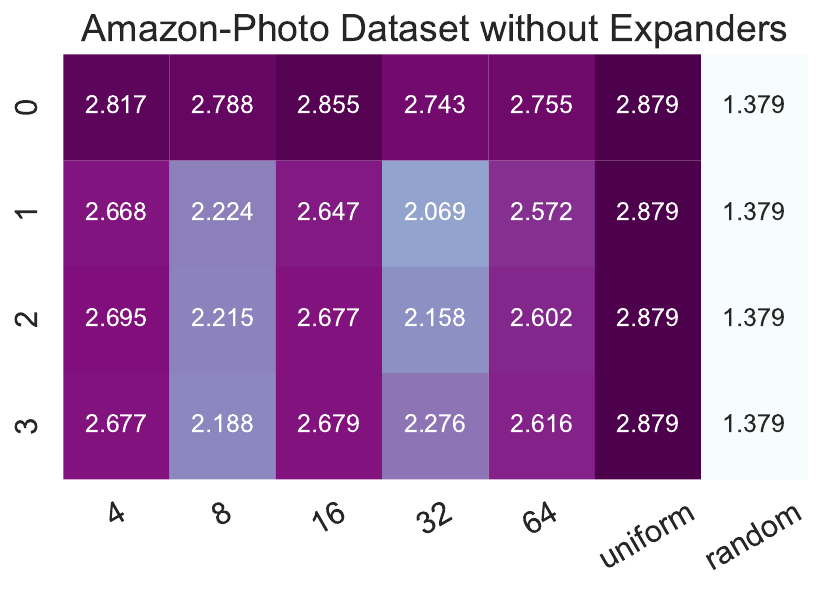}}
\hspace{1em}
\subfloat[][]{\includegraphics[width = 2.5in]{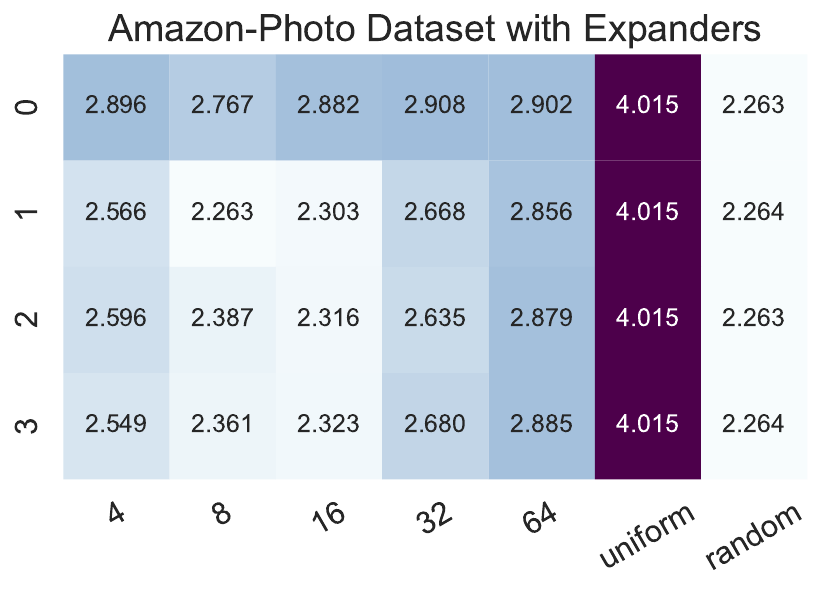}}

\caption{Average entropy of attention scores for nodes across different layers.}
\label{fig:entropy}
\end{figure*}

\subsection{Entropy of Attention Scores}  
\label{sec:entropy}
In this experiment, we analyze the entropy of the attention scores to examine how they change across layers. When the scores approach a one-hot vector, we refer to them as sharp attentions, while more smooth scores resemble a uniform distribution over the neighbors. The goal is to assess how sharp or smooth the attention scores are, on average, across the nodes. To achieve this, we use the entropy metric. Higher entropy indicates more smooth attention scores, while entropy is zero for one-hot vectors. We calculate the entropy for each node's neighborhood and then average the entropies across all nodes and all random seeds in the layer. The results are presented in \cref{fig:entropy}.

An insightful observation from this experiment is that the first layer, across all four datasets, consistently exhibits smoother attention scores, while the scores become sharper in subsequent layers. Generally, however, the attention scores are not very sharp in experiments without expander graphs, suggesting that all neighbors are likely similarly informative. This does not necessarily imply that all these nodes are equally important. If identical nodes with the same neighborhoods surround a node, all of them will receive equal attention scores, which indicates no selection in this case. Thus, this does not contradict the idea that a sparse matrix can estimate the same results.

Sharpness varies across different hidden dimensions, which may be due to factors such as training dynamics, learning rate, and the varying temperature setup for different hidden dimensions. Regardless, in all datasets and across all hidden dimensions, the first layer consistently has higher entropy. This suggests that for sampling, larger neighborhood sizes may be needed in the first layer, while smaller neighborhood sizes could suffice in the subsequent layers.

\begin{tcolorbox}[colback=yellow!10!white, colframe=yellow!50!black, title=\faLightbulb\ Insight 2]
Attention scores are smoother in the first layer, and become sharper in subsequent layers.
\end{tcolorbox}

\begin{figure*}[]
\captionsetup[subfloat]{farskip=10pt,captionskip=2pt}
\centering
\subfloat[][Actor Dataset without Expander]{\includegraphics[width = \textwidth]{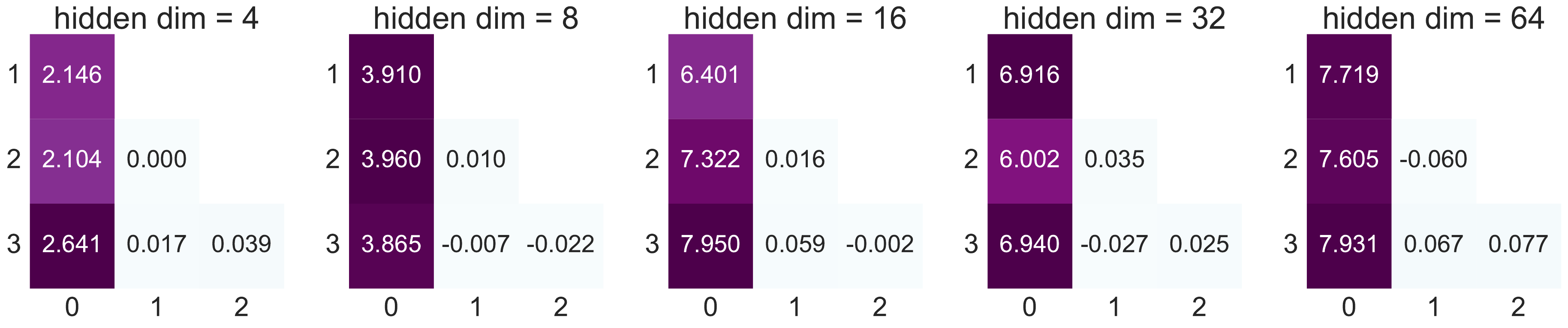}}\\
\subfloat[][Actor Dataset with Expander]{\includegraphics[width = \textwidth]{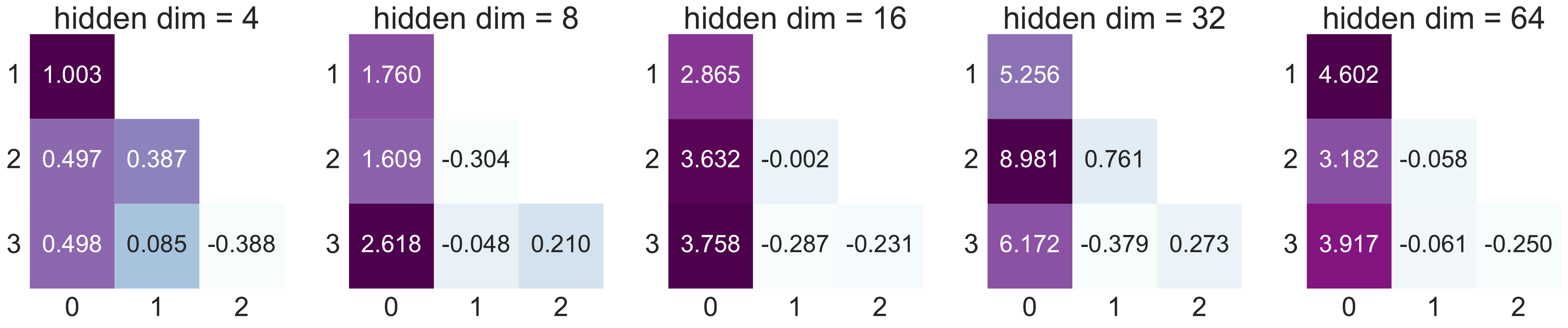}}\\
\subfloat[][Amazon-Photo Dataset without Expander]{\includegraphics[width = \textwidth]{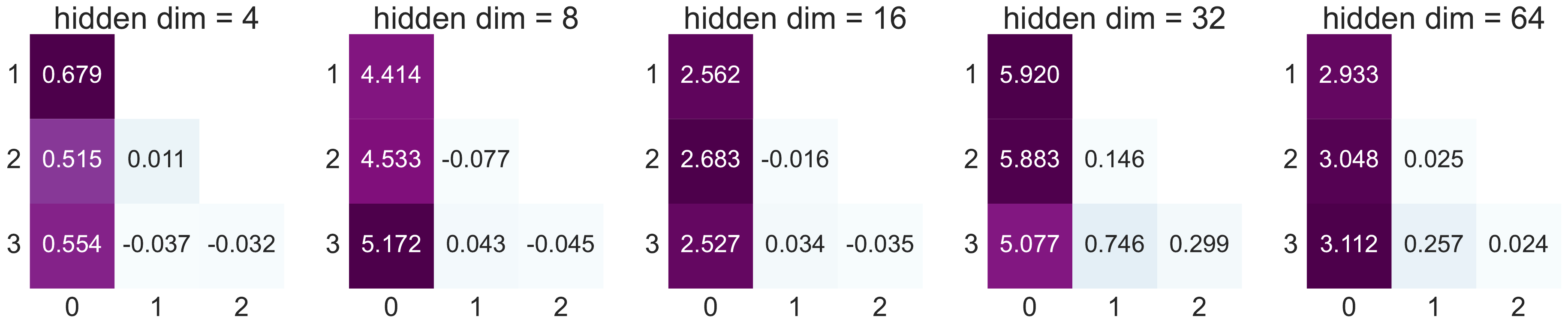}}\\
\subfloat[][Amazon-Photo Dataset with Expander]{\includegraphics[width = \textwidth]{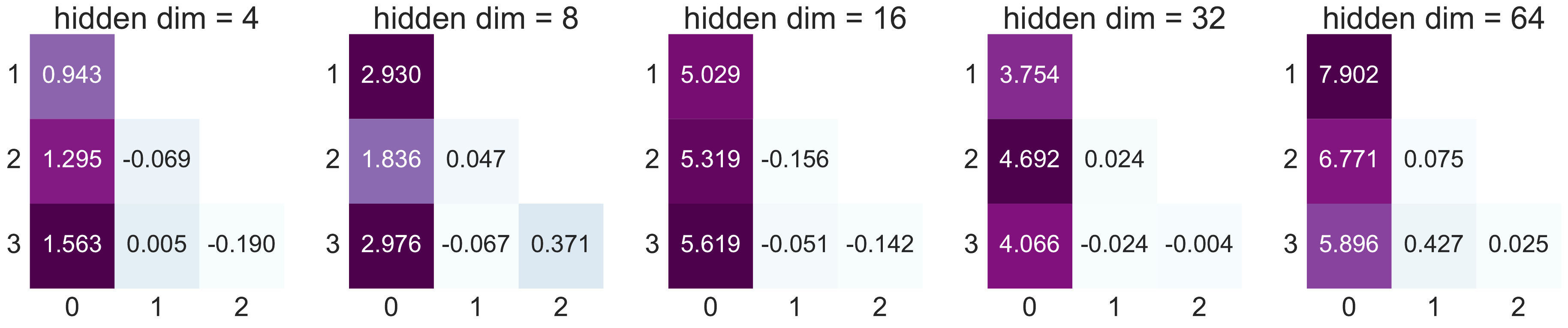}}
\caption{Inter-layer energy distances ($\times 100$) for different hidden dimensions.}
\label{fig:interlayer_edists}
\end{figure*}

\subsection{Inter-layer Attention Scores Similarity}  
\label{sec:inter_layer}
After observing that the entropy is higher in the first layer and similar across the subsequent layers, it is worth examining the distance between the attention scores of each pair of layers. The experimental results are presented in \cref{fig:interlayer_edists}. All values are relatively small compared to the previous ones, so they are multiplied by $100$ for better presentation. The results show that, consistently, the first layer has some distance from all other layers, but the layers following it exhibit very similar attention scores. This suggests that the initial network may be trained using fewer layers, and further layer sampling could be achieved by repeating the attention scores from the final layer to train a deeper Spexphormer model.

\begin{tcolorbox}[colback=yellow!10!white, colframe=yellow!50!black, title=\faLightbulb\ Insight 3]
The attention scores in the layers after the first are consistently very similar to one another, but distinct from the attention scores in the first layer.
\end{tcolorbox}

\begin{figure}
    \centering
    \includegraphics[width=\linewidth]{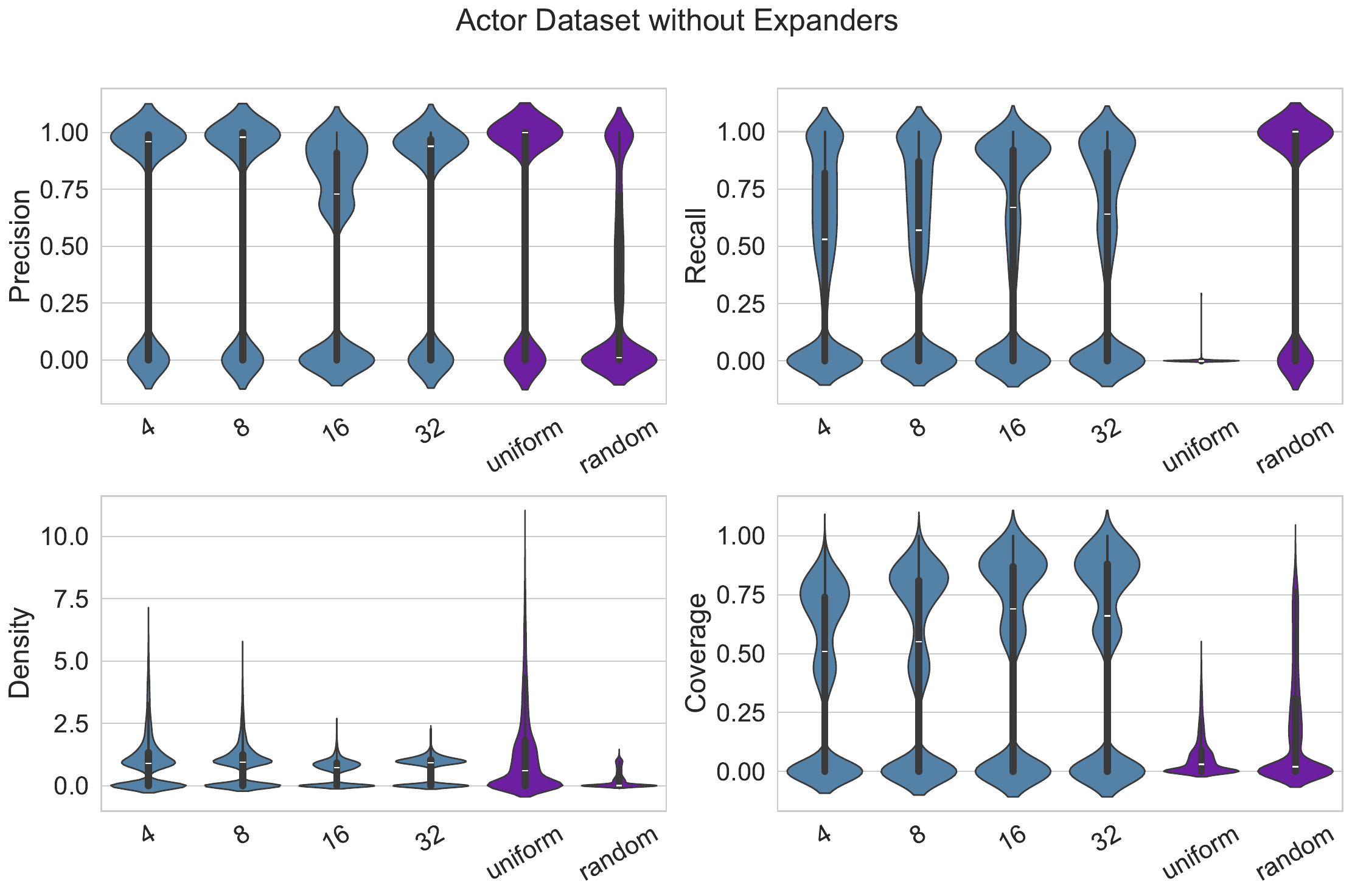}
    \caption{Violin plots of Precision, Recall, Density, and Coverage metrics for the Actor dataset without expander graphs.}
    \label{fig:prdc_actor_wo_exp}
\end{figure}

\begin{figure}
    \centering
    \includegraphics[width=\linewidth]{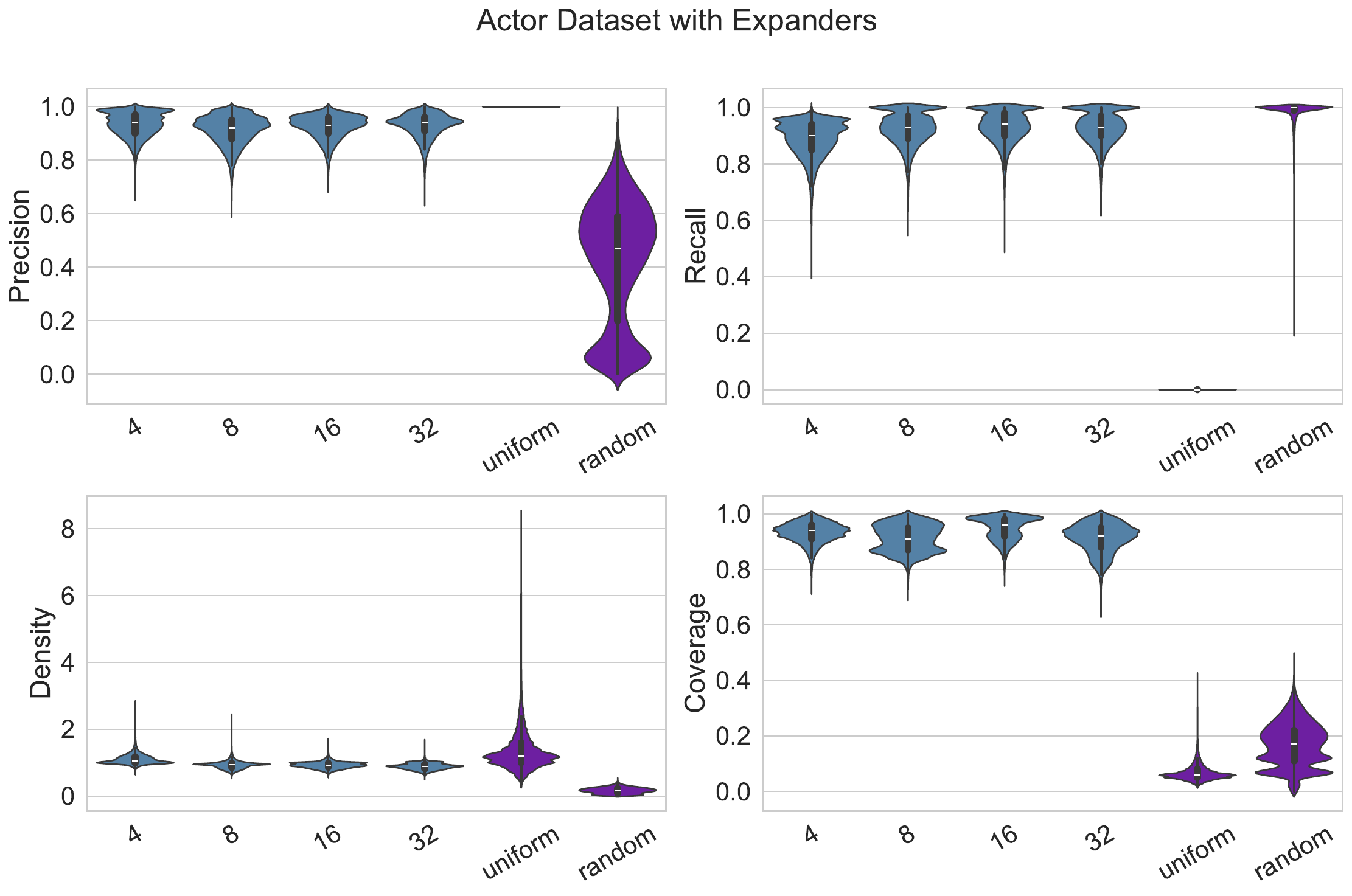}
    \caption{Violin plots of Precision, Recall, Density, and Coverage metrics for the Actor dataset with expander graphs.}
    \label{fig:prdc_actor_w_exp}
\end{figure}

\begin{figure}
    \centering
    \includegraphics[width=\linewidth]{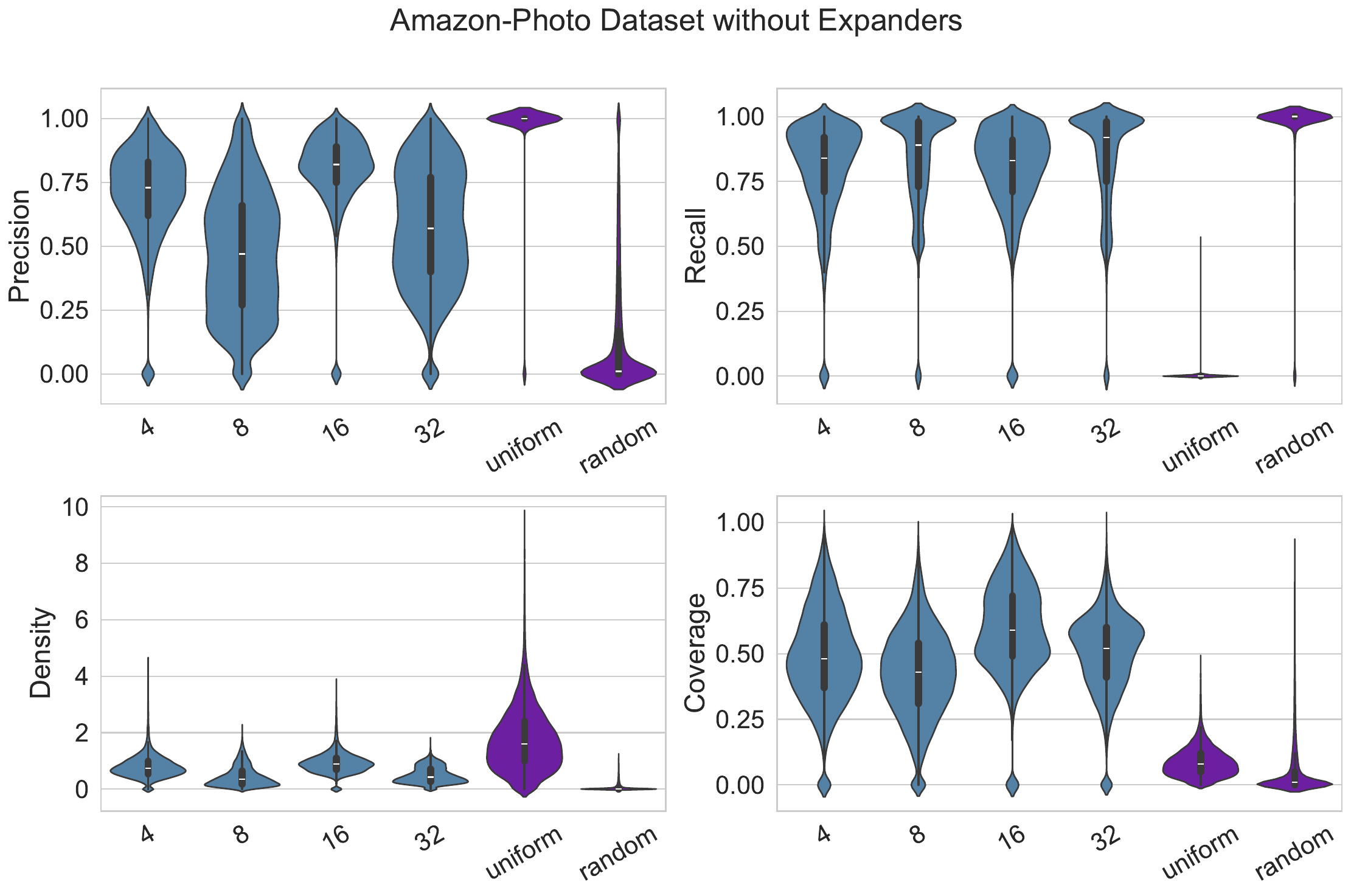}
    \caption{Violin plots of Precision, Recall, Density, and Coverage metrics for the Amazon-Photo dataset without expander graphs.}
    \label{fig:prdc_photo_wo_exp}
\end{figure}

\begin{figure}
    \centering
    \includegraphics[width=\linewidth]{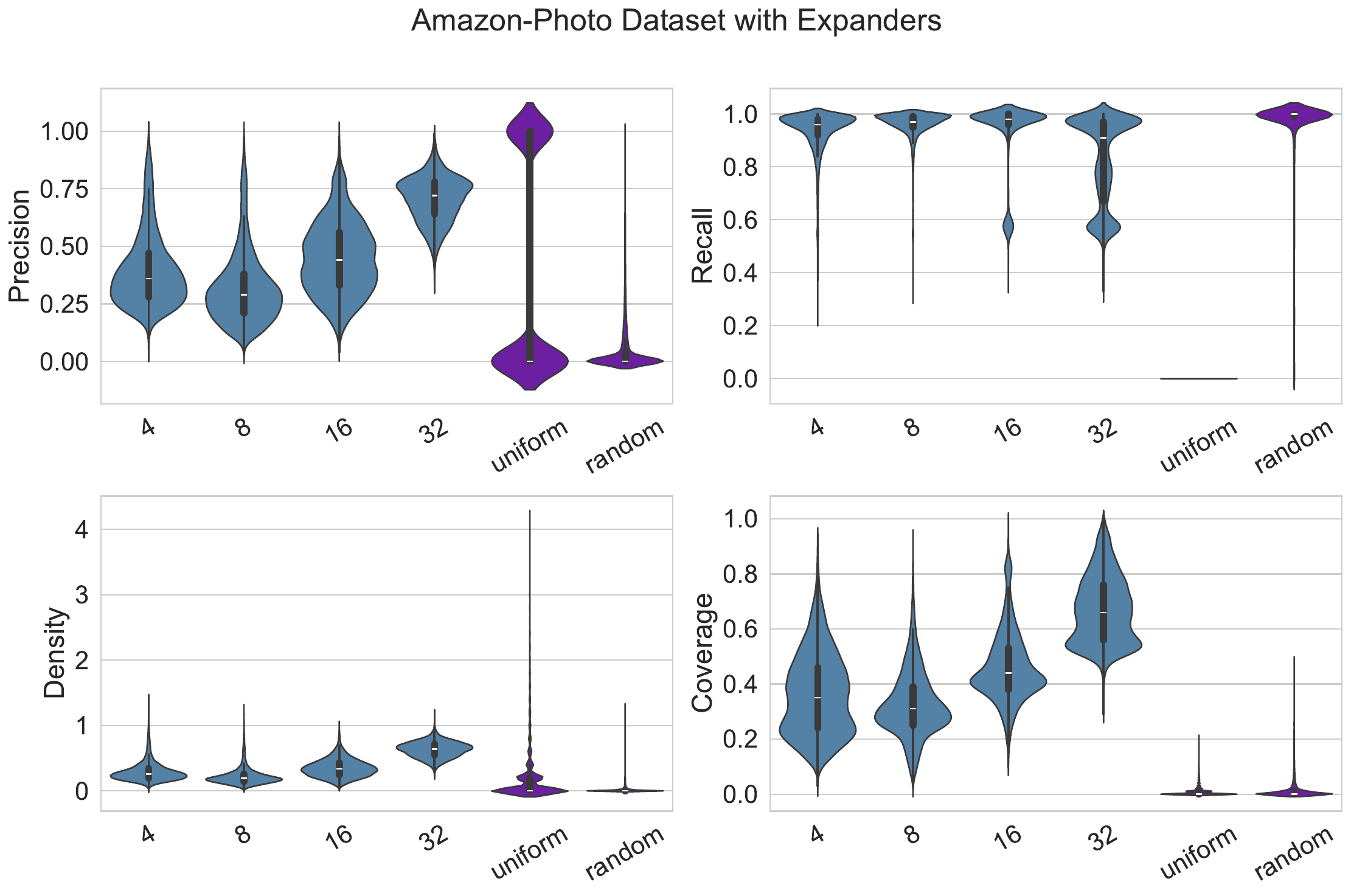}
    \caption{Violin plots of Precision, Recall, Density, and Coverage metrics for the Amazon-Photo dataset with expander graphs.}
    \label{fig:prdc_photo_w_exp}
\end{figure}

\subsection{Precision, Recall, Density, \& Coverage (PRDC)}  
\label{sec:prdc}
Since the energy distance may not fully capture how distributions match in some cases, alternative metrics have been proposed, primarily for assessing the performance of generative models \citep{sajjadi2018assessing, naeem2020reliable}. In this work, we apply these metrics by considering the attention scores from the network with a hidden dimension of $64$ as the reference distribution, assuming that all other dimensions aim to generate the same distribution. We use violin plots to illustrate the distribution of PRDC values across the nodes in each layer. The results are presented in \cref{fig:prdc_actor_wo_exp,fig:prdc_actor_w_exp,fig:prdc_photo_wo_exp,fig:prdc_photo_w_exp}. The plots show the kernel density estimate of the corresponding metrics across all nodes, layers, and random initializations.

Precision \& Recall and Density \& Coverage are pairs of metrics that together describe how well the distribution has been learned. Excelling in just one of these metrics does not necessarily imply that the samples are close to each other. As shown in the results, attention scores from other hidden dimensions consistently achieve high values across all metrics, while uniform distribution and random attention scores fall short in at least one of the metrics from each pair.

\begin{figure}
    \centering
    \includegraphics[width=\linewidth]{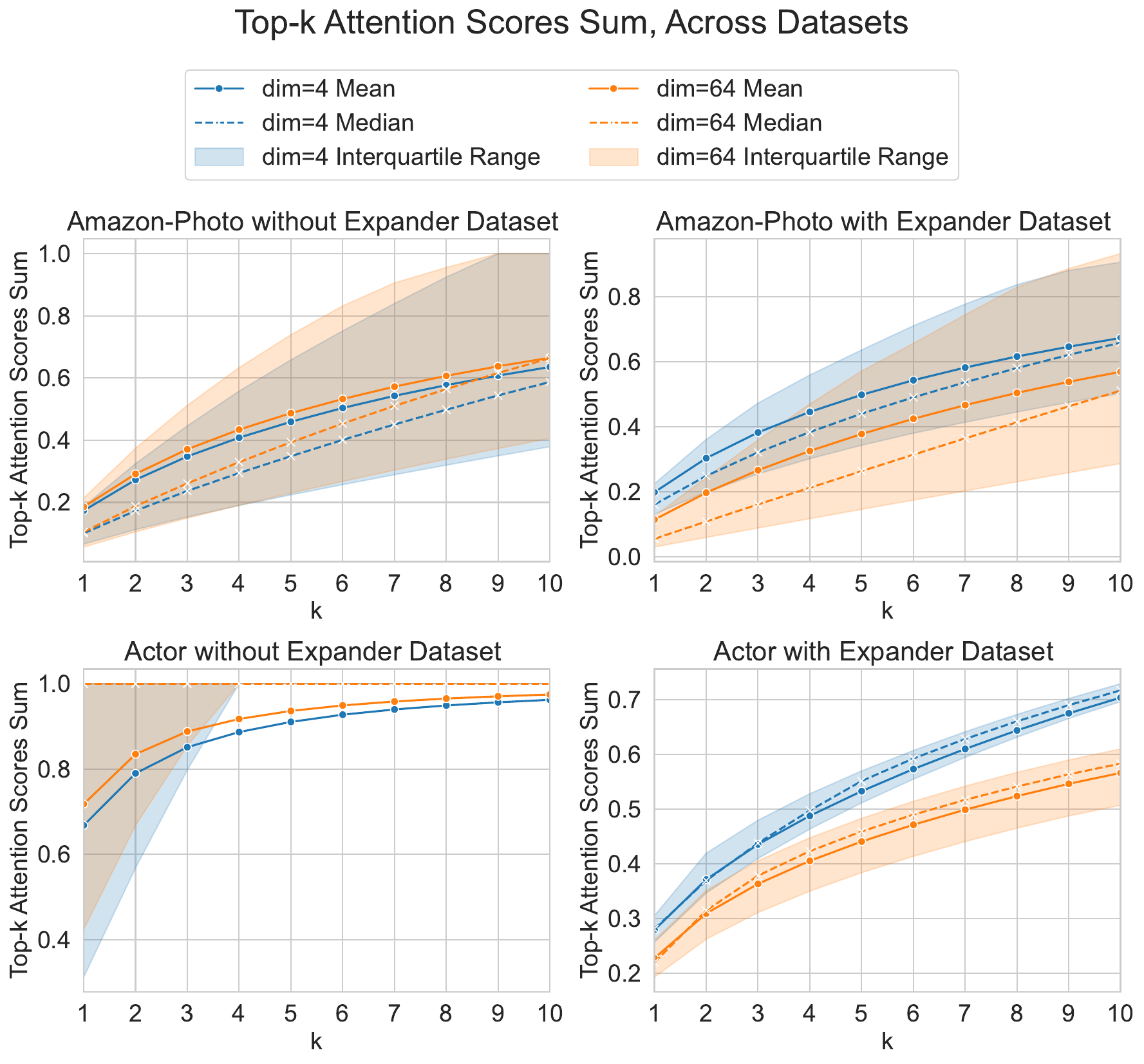}
    \caption{Top-k attention scores sum for k values between $1$ to $10$.}
    \label{fig:topk-attention-sum}
\end{figure}

\subsection{Top-k Attention Sum}
\label{sec:top-k}
Another way to assess the sharpness of the attention scores is by examining the sum of the top-$k$ attention scores. If the top-$k$ attention scores for a small $k$ almost sum to one for all nodes, then using the top-$k$ scores can closely approximate the representations of the larger network. However, this is not always the case. In this experiment, we analyze the sum of the top-$k$ attention scores for $k$ ranging from one to ten, across all nodes for hidden dimensions of 64 and 4. While the top-$k$ attention score distributions are similar, the assumption that the sum will be close to one is unrealistic and does not occur frequently. The results, shown in \cref{fig:topk-attention-sum}, include mean, median, and interquartile range, which indicate the spread of the middle $50\%$ of the results. These results suggest that top-$k$ attention selection may not be fully representative in transductive learning on graphs. This could be due to the presence of many similar nodes, causing the attention to be distributed across these nodes rather than being concentrated on a small subset, which affects the ability to approximate the larger network effectively using just the top-$k$ scores.

\begin{tcolorbox}[colback=yellow!10!white, colframe=yellow!50!black, title=\faLightbulb\ Insight 4]
The attention scores in the layers after the first are consistently very similar to one another, but distinct from the attention scores in the first layer.
\end{tcolorbox}

\begin{figure}
    \centering
    \includegraphics[width=\linewidth]{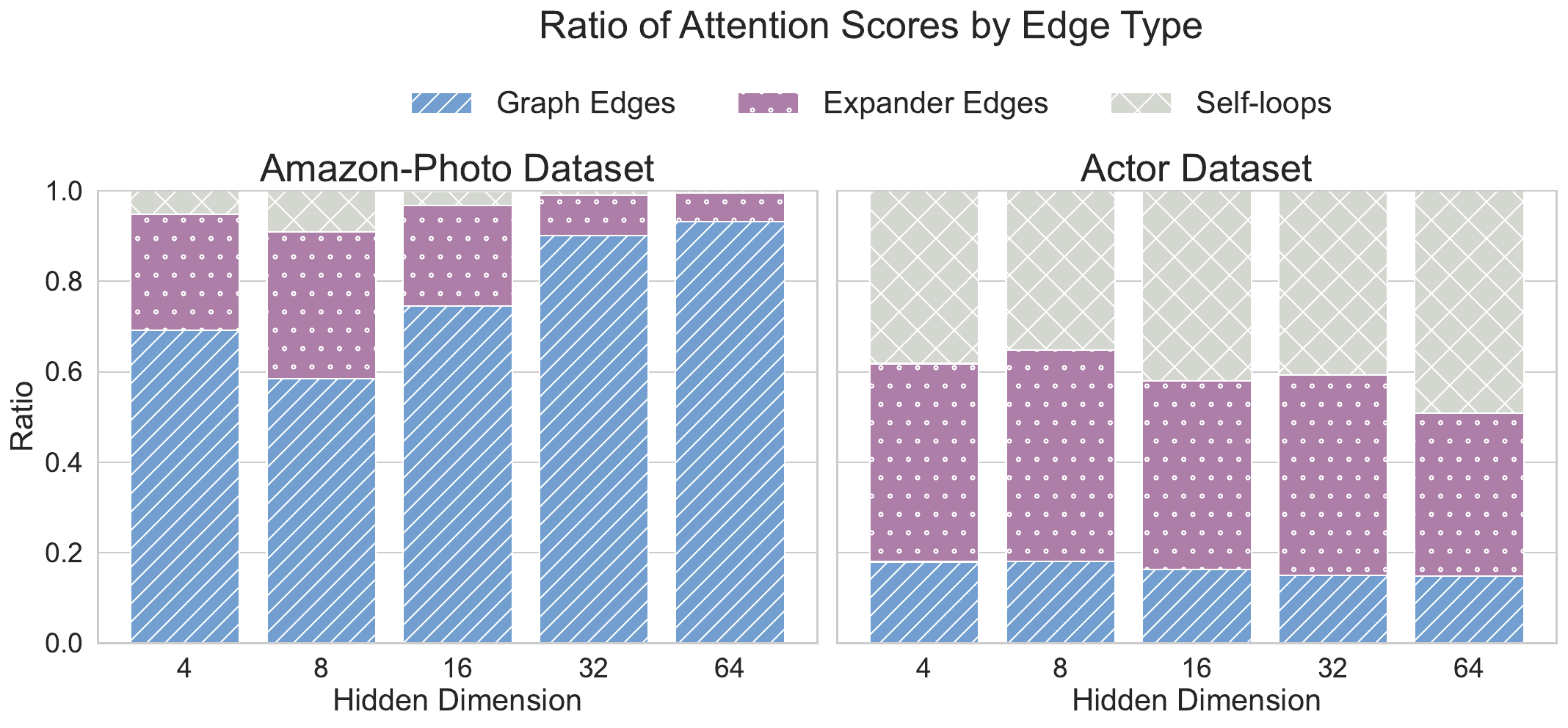}
    \caption{Average sum of attention scores for different edge types—graph edges, expander edges, and self-loops—per node neighborhood. The total sum of attention scores per node is one. }
    \label{fig:edge_type_sum}
\end{figure}

\begin{figure}
    \centering
    \includegraphics[width=\linewidth]{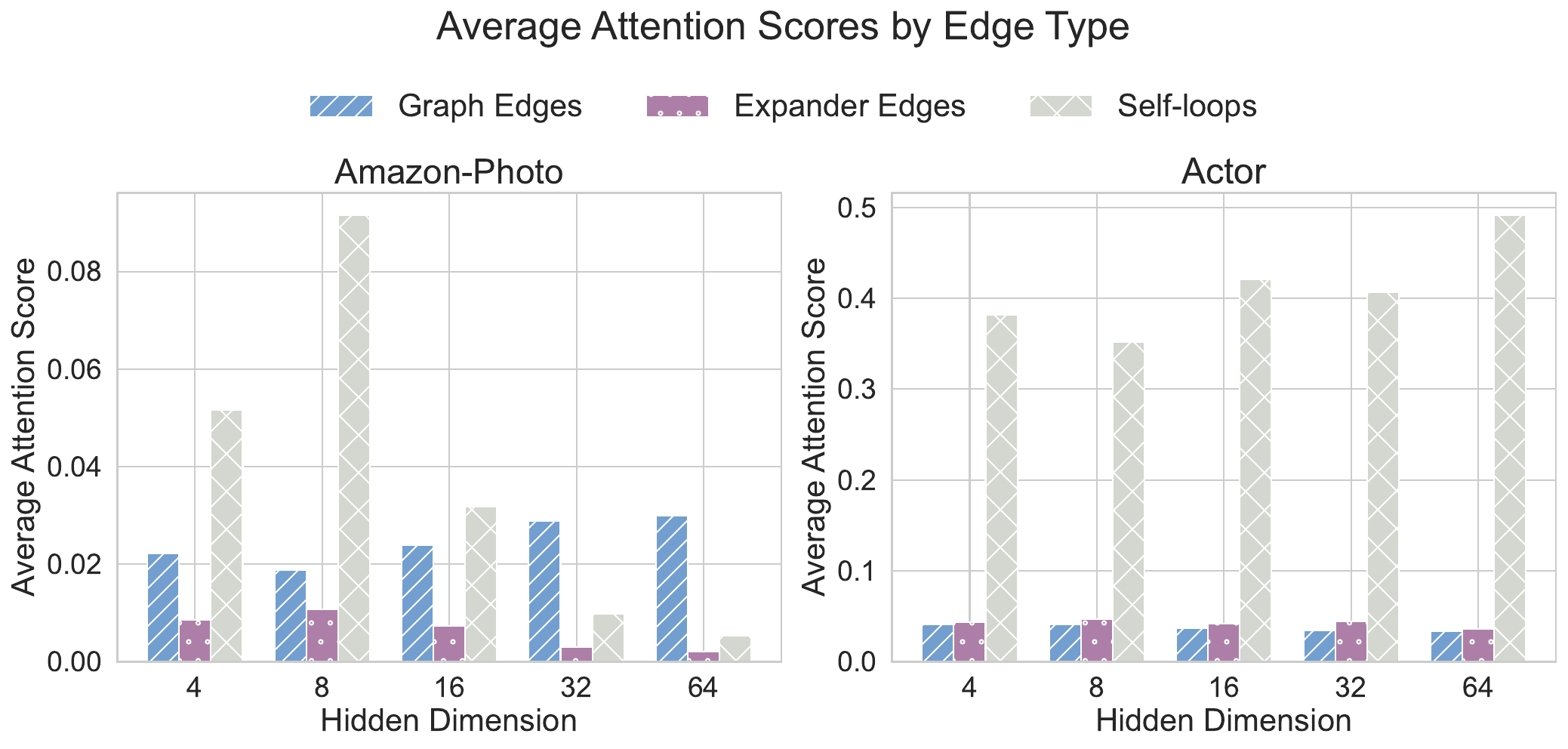}
    \caption{Average attention scores for different edge types across two datasets and for different hidden dimensions.}
    \label{fig:edge_type_normalized}
\end{figure}

\subsection{Attention Scores by Edge Type}
\label{sec:edge_type}
An interesting question is to examine the ratio of attention scores coming from graph edges versus expander edges and self-loops. \cref{fig:edge_type_sum} illustrates how much of the attention score, on average, is attributed to each edge type across different hidden dimensions. We average the values over all nodes in all layers and random initializations. As expected, for a homophilic dataset like Amazon-Photo, the graph edges are more important. As the model's hidden dimension increases, the model learns to place more attention on these edges. However, for a heterophilic dataset like Actor, the story is different, with graph edges playing a lower role. In \cref{fig:edge_type_normalized}, we present a normalized version showing the average attention score by edge type.

\section{Discussion}
\label{sec:discussion}
Graph datasets arise from various domains, meaning that they might have differing inductive biases. More expressive methods may not necessarily yield better results on all datasets \citep{franks2024weisfeiler}. Depending on the architecture and the task, more complex models can even lead to poorer results. Here, we discuss possible scenarios in which our model can be a good fit as well as the shortcomings of other classes of models that are overcome by our model.

\paragraph{Graph Structure} The relevance of the structure of the graph to the task can vary. For the simple synthetic task introduced in \ref{fig:synthetic_task}, the structure of the graph does not matter. So Transformers without inductive biases of the graphs are expressive enough to solve this problem; however message-passing networks will be restricted to the graph edges and rely on enough number of layers and may be challenged by oversquashing and oversmoothing problems. On the other hand, if the structure of the graph matters, such as counting the number of neighbor nodes with the same color for each node, the structure and the edges will be an important part. Transformers without expressive enough encodings to identify the graph edges will fail in this task. On the other hand, MPNNs even with one layer can easily solve this problem. Our approach enables solving problems in either case, by having both expander graphs for universal information propagation and the actual graph edges for inductive bias, allowing the model to decide the subset of edges that suit the task better --- only graph edges, only expander edges or a combination of both.

\paragraph{Short-range Vs. Long-range Dependencies} If the neighboring nodes tend to be from the same class, i.e., {\em high homophily}, MPNNs and methods such as NAGphormer \citep{nagphormer22}, which summarize the neighborhood have good inductive biases; whereas Transformers without proper identification for the neighborhoods may not be as fit for this task. Heterophily may not necessarily mean long-range dependencies, label of each node may just depend on the neighbor nodes, but still label of the neighbor nodes may be different most of the time. For example, for finding the grammatical function of the words in a sentence from a very long text, neighboring words are usually enough for this identification, and nearby words would be from different classes. On the other hand, some tasks may require long-range dependencies --- identifying if there are other people in a social network with similar interests or the synthetic task introduced in \ref{fig:synthetic_task} are some examples. Local models such as MPNNs would require deeper networks for modeling long-range dependencies that makes them prone to common problems such as oversquashing and oversmoothing \citep{topping2021understanding, di2023does, di2023over, rusch2023survey}. Our approach can be reduced to MPNN by giving lower attention scores to the expander edges, for learning on the tasks with short-range dependencies only. And also lets the long-range dependency modeling using expander edges. While models designed specifically for some of these tasks may have the advantage of reduced complexity. But our approach lets learning without concern about the nature of the problem or having domain knowledge for the task or graph.

\paragraph{Subsampling Graphs} Many approaches break the graph into sections or subsample nodes or neighbors for training. This approach has shown promising results in many works such as \citep{ZengZSKP20, hamilton2017inductive, liu2021sampling}. However, there are many cases in which these approaches are not expressive enough. Clustering the nodes or batching and subsampling based on the neighborhood will not have the required inductive biases to solve the tasks with long-range dependencies. Approaches such as neighbor sampling or connected-subgraph sampling not only inherit the limits of the MPNN networks, but may even miss short-range dependencies. For example, Example (c) in \ref{fig:synthetic_task} by merely random selection of the neighbors or subgraphs without considering the task. Random subset of node selection that has been used in several promising papers such as \cite{wu2022nodeformer, wu2023difformer, wu2024simplifying} gives a chance for nodes from the same label to appear in the same batch, but the batch-size should increase with the graph size accordingly. Very small ratio of batch size to graph size would mean many edges or possible pair of nodes will never be appear in any batch and depending on the task this can limit the power of these models. Also, these models are usually not memory efficient, as graph size grows, they can not keep the batches small, and the required memory grows accordingly. On the other hand, our approach (1) makes smarter selection of neighbors based on the small network's attention scores; (2) our sampling allows making k-hop neighborhood subgraphs from the extended graph connectivity, and (3) allows the training by trading off memory and time, without critical harm to the model's expressive power. Unline the GraphSAGE and SGFormer, which use the full graph for the inference time our model uses the same sampling and batching techniques, letting efficient inference beside the efficient training.

\end{document}